\documentclass[english,10pt]{article}
\usepackage{geometry}
\usepackage[T1]{fontenc}
\usepackage[latin9]{inputenc}
\usepackage{bm}
\usepackage{amsmath,mathtools}
\usepackage{amssymb}
\usepackage[unicode=true,
 bookmarks=false,
 breaklinks=false,pdfborder={0 0 1},colorlinks=false]
 {hyperref}
\hypersetup{
 colorlinks,citecolor=blue,filecolor=blue,linkcolor=blue,urlcolor=blue}

\makeatletter
\usepackage{amsthm}
\usepackage{cite}  
\usepackage{comment}
\usepackage{natbib}
\usepackage{booktabs}

\usepackage{xcolor,colortbl}
\definecolor{Gray}{gray}{0.85}

\usepackage{enumitem}

\usepackage{graphicx}

\usepackage[linesnumbered,ruled,vlined]{algorithm2e}

\SetCommentSty{mycommfont}
\usepackage{algorithmic}

\usepackage{float}
\usepackage{multirow}
 
\usepackage{dsfont}
\usepackage{tcolorbox}

\usepackage{color}
\definecolor{yxc}{RGB}{255,0,0}
\definecolor{yjc}{RGB}{125,0,0}
\definecolor{ytw}{RGB}{255,69,0}
\definecolor{gen}{RGB}{0,0,200}
\definecolor{yly}{RGB}{76,153,0}

\newcommand{\yxc}[1]{\textcolor{yxc}{[YXC: #1]}}

\allowdisplaybreaks

\DeclareMathOperator{\ind}{\mathds{1}}  % Indicator

\newcommand{\mymid}{\,|\,}

\newcommand{\com}{\mathsf{aug}}

\newcommand{\cS}{\mathcal{S}}
\newcommand{\cA}{\mathcal{A}}

\newcommand{\piexploreh}{\pi^{\mathsf{explore},h}}

\newcommand{\behavior}{\widehat{\mu}_{\mathsf{b}}}
\newcommand{\cb}{c_{\mathsf{b}}}

\title{Minimax-Optimal Reward-Agnostic Exploration \\ in Reinforcement Learning\footnotetext{Accepted for presentation at the
Conference on Learning Theory (COLT) 2024.}}

\author{Gen Li\thanks{Department of Statistics and Data Science, Wharton School, University
of Pennsylvania, Philadelphia, PA 19104, USA; Email: \texttt{\{ligen,yuxinc\}@wharton.upenn.edu}.} \and  Yuling Yan\thanks{Department of Operations Research and Financial Engineering, Princeton
University, Princeton, NJ 08544, USA; Email: \texttt{\{yulingy,jqfan\}@princeton.edu}.} \and Yuxin Chen\footnotemark[1] \and Jianqing Fan\footnotemark[2]}

\date{April 2023; ~Revised: May 2024}

\makeatother

\begin{document}

\theoremstyle{plain} \newtheorem{lemma}{\textbf{Lemma}}\newtheorem{proposition}{\textbf{Proposition}}\newtheorem{theorem}{\textbf{Theorem}}\newtheorem{assumption}{\textbf{Assumption}}

\theoremstyle{remark}\newtheorem{remark}{\textbf{Remark}}

\maketitle 

\begin{abstract}

This paper studies reward-agnostic exploration in reinforcement learning (RL) --- a scenario where the learner is unware of the reward functions during the exploration stage ---   
and designs an algorithm that improves over the state of the art. 
More precisely, consider a finite-horizon inhomogeneous Markov decision process with $S$ states, $A$ actions, and horizon length $H$, 
and suppose that there are no more than a polynomial number of given reward functions of interest. 
By collecting an order of 
\begin{align*}
	\frac{SAH^3}{\varepsilon^2} \text{ sample episodes ~~(up to log factor)} 
\end{align*}
without guidance of the reward information, 
our algorithm is able to find $\varepsilon$-optimal policies for all these reward functions, provided that $\varepsilon$ is sufficiently small. 
This forms the first reward-agnostic exploration scheme in this context that achieves provable minimax optimality. 
Furthermore, once the sample size exceeds $\frac{S^2AH^3}{\varepsilon^2}$ episodes (up to log factor), our algorithm is able to yield $\varepsilon$ accuracy for arbitrarily many reward functions (even when they are adversarially designed),  
a task commonly dubbed as ``reward-free exploration.''  
The novelty of our algorithm design draws on insights from offline RL: the exploration scheme attempts to maximize a critical reward-agnostic quantity that dictates the performance of offline RL, 
while the policy learning paradigm leverages ideas from sample-optimal offline RL paradigms.

\end{abstract}


\setcounter{tocdepth}{2}
\tableofcontents

\section{Introduction}

Understanding the efficacy and efficiency of exploration 
is one of the central research lines in statistical reinforcement learning (RL), 
dating back to earlier works like \citet{lai1985asymptotically} from the multi-armed bandit literature.  
In truth, a wealth of RL applications anticipates automated decision making without prior knowledge of the environment, 
for which exploration serves as an essential component to acquire and improve knowledge of the unknowns.   
With the aim to maximize information gain and enhance data efficiency, 
a large strand of recent works sought to balance the need to explore (i.e., gaining new information about under-explored states and actions)  with the desire for exploitation (i.e., taking advantage of what is currently viewed to be favorable) 
in a statistically efficient manner 
\citep{kearns2002near,auer2006logarithmic,jaksch2010near,brafman2002r,azar2017minimax,jin2018q,zhang2020almost,li2021breaking,simchowitz2019non,lattimore2020bandit,zanette2019tighter}.

\subsection{Reward-agnostic exploration} 

Noteworthily, a dominant fraction of existing exploration schemes required prior information about the reward function, or at least information about those immediate rewards assocciated with the sampled state-action pairs.  
This requirement, however, becomes ill-suited  for applications that do not come with pre-specified reward functions. 
For instance, in offline/batch RL, new learning tasks need to be performed on the basis of pre-collected data, 
but the reward function of this new task might not be readily available during the data pre-collection process;  
in online recommendation systems, when a recommendation is presented to a user,  her feedback might not be instantaneously revealed, but instead be received in a considerably delayed manner.  
There is also no shortage of applications where the reward functions are frequently updated  to meet multiple objectives or encourage different behavior, 
but it is clearly undesirable to recollect all data from scratch whenever the reward functions change.   
The breadth of these applications underscores the need to re-examine reward-agnostic exploration, namely, 
{
\setlist{rightmargin=\leftmargin}
\begin{itemize}
	\item[]	{\em How to explore the unknown environment in the most statistically efficient way even when the learner is unaware of the reward function(s) during the exploration stage? 
		Can we perform exploration just once but still achieve efficiency for multiple unseen reward functions simultaneously? }
\end{itemize}
}

%\paragraph{Challenges.} 
%
The design of reward-agnostic exploration or pure exploration is, as one would expect, substantially more challenging than the reward-aware counterpart. 
For the most part, existing reward-aware exploration schemes prioritize exploration of the ``important'' part of the environment in a task-specific manner, 
which could often be ill-suited to a new task with drastically different rewards.   
To address this issue,  a natural remedy  is to attempt exploration of all non-negligible states/actions thoroughly, 
in a way that is {\em simultaneously} adequate for all possible reward functions. 
This constitutes a key principle underlying a recent strand of works that goes by the name of {\em reward-free exploration (RFE)} \citep{jin2020reward,zhang2021near,kaufmann2021adaptive}. 
In the problem of reward-free exploration, one targets a pure exploration scheme such that the samples collected via this scheme 
are effective uniformly over all possible reward functions (including the ones that are adversarially designed). 
To facilitate more concrete discussion, imagine we are interested in a finite-horizon inhomogeneous Markov decision process (MDP) with $S$ states, $A$ actions, and horizon length $H$. 
The state-of-the-art results on this front \citep{menard2021fast} asserted that: to enable reward-free exploration for an arbitrary set of reward functions, it suffices to sample\footnote{For any two non-negative functions $f$ and $g$, the notation $f\big(S,A,H,\frac{1}{\varepsilon},\frac{1}{\delta}\big)= O\big( g\big(S,A,H,\frac{1}{\varepsilon},\frac{1}{\delta}\big) \big)$ means that there exists a universal constant $C_1>0$ such that $f\big(S,A,H,\frac{1}{\varepsilon},\frac{1}{\delta}\big)\leq C_1 g\big(S,A,H,\frac{1}{\varepsilon},\frac{1}{\delta}\big)$. The notation $f\big(S,A,H,\frac{1}{\varepsilon},\frac{1}{\delta}\big)= \widetilde{O}\big( g\big(S,A,H,\frac{1}{\varepsilon},\frac{1}{\delta}\big) \big)$ is defined similarly except that it hides any logarithmic factor.} 
\begin{equation}
	\widetilde{O}\bigg( \frac{H^3S^2A}{\varepsilon^2} \bigg) \text{ episodes}
\qquad \text{(reward-free exploration)}
	\label{eq:sample-complexity-RFE}
\end{equation}
during the exploration stage, where $\varepsilon$ denotes the target accuracy level. Notably, this bound is valid regardless of how many reward functions are under consideration.

The aim set out in the above RFE works, however, could sometimes be overly conservative in practice.  
There are plenty of practical scenarios where only a small number (or no more than a polynomial number) of reward functions matter.  
If this were the case, then the sample complexity bound \eqref{eq:sample-complexity-RFE} could have been too conservative. 
Following this motivation, a recent work \citet{zhang2020task} came up with a model-free task-agnostic exploration algorithm that succeeds with
\begin{equation}
	\widetilde{O}\bigg( \frac{H^5SA}{\varepsilon^2} \bigg) \text{ episodes}, 
%\qquad \text{(reward-free exploration)}
	\label{eq:sample-complexity-TAE}
\end{equation}
as long as the total number of reward functions is no larger than $\mathsf{poly}(S,A,H)$. 
In some aspect, this results in significant sample size saving if one onlys cares about a reasonable number of reward functions, given that the scaling with the number of states is dramatically improved (i.e., from $S^2$ to $S$) . 
Nevertheless, this bound \eqref{eq:sample-complexity-TAE} remains inferior to \eqref{eq:sample-complexity-RFE} --- by a factor of $H^2$ --- when it comes to the horizon dependency. 
A natural question arises that motivates the present paper: 
can we hope to develop a pure exploration scheme that is statistically optimal in terms of all salient parameters (i.e., $S,A,H$) of the MDP, 
provided that there are only a few fixed reward functions of interest?

\newcommand{\topsepremove}{\aboverulesep = 0mm \belowrulesep = 0mm} \topsepremove

\begin{table}[t]

\begin{center}

\begin{tabular}{c|c|c|c}
\hline 
\toprule 
	\multirow{2}{*}{setting} & \multirow{2}{*}{paper} & \multirow{2}{*}{upper bound}  & number/type of   $\vphantom{\frac{1^{7}}{1^{7^7}}}$  \tabularnewline
	 &  &    &  reward functions  $\vphantom{\frac{1^{7}}{1^{7^7}}}$ \tabularnewline
\hline 
\toprule
	\multirow{3}{*}{reward-agnostic} & \citet{zhang2020task} & $\frac{H^{5}SA}{\varepsilon^{2}}$ $\vphantom{\frac{1^{7^{7^7}}}{1^{7^{7^7}}}}$ 
	 & \multirow{3}{*}{$\mathsf{poly}(H,S,A)$, fixed} \tabularnewline
\cline{2-3} 
	& this paper & $\frac{H^{3}SA}{\varepsilon^{2}}$ $\vphantom{\frac{1^{7^{7^7}}}{1^{7^{7^7}}}}$   &\tabularnewline
%\cline{2-4} \cline{3-4} \cline{4-4} 
% & lower bound & $\frac{H^{3}SA}{\varepsilon^{2}}$ $\vphantom{\frac{1^{7^{7^7}}}{1^{7^{7^7}}}}$ & $\mathsf{poly}(H,S,A)$, fixed\tabularnewline
\hline 
\toprule 
	\multirow{7}{*}{reward-free} & \citet{jin2020reward}  &  $\frac{H^{5}S^2A}{\varepsilon^{2}}$ $\vphantom{\frac{1^{7^{7^7}}}{1^{7^{7^7}}}}$ 
	 & \multirow{7}{*}{$\text{arbitrary}$} 
	 \tabularnewline
\cline{2-3} 
	& \citet{kaufmann2021adaptive} & $\frac{H^{4}S^2A}{\varepsilon^{2}}$ $\vphantom{\frac{1^{7^{7^7}}}{1^{7^{7^7}}}}$ & \tabularnewline
\cline{2-3} 
	& \citet{menard2021fast} & $\frac{H^{3}S^2A}{\varepsilon^{2}}$ $\vphantom{\frac{1^{7^{7^7}}}{1^{7^{7^7}}}}$ & \tabularnewline
\cline{2-3} 
	& \citet{qiao2022sample} & $\frac{H^{5}S^2A}{\varepsilon^{2}}$ $\vphantom{\frac{1^{7^{7^7}}}{1^{7^{7^7}}}}$ & \tabularnewline
\cline{2-3} 
 & this paper & $\frac{H^{3}S^2A}{\varepsilon^{2}}$ $\vphantom{\frac{1^{7^{7^7}}}{1^{7^{7^7}}}}$ & \tabularnewline
%\cline{2-4} \cline{3-4} \cline{4-4} 
% & lower bound & $\frac{H^{3}S^2A}{\varepsilon^{2}}$ $\vphantom{\frac{1^{7^{7^7}}}{1^{7^{7^7}}}}$ & $\text{arbitrary}$\tabularnewline
\hline 
\toprule 
\end{tabular}

\end{center}
	\caption{ Sample complexity comparisons for reward-agnostic/reward-free algorithms on episodic inhomogeneous finite-horizon MDPs. 
	For ease of presentation, we omit all logarithmic factors and lower-order terms. 
	\label{tab:prior-work}}  
\end{table}

\subsection{This paper} 
One focus of this paper is to design a sample-efficient exploration algorithm, for a scenario where (i) there might be one or more fixed reward functions of interest, 
 (ii) all reward functions are invisible to the learner during the exploration stage, and (iii) we are only allowed to collect data once but are asked to accommodate all these reward functions simultaneously.  
Given that prior research on reward-free exploration mostly required reliable learning uniformly over all possible reward functions, 
the current paper adopts the terminology {\em reward-agnostic exploration} to differentiate it from RFE, allowing for the possibility that only a small number of reward functions count. 
To be more specific, the scenario studied herein consists of two stages: 
(a) the {\em exploration stage}: in which the learner explores the environment by collecting sample episodes in the absence of any reward information; 
(b) the {\em policy learning stage}: in which the learner is informed of the reward functions and computes a desired policy for each reward function of interest. 
Focusing on inhomogeneous finite-horizon MDPs as previously mentioned,  our main contributions can be summarized below. 
\begin{itemize}
	\item {\em Reward-agnostic exploration.} Imagine that there are no more than $\mathsf{poly}(S,A,H)$ reward functions of interest. We put forward a reward-agnostic exploration scheme, which paired with a policy learning algorithm allows us to find 
		an $\varepsilon$-optimal policy with a collection of  
\begin{equation}
	\widetilde{O}\bigg( \frac{H^3SA}{\varepsilon^2} \bigg) \text{ episodes}, 
%\qquad \text{(reward-free exploration)}
	\label{eq:sample-complexity-RAE-ours}
\end{equation}
ignoring the lower-order term. This sample complexity is essentially un-improvable --- in a minimax sense --- even when the reward function is {\em a priori} known \citep{domingues2021episodic,jin2018q}.

	\item {\em Reward-free exploration.} As it turns out, the proposed algorithm is also sample-efficient in the context of reward-free exploration; more precisely, it is able to return an $\varepsilon$-optimal policy uniformly over all sorts of reward functions (including those designed adversarially) with a sample size no larger than 
\begin{equation}
	\widetilde{O}\bigg( \frac{H^3S^2A}{\varepsilon^2} \bigg) \text{ episodes}. 
%\qquad \text{(reward-free exploration)}
	\label{eq:sample-complexity-RFE-ours}
\end{equation}
This sample complexity bound matches that of the best-performing RFE algorithms (i.e., \citet{menard2021fast}) proposed in the literature. 

\end{itemize}
\noindent More detailed comparisons between our results and past works are summarized in Table~\ref{tab:prior-work}.

Before concluding, we remark that the crux of our approach is to leverage ideas from recent development in offline/batch RL, with the following key elements.  
\begin{itemize}
	\item In the exploration stage, we identify a sort of reward-irrelevant quantities --- determined mainly by the occupancy distributions --- that dictate the performance of offline policy learning.  
		The exploration policy is then selected adaptively in an attempt to optimize such reward-irrelevant quantities. 

	\item In the policy learning stage, 
		we invoke a minimax optimal offline RL algorithm --- namely, a pessimistic model-based approach \citep{li2022settling} --- that computes a near-optimal policy based on the samples collected in the exploration stage as well as the revealed reward information.  

\end{itemize}

%\citep{schmidhuber1991curious}

\subsection{Notation} 

For any set $\mathcal{X}$, we denote by $\Delta(\mathcal{X})$ the probability simplex over $\mathcal{X}$. For any integer $m$, we define $[m]\coloneqq \{1,\ldots,m\}$. 
For any distribution $\rho\in \Delta(\cS)$ and any vector $V\in \mathbb{R}^S$, we define the associated variance as
\begin{equation}
	\mathsf{Var}_{\rho}(V) \coloneqq \sum_{s\in \cS} \rho(s) \big(V(s)\big)^2 - \Big( \sum_{s\in \cS} \rho(s) V(s) \Big)^2.
\end{equation}
For any two non-negative functions $f$ and $g$, 
the notation $f\big(S,A,H,\frac{1}{\varepsilon},\frac{1}{\delta}\big) \lesssim  g\big(S,A,H,\frac{1}{\varepsilon},\frac{1}{\delta}\big) $  
means that there exists a universal constant $C_1>0$ such that $f\big(S,A,H,\frac{1}{\varepsilon},\frac{1}{\delta}\big)\leq C_1 g\big(S,A,H,\frac{1}{\varepsilon},\frac{1}{\delta}\big)$; 
the notation $f\big(S,A,H,\frac{1}{\varepsilon},\frac{1}{\delta}\big) \gtrsim  g\big(S,A,H,\frac{1}{\varepsilon},\frac{1}{\delta}\big) $ 
indicates the existence of some universal constant $C_2>0$ such that $f\big(S,A,H,\frac{1}{\varepsilon},\frac{1}{\delta}\big)\geq C_2 g\big(S,A,H,\frac{1}{\varepsilon},\frac{1}{\delta}\big)$;   
and the notation  $f\big(S,A,H,\frac{1}{\varepsilon},\frac{1}{\delta}\big) \asymp g\big(S,A,H,\frac{1}{\varepsilon},\frac{1}{\delta}\big) $ 
means that  $f\big(S,A,H,\frac{1}{\varepsilon},\frac{1}{\delta}\big) \lesssim g\big(S,A,H,\frac{1}{\varepsilon},\frac{1}{\delta}\big) $ 
and  $f\big(S,A,H,\frac{1}{\varepsilon},\frac{1}{\delta}\big) \gtrsim  g\big(S,A,H,\frac{1}{\varepsilon},\frac{1}{\delta}\big) $ hold simultaneously.

\section{Problem formulation}
\label{sec:background}

In this section, we introduce some preliminaries for Markov decision processes, and formulate the problem. 

\paragraph{Basics of Markov decision processes.} 
We study finite-horizon MDPs with horizon length $H$ and inhomogeneous probability transition kernels. 
The state space of the MDP is denoted by $\mathcal{S}=\{1,\ldots,S\}$, containing $S$ different states; 
the action space is denoted by $\mathcal{A}=\{1,\ldots,A\}$, comprising $A$ possible actions; 
we also let $P= \{P_h\}_{1\leq h\leq H}$ (with $P_h: \mathcal{S}\times \mathcal{A} \rightarrow \Delta(\mathcal{S})$) represent the inhomogeneous probability transition kernel, 
such that taking action $a$ in state $s$ at step $h$ results in a new state at step $h+1$ following the distribution $P_h(\cdot \mymid s,a)$.  
For notational convenience, we shall often adopt the shorthand notation
\begin{equation}
	P_{h,s,a} \coloneqq P_h(\cdot \mymid s,a) \in \Delta(\mathcal{S}) .
\end{equation}
A policy $\pi$ is a (possibly randomized) strategy for action selection. 
In particular, when policy $\pi$ is said to be Markovian (so that the action selection rule at step $h$ depends only on the state at this step), 
we often represent it as
$\pi = \{\pi_h \}_{1\leq h\leq H}$, 
where $\pi_h$ is a mapping from $\mathcal{S}$ to $\Delta(\mathcal{A})$, 
with $\pi_h(\cdot\mymid s)\in \Delta(\mathcal{A})$ specifying how to choose actions at step $h$.  
When $\pi$ is a deterministic policy, we overload the notation by letting $\pi_h(s)$ denote the action chosen by $\pi$ in state $s$ at step $h$. 
Another ingredient that merits particular attention is the reward function, 
which does not affect how the MDP evolves. For any reward function $r=\{r_h\}_{1\leq h\leq H}$ with $r_h: \mathcal{S} \times \mathcal{A} \rightarrow [0,1]$, 
the element $r_h(s,a)$ indicates the immediate reward gained in state $s$ and step $h$ while executing action $a$;   
here, we follow  the convention by assuming $r_h(s,a)\in [0,1]$ for each $(s,a,h)$ triple. 
The readers are also referred to standard textbooks like \citet{bertsekas2017dynamic} for more backgrounds.

For any policy $\pi$, the value function $V^{\pi}=\{V_h^{\pi}\}_{1\leq h\leq H}$ (with $V^{\pi}_h: \cS\rightarrow \mathbb{R}$) and the Q-function $Q^{\pi}=\{Q_h^{\pi}\}_{1\leq h\leq H}$ (with $Q^{\pi}_h: \cS\times \cA\rightarrow \mathbb{R}$)  associated with $\pi$ at step $h$ are defined and denoted by
\begin{subequations}
\begin{align}
	\forall s\in\mathcal{S}: \qquad  &V_h^\pi(s)\coloneqq \mathbb{E} \left[\sum_{t=h}^{H} r_t (s_t,a_t) \,\Big |\, s_h=s; \pi \right], \\
	\forall (s,a)\in\mathcal{S}\times\mathcal{A}: \qquad  &Q_h^\pi(s,a)\coloneqq \mathbb{E} \left[\sum_{t=h}^{H} r_t (s_t,a_t) \,\Big |\, s_h=s, a_h=a; \pi \right],
\end{align}
\end{subequations}
where the expectation is taken over the randomness of a sample trajectory $\{(s_t,a_t)\}_{t=h}^H$ induced by the underlying MDP under policy $\pi$. 
In words, the value function (resp.~Q-function) quantifies the expected cumulative reward starting from step $h$ conditioned on a given state (resp.~state-action pair) at step $h$. 
It is well-known that there exists at least one deterministic policy --- denoted by $\pi^{\star}=\{\pi_h^{\star} \}_{1\leq h\leq H}$ throughout --- that simultaneously maximizes the value function and the Q-function for all $(s,a,h)\in \cS\times \cA\times [H]$. 
Here and below, we shall refer to $\pi^{\star}$ as the optimal deterministic policy, and denote the optimal value function $V^{\star}=\{V^{\star}_h\}_{1\leq h\leq H}$ and optimal Q-function  $Q^{\star}=\{Q^{\star}_h\}_{1\leq h\leq H}$ respectively as follows: 
\begin{align}
	V^{\star}_h(s) =   V^{\pi^{\star}}_h(s) = \sup_{\pi} V^{\pi}_h(s)
	\quad \text{and} \quad
	Q^{\star}_h(s,a) =   Q^{\pi^{\star}}_h(s,a) = \sup_{\pi} Q^{\pi}_h(s,a),
	\quad \forall (s,a,h)\in \cS\times \cA\times [H]. 
\end{align}
When the initial state is drawn from a state distribution $\rho \in \Delta(\cS)$, we denote 
\begin{equation}
	V_1^{\pi}(\rho) \coloneqq \mathop{\mathbb{E}}_{s\sim \rho} \big[ V_1^\pi(s) \big]
	\qquad \text{and} \qquad
	V_1^{\star}(\rho) \coloneqq \mathop{\mathbb{E}}_{s\sim \rho} \big[ V_1^{\star}(s) \big]. 
\end{equation}
Additionally, we find it convenient to define the following occupancy distributions associated with policy $\pi$ at step $h$: 
\begin{subequations}
\begin{align}
	\forall s\in\mathcal{S}: \qquad  &d_h^\pi(s)\coloneqq \mathbb{P} \left(s_h=s \mymid s_1\sim \rho; \pi \right), \\
	\forall (s,a)\in\mathcal{S}\times\mathcal{A}: \qquad  &d_h^\pi(s,a)\coloneqq \mathbb{P} \left(s_h=s,a_h=a \mymid s_1\sim\rho;\pi \right),
\end{align}
\end{subequations}
where the sample trajectory $\{(s_h,a_h)\}_{h=1}^H$ is generated according to the initial state distribution $s_1\sim \rho$ and policy $\pi$.
Evidently, the occupancy state distribution for step $h=1$ reduces to
\begin{align}
	\forall s\in\mathcal{S}: \qquad  &d_1^\pi(s) = \rho(s).
\end{align}

\paragraph{Learning processes and goals.}  
The learning process considered in this paper can be split into two separate stages. 
In the first stage --- called the {\em exploration stage} --- the learner executes the MDP sequentially for $N_{\mathsf{tot}}$ times to collect $N_{\mathsf{tot}}$ sample episodes each of length $H$,
 using any exploration policies it selects. 
 Throughout this paper, each sample episode starts from an initial state independently generated from a distribution $\rho \in \Delta(\cS)$, 
 where $\rho$ is unknown to the learner. 
A key constraint, however, is that no reward information whatsoever can be utilized to guide data collection.  
In the second stage --- termed the {\em policy learning stage} --- the learner attempts to compute a policy on the basis of these $N_{\mathsf{tot}}$ episodes; 
no additional sampling is permitted during this stage.

The performance metric we focus on is the sample complexity; 
more precisely, given an initial state distribution $\rho$, a target accuracy level $\varepsilon$ and a reward function of interest, 
the sample complexity of an algorithm refers to the minimum number of sample episodes (during the exploration stage) that allows one to achieve $V^{\star}(\rho) - V^{\widehat{\pi}}(\rho) \leq \varepsilon$, with $\widehat{\pi}$ denoting the policy estimated by this algorithm. 
Depending on the number/type of reward functions of interest, 
we study the following two scenarios and use different names to distinguish them.

\begin{itemize}

\item {\em Reward-agnostic exploration.} 
In this case, 
we suppose that  there exist $m_{\mathsf{reward}}$ given reward functions of interest, generated independently from the data samples.  
We are asked to compute a desirable policy for each of these reward functions, using the same set of data collected during the exploration stage.   
The aim is to achieve a sample complexity that scales optimally in $S,A,H$ and slowly in $m_{\mathsf{reward}}$.

\item {\em Reward-free exploration.} 
This case assumes the presence of an arbitrarily large number $m_{\mathsf{reward}}$ of reward functions (e.g., $m_{\mathsf{reward}}$ could be (super)-exponential in $S$, $A$ and $H$ or even unbounded), 
where each reward function can even be adversarially chosen. 
We seek to develop an exploration and learning paradigm that achieves a desirable sample complexity independent from $m_{\mathsf{reward}}$, possibly at the price of larger scaling in other parameters like $S$. 
The statistical feasibility of this goal stems from a basic observation: if we have available enough samples to ensure accurate learning of the transition kernel, 
then one can find reliable policies regardless of what reward functions are in use.

\end{itemize}

\section{Algorithm}

In this section, we present the proposed procedure, followed by some intuitive explanation of the design rationale.   
Here and throughout, we denote 
\begin{align}
	\Pi ~\coloneqq \text{ the set of all deterministic policies.}
	\label{eq:defn-Pi-set}
\end{align}
%
%the set of all deterministic policies. 

\subsection{A two-stage algorithm}

On a high level, the algorithm we propose consists of the following two stages: 
\begin{itemize}
	\item {\em Stage 1: reward-agnostic exploration.} In this stage, we take samples in the hope of exploring the unknown environment 
		adequately.  No reward information is available in this stage. 
	\begin{itemize}
		\item {\em Stage 1.1: estimating occupancy distributions.} 
			For each step $ h\in [H]$, 
			we draw $N$ episodes of samples, which allow us to estimate the occupancy distributions at this step induced by any policy. 

		\item {\em Stage 1.2: computing a desirable behavior policy and drawing samples.} Equipped with our estimated occupancy distributions, 
			we compute a desirable behavior policy $\behavior$ --- taking the form of a finite mixture of deterministic policies ---  
			and employ it to draw $K$ episodes of samples.  
	\end{itemize}

	\item {\em Stage 2: policy learning via offline RL.} 
	 The reward functions are revealed in this stage, but no more samples can be obtained.   
	Our algorithm then learns, for each reward function, a near-optimal policy by applying a model-based offline RL algorithm to the $K$ episodes drawn in Stage 1.2.  

\end{itemize}
The total number of episodes of exploration in our algorithm is therefore $N_{\mathsf{tot}} = K+NH$, where all episodes are collected in Stage 1. 
In the sequel, we elaborate on the algorithm details, with the intuition explained in Section~\ref{sec:intuition}.

\paragraph{Initialization.} 
Let us begin by describing how to initialize our algorithm. 
	Generate $N$ sample episodes independently each of length 1. 
	As mentioned previously, the initial states of these $N$ episodes are i.i.d.~drawn from the distribution 
	$\rho$, denoted by 
	%$\{s_1^n\}_{1 \le n \le N}$ with 
	%
	\[
		s_1^{n,0} \overset{\text{i.i.d.}}{\sim} \rho, 
		\qquad 1\leq n\leq N. 
	\]
	For any policy $\pi$, we set the empirical occupancy distribution induced by $\pi$ for step $h=1$ as follows:  
\begin{subequations}
\label{eq:init-occupancy}
\begin{align}
	\widehat{d}_{1}^{\pi}(s) &= \frac{1}{N}\sum_{n = 1}^N \ind(s_1^{n,0} = s), \qquad &&\forall s\in \cS; \\
	\widehat{d}_{1}^{\pi}(s, a) &= \widehat{d}_{1}^{\pi}(s) \pi_1(a\mymid s),\qquad &&\forall (s,a)\in \cS \times \cA.   
\end{align}
\end{subequations}
%
%Here, $\mathcal{T}_{\xi}(\cdot)$ is a hard-thresholding operator that supresses small input, namely, 
%%
%\begin{equation}
%	\mathcal{T}_{\xi}(x) \coloneqq x\ind(x > \xi),\qquad\text{with }\xi = \frac{c_{\xi}H^3S^2A^2\log\frac{HSA}{\delta}}{N}
%	\label{eq:hard-threshold-op}
%\end{equation}
%%
%for some universal constant $c_{\xi}>0$ large enough. 

\paragraph{Stage 1.1: estimating occupancy distributions.}  
	Armed with the estimates \eqref{eq:init-occupancy} for the initial step, 
	we would like to continue estimating the occupancy distributions for the remaining steps. 
	Towards this end, we proceed in a forward manner from $h=1,\ldots, H-1$.  
	Suppose we already have access to the empirical occupancy distribution $\widehat{d}_{h}^{\pi}$ 
	w.r.t.~every policy $\pi$ at step $h$. 
	We then attempt estimation for the next step $h+1$ by acquiring a further set of samples, 
	as described below.  
	\begin{itemize}
		\item {\em Select an exploration policy.}   Approximately solve the following convex program:  
		\begin{align}
		%\widehat{\mu}_{h} \approx \arg\max_{\mu\in \Delta(\Pi)}\sum_{(s, a)\in \cS\times \cA} \log\bigg[\frac{1}{KH} + \sum_{\pi\in \Pi} \mu(\pi)\widehat{d}_{h}^{\pi}(s, a) \bigg], 
			\widehat{\mu}^{h}\approx\arg\max_{\mu\in\Delta(\Pi)}\sum_{(s,a)\in\cS\times\cA}\log\bigg[\frac{1}{KH}+ \mathop{\mathbb{E}}_{\pi\sim\mu}\big[\widehat{d}_{h}^{\pi}(s,a)\big]\bigg], 
			\label{defi:target-empirical-h}
		\end{align}
		whose rationale will be elucidated in Section~\ref{sec:intuition}.
		This leads to our exploration policy that assists in model estimation for step $h+1$: 
		\begin{align}
			\piexploreh = \mathbb{E}_{\pi\sim\widehat{\mu}^h} [\pi]. 
			\label{eq:defn-pi-explore-h}
		\end{align}
		Here, $\piexploreh$ is a mixture of deterministic policies, 
			with the weight vector 
			$\widehat{\mu}^{h}$ chosen to solve an ``infinite-dimensional'' optimization problem in \eqref{defi:target-empirical-h}. 
		Note, however, that instead of finding an exact solution (which could have an infinite support size and 
			be computationally infeasible),  
			we will introduce a tractable optimization-based subroutine in Section~\ref{sec:subroutines} 
			to yield an approximate solution with finite support.

		\item {\em Sampling and estimation.}  
			We collect $N$ independent episodes each of length $h+1$ --- denoted by  
			$\big\{ s_1^{n,h}, a_1^{n,h}, s_2^{n,h}, a_2^{n,h}, \ldots, s_{h+1}^{n,h} \big\}_{1\leq n\leq N}$ 
			--- 
			using the exploration policy $\piexploreh$ (cf.~\eqref{eq:defn-pi-explore-h}).  
			% where $\{(s_t^{n,h}, a_t^{n,h})\}_{1\leq t\leq H}$ denotes the $n$-th sample trajectory. 
			Aggregate all sample transitions at the $h$-th step to construct $\widehat{P}_{h}: \cS\times \cA \times \cS \rightarrow \mathbb{R}$ such that
\begin{align}
	\widehat{P}_{h}(s^{\prime} \mymid s, a) = 
	\frac{\ind(N_h(s, a) > \xi)}{\max\big\{ N_h(s, a),\, 1 \big\} }\sum_{n = 1}^N \ind(s_{h}^{n,h} = s, a_{h}^{n,h} = a, s_{h+1}^{n,h} = s^{\prime})  
%	\qquad \forall (s,a,s')\in \cS\times \cA\times \cS, 
	\label{eq:empirical-P-step-h}
	%\pi_{h+1}(a \mymid s).
\end{align}
for all $(s,a,s')\in \cS\times \cA\times \cS$; here, we take
\begin{equation}
	N_h(s, a) = \sum_{n = 1}^N \ind\big( s_{h}^{n,h} = s, a_{h}^{n,h} = a \big), 
	\qquad \forall (s,a)\in \cS\times \cA,  
\end{equation}
and  $\xi$ represents the following pre-specified quantity  
\begin{equation}
	\xi = c_{\xi}H^3S^3A^3\log\frac{HSA}{\delta}
	\label{eq:hard-threshold}
\end{equation}
for some large enough universal constant $c_{\xi}>0$.  
In words, $\widehat{P}_{h}(\cdot \mymid s, a)$ reflects the empirical transition frequency as long as there are sufficient samples visiting $(s,a)$ at step $h$; 
			otherwise it is simply set to zero (so in this sense, $\widehat{P}_{h}$ is not a transition kernel itself). 
With the above empirical estimates \eqref{eq:empirical-P-step-h} in place, 
we can compute, for any deterministic policy $\pi$, the following estimate of the empirical occupancy distribution induced by $\pi$ for step $h+1$: 
\begin{subequations}
\label{eq:d_empirical} 	
\begin{align}
	\widehat{d}_{h+1}^{\pi}(s) &= \big\langle \widehat{P}_{h}(s \mymid \cdot, \cdot), \,
	\widehat{d}_{h}^{\pi} (\cdot, \cdot) \big\rangle,  \qquad && \forall s\in \cS,\\
	\widehat{d}_{h+1}^{\pi}(s,a) &= \widehat{d}_{h+1}^{\pi}(s) \pi_{h+1}(a\mymid s),
	&& \forall (s,a)\in \cS\times \cA.
\end{align}
\end{subequations}
%
%where $\mathcal{T}_{\xi}(\cdot)$ is defined in \eqref{eq:hard-threshold-op}. 

\end{itemize}

\paragraph{Stage 1.2: computing a behavior policy and drawing samples.}  
Armed with the expressions of the estimated occupancy distribution $\{\widehat{d}_{h}^{\pi}\}$ for any $\pi\in \Pi$, 
we propose to solve
\begin{align}
	\behavior  
	\approx 
	\arg\max_{\mu\in\Delta(\Pi)}\left\{ \sum_{h=1}^{H}\sum_{(s,a)\in\cS\times\cA}\log\bigg[\frac{1}{KH}+\mathbb{E}_{\pi\sim\mu}\big[\widehat{d}_{h}^{\pi}(s,a)\big]\bigg]\right\} .
	\label{defi:target-empirical}
\end{align}
which resembles \eqref{defi:target-empirical-h} except that the objective function now involves summation over all $h$; 
as before, it will be approximately solved by means of a tractable subroutine (to be described momentarily). 
We then set
\begin{equation}
	\widehat{\pi}_{\mathsf{b}} = \mathbb{E}_{\pi\sim\widehat{\mu}_{\mathsf{b}}} [\pi]
	\label{eq:final-behavior-policy}
\end{equation}
as a behavior policy,  
and sample  $K$ independent sample trajectories each of length $H$ --- denoted by $\big\{ s_1^{n,\mathsf{b}}, a_1^{n,\mathsf{b}}, s_2^{n,\mathsf{b}}, a_2^{n,\mathsf{b}}, \ldots, s_H^{n,\mathsf{b}} \big\}_{1\leq n\leq K}$.

\paragraph{Stage 2: policy learning via offline RL.} 
With the above $K$ episodes $\big\{ s_1^{n,\mathsf{b}}, a_1^{n,\mathsf{b}}, s_2^{n,\mathsf{b}}, a_2^{n,\mathsf{b}}, \ldots, s_H^{n,\mathsf{b}} \big\}_{1\leq n\leq K}$ at hand, 
we compute the final policy estimate  $\widehat{\pi}$ by running a sample-efficient offline RL algorithm.  
A candidate offline RL algorithm is a pessimistic model-based algorithm studied in \citet{li2022settling}. 
There is a slight difference here: based on our empirical estimates of the occupancy distributions, 
we can readily calculate the following quantity: 
\begin{align} \label{eq:lower-samples}
	\widehat{N}_h^{\mathsf{b}}(s, a) = \bigg[\frac{K}{4} \mathop{\mathbb{E}}_{\pi\sim\behavior}\big[\widehat{d}_{h}^{\pi}(s,a)\big] - \frac{K\xi}{8N} - 3\log\frac{HSA}{\delta}\bigg]_+,
\end{align}
where $[x]_+ \coloneqq \max\{x, 0\}$. As we shall prove later in Section~\ref{sec:proof-main}, $\widehat{N}_h^{\mathsf{b}}(s, a)$ serves as --- with high probability --- a lower bound on the total number of visits to $(s,a,h)$ among these $K$ sample episodes, which will be employed to subsample the sample transitions (instead of exploiting two-fold sample splitting as in \citet{li2022settling}) and construct lower confidence bounds. 
%. \yly{This lower bound will be employed to construct certain lower confidence bounds that will be used to penalize the value iteration algorithm.}
Precise descriptions of this offline RL algorithm can be found in Appendix~\ref{sec:offline-algorithm}. 

% propose to run a model-based offline RL algorithm based on the above $K$ sample trajectories $\big\{ s_1^{n,\mathsf{b}}, a_1^{n,\mathsf{b}}, s_2^{n,\mathsf{b}}, a_2^{n,\mathsf{b}}, \ldots, s_H^{n,\mathsf{b}} \big\}_{1\leq n\leq K}$, and 
%	output a final policy $\widehat{\pi}$. 

\bigskip
\noindent
The whole procedure of the proposed algorithm is summarized in Algorithm~\ref{alg:main}.

\begin{algorithm}[t]
	\DontPrintSemicolon
	\SetNoFillComment
	\textbf{Input:} state space $\mathcal{S}$, action space $\mathcal{A}$, horizon length $H$, initial state distribution $\rho$, target success probability $1-\delta$, 
	threshold $\xi = c_{\xi}H^3S^3A^3\log(HSA/\delta)$.   \\
	\tcc{Stage 1: reward-agnostic exploration}
	\vspace{-0.3ex}
	\tcc{Stage 1.1: estimating occupancy distributions}
	Draw $N$ i.i.d.~initial states $s_1^{n,0} \overset{\mathrm{i.i.d.}}{\sim} \rho $ $(1\leq n\leq N)$, and define the following functions
	\vspace{-1.5ex}
	\begin{equation}
		\label{eq:hat-d-1-alg}
		\widehat{d}_1^{\pi}(s) = \frac{1}{N}\sum_{n=1}^{N}\ind\{s_1^{n,0}=s\}, \qquad \widehat{d}_1^\pi (s,a)=\widehat{d}_1^\pi (s) \pi_1(a\mymid s)
	\vspace{-1ex}
	\end{equation}
	for any deterministic policy $\pi: \mathcal{S}\times[H]\to\Delta(\mathcal{A})$ and any $(s,a)\in \cS\times \cA$. (Note that these functions are defined for future purpose and not computed for the moment, as we have not specified policy $\pi$.) \\
	\For{$ h = 1$ \KwTo $H-1$}{
		Call Algorithm \ref{alg:sub1} to compute an exploration policy $\pi^{\mathsf{explore},h}$. \\
		Draw $N$ independent trajectories $\{s_1^{n,h},a_1^{n,h},\dots,s_{h+1}^{n,h}\}_{1\leq n\leq N}$
		using policy $\pi^{\mathsf{explore},h}$ and compute
		\vspace{-1.5ex}
		\[
		\widehat{P}_{h}(s^{\prime} \mymid s, a) = 
		\frac{\ind(N_h(s, a) > \xi)}{\max\big\{ N_h(s, a),\, 1 \big\} }\sum_{n = 1}^N \ind(s_{h}^{n,h} = s, a_{h}^{n,h} = a, s_{h+1}^{n,h} = s^{\prime}),  
		\qquad \forall (s,a,s')\in \cS\times \cA\times \cS, 
		\vspace{-1ex}
		\]
		where $N_h(s, a) = \sum_{n = 1}^N \ind\{ s_{h}^{n,h} = s, a_{h}^{n,h} = a \}$.  \\
		For any deterministic policy $\pi: \mathcal{S}\times[H]\to\Delta(\mathcal{A})$ and any $(s,a)\in \cS\times \cA$, define
		\vspace{-1.5ex}
		\begin{equation}
			\label{eq:hat-d-h-alg}
			\widehat{d}_{h+1}^{\pi}(s) = \big\langle \widehat{P}_{h}(s \mymid \cdot, \cdot), \,
			\widehat{d}_{h}^{\pi} (\cdot, \cdot) \big\rangle,  \qquad  
			\widehat{d}_{h+1}^{\pi}(s,a) = \widehat{d}_{h+1}^{\pi}(s) \pi_{h+1}(a\mymid s) .
			\vspace{-1ex}
		\end{equation}
		}
	\tcc{Stage 1.2: computing a behavior policy and drawing samples}	
	Call Algorithm \ref{alg:sub2} to compute a behavior policy $\pi_{\mathsf{b}}$.\\
	Draw $K$ independent sample trajectories $\big\{ s_1^{n,\mathsf{b}}, a_1^{n,\mathsf{b}}, s_2^{n,\mathsf{b}}, a_2^{n,\mathsf{b}}, \ldots, s_H^{n,\mathsf{b}} \big\}_{1\leq n\leq K}$ using policy $\pi_{\mathsf{b}}$.\\
	\tcc{Stage 2: policy learning via offline RL}
	For each reward function of interest, 
	%receive rewards $\big\{r_1^{n,\mathsf{b}},\ldots,r_{H-1}^{n,\mathsf{b}}\big\}_{1\leq n\leq K}$ 
	call Algorithm \ref{alg:offline_RL} to compute a policy estimate $\widehat{\pi}$. \\
	\textbf{Output:} the policy estimate $\widehat{\pi}$ for each reward function of interest.
	\caption{Reward-agnostic RL: main algorithm.\label{alg:main}}
\end{algorithm}

\begin{algorithm}[t]
	\DontPrintSemicolon
	\SetNoFillComment
	\vspace{-0.3ex}
	\textbf{Initialize}: $\mu^{(0)}=\delta_{\pi_{\mathsf{init}}}$ for an arbitrary policy $\pi_{\mathsf{init}}\in\Pi$, $T_{\max}=\lfloor50SA\log(KH)\rfloor$. \\
	\For{$ t = 0$ \KwTo $T_{\max}$}{
	Compute the optimal deterministic policy $\pi^{(t),\mathsf{b}}$ of the MDP $\mathcal{M}^h_{\mathsf{b}}=(\cS \cup \{s_{\com}\},\cA,H,\widehat{P}^{\com, h},r_{\mathsf{b}}^h)$, 
		where $r_{\mathsf{b}}^h$ is defined in \eqref{eq:reward-function-mub-h}, and $\widehat{P}^{\com, h}$ is defined in \eqref{eq:augmented-prob-kernel-h}; 
		let $\pi^{(t)}$ be the corresponding optimal deterministic policy 
		of $\pi^{(t),\mathsf{b}}$  in the original state space. 
		\tcp{find the optimal policy} 
	Compute \tcp{choose the stepsize}
	\vspace{-2ex}
	\[
	\alpha_{t}=\frac{\frac{1}{SA}g(\pi^{(t)},\widehat{d},\mu^{(t)})-1}{g(\pi^{(t)},\widehat{d},\mu^{(t)})-1}, \quad\text{where}\quad 
	g(\pi, \widehat{d}, \mu)=\sum_{(s,a)\in \cS\times\cA}\frac{\frac{1}{KH}+\widehat{d}_{h}^{\pi}(s,a)}{\frac{1}{KH}+\mathbb{E}_{\pi\sim\mu}[\widehat{d}_h^{\pi}(s,a)]}.
	\vspace{-2ex}			
	\]
	Here, $\widehat{d}_h^{\pi}(s,a)$ is computed via \eqref{eq:hat-d-1-alg} for $h=1$, and \eqref{eq:hat-d-h-alg} for $h\geq 2$.
	%and $\widehat{d}_{h}^{\mu}(s,a)=\mathbb{E}_{\pi\sim\mu}[\widehat{d}_h^{\pi}(s,a)]$.
	\\
	\vspace{0.2ex}
	If $g(\pi^{(t)},\widehat{d},\mu^{(t)})\leq 2SA$ then \textsf{exit for-loop}.  \label{line:stopping-subroutine-h}
	\tcp{stopping rule} 
	Update
	\tcp{Frank-Wolfe update} 
	\vspace{-0.5ex}
	\[
	\mu^{(t+1)}=\left(1-\alpha_t\right)\mu^{(t)}+\alpha_t \ind_{\pi^{(t)}}.
	\vspace{-1ex}
	\]\\
	}
	\textbf{Output:}  the exploration policy $\pi^{\mathsf{explore},h}=\mathbb{E}_{\pi\sim \mu^{(t)}}[\pi]$ and the weight $\widehat{\mu}^h=\mu^{(t)}$. 
	\label{line:output-Alg-sub1}
	\caption{Frank-Wolfe-type subroutine for solving \eqref{defi:target-empirical-h} for step $h$.\label{alg:sub1}}
\end{algorithm}

\begin{algorithm}[t]
	\DontPrintSemicolon
	\SetNoFillComment
	\vspace{-0.3ex}
	\textbf{Initialize}: $\mu^{(0)}_{\mathsf{b}}=\delta_{\pi_{\mathsf{init}}}$ for an arbitrary policy $\pi_{\mathsf{init}}\in\Pi$, $T_{\max}=\lfloor50SAH\log(KH)\rfloor$. \\
	\For{$ t = 0$ \KwTo $T_{\max}$}{
		Compute the optimal deterministic policy $\pi^{(t),\mathsf{b}}$  of the MDP $\mathcal{M}_{\mathsf{b}}=(\cS \cup \{s_{\com}\},\cA,H,\widehat{P}^{\com},r_{\mathsf{b}})$, 
		where $r_{\mathsf{b}}$ is defined in \eqref{eq:reward-function-mub}, and $\widehat{P}^{\com}$ is defined in \eqref{eq:augmented-prob-kernel}; 
		let $\pi^{(t)}$ be the corresponding optimal deterministic policy 
		of $\pi^{(t),\mathsf{b}}$  in the original state space. 
		\label{line:pit-alg:sub2}
		\tcp{find the optimal policy}
		%\vspace{-1.5ex}
		%\begin{align*}
		%	r_{\mathsf{b},h}(s,a)=
		%	\begin{cases}
		%		\frac{1}{\frac{1}{KH}+\mathbb{E}_{\pi\sim\mu^{(t)}_{\mathsf{b}}}\big[\widehat{d}_{h}^{\pi}(s,a)\big]}, 
		%		\qquad &\forall (s,a,h)\in \cS\times \cA\times [H], \\
		%		0, & \forall (s,a,h) \in \{s^{\com}\} \times \cA\times [H]. 
		%	\end{cases}
		%\vspace{-1ex}		
		%\end{align*}\\
		Compute \tcp{choose the stepsize}
		\vspace{-1.5ex}
		\[
		\alpha_{t}=\frac{\frac{1}{SAH}g(\pi^{(t)},\widehat{d},\mu^{(t)}_{\mathsf{b}})-1}{g(\pi^{(t)},\widehat{d},\mu^{(t)}_{\mathsf{b}})-1}, \quad\text{where}\quad 
		g(\pi,\widehat{d},\mu)=\sum_{h=1}^{H}\sum_{(s,a)\in\cS\times\cA}\frac{\frac{1}{KH}+\widehat{d}_{h}^{\pi}(s,a)}{\frac{1}{KH}+\mathbb{E}_{\pi\sim\mu}\big[\widehat{d}_{h}^{\pi}(s,a)\big]}.		
		\vspace{-1ex}			
		\]
		Here, $\widehat{d}_h^{\pi}(s,a)$ is computed via \eqref{eq:hat-d-1-alg} for $h=1$, and \eqref{eq:hat-d-h-alg} for $h\geq 2$.
		%and $\widehat{d}_{h}^{\mu}(s,a)=\mathbb{E}_{\pi\sim\mu}[\widehat{d}_h^{\pi}(s,a)]$.
		\\
		If $g(\pi^{(t)},\widehat{d},\mu^{(t)}_{\mathsf{b}})\leq 2HSA$ then \textsf{exit for-loop}. \label{line:stopping-alg:sub2} 
		\tcp{stopping rule}
		%\\
		Update 
		\tcp{Frank-Wolfe update}
		\vspace{-0.5ex}
		\[
		\mu^{(t+1)}_{\mathsf{b}} = \left(1-\alpha_t\right)\mu^{(t)}_{\mathsf{b}}+\alpha_t \ind_{\pi^{(t)}}.
		\vspace{-1ex}
		\]\\
	}
	\textbf{Output:} the behavior policy $\pi_{\mathsf{b}}=\mathbb{E}_{\pi\sim \mu^{(t)}_{\mathsf{b}}}[\pi]$ and the associated weight $\behavior = \mu^{(t)}_{\mathsf{b}}$.
	\caption{Frank-Wolfe-type subroutine for solving \eqref{defi:target-empirical}.\label{alg:sub2}}
\end{algorithm}

\subsection{Subroutines: approximately solving the subproblems~\eqref{defi:target-empirical-h} and~\eqref{defi:target-empirical}} 
\label{sec:subroutines}
Thus far, we have not yet specified how to 
approximately solve the subproblems \eqref{defi:target-empirical-h} and~\eqref{defi:target-empirical}. 
As it turns out, a simple Frank-Wolfe type procedure \citep{frank1956algorithm} 
allows one to solve these efficiently, to be detailed in this subsection.

%taking the derivative of $f(\cdot)$ w.r.t.~the coefficient $\mu(\pi)$ yields \yly{shall we use the notion of first variation instead? 
%then we can use $\delta f(\mu)(\pi)$ to denote \eqref{eq:derivative-f-objective-explore}.}
%%
%\begin{equation}
%\frac{\partial f(\mu)}{\partial\mu(\pi)}
%	=\sum_{h=1}^{H}\sum_{(s,a)\in\cS\times\cA}\frac{\widehat{d}_{h}^{\pi}(s,a)}{\frac{1}{KH}+\mathbb{E}_{\pi'\sim\mu}\big[\widehat{d}_{h}^{\pi'}(s,a)\big]} .
%	%=\sum_{h=1}^{H}\sum_{(s,a)\in\cS\times\cA}\frac{\widehat{d}_{h}^{\pi}(s,a)}{\frac{1}{KH}+\sum_{\pi'\in\Pi}\mu(\pi')\widehat{d}_{h}^{\pi'}(s,a)}
%	%=\sum_{h=1}^{H}\sum_{(s,a)\in\cS\times\cA}\frac{\widehat{d}_{h}^{\pi}(s,a)}{\frac{1}{KH}+\widehat{d}_{h}^{\mu}(s,a)},	
%	\label{eq:derivative-f-objective-explore}
%\end{equation}
%%
%%where we define
%%
%%\begin{equation}
%%	\widehat{d}_{h}^{\mu} \coloneqq \sum_{\pi'\in\Pi}\mu(\pi')\widehat{d}_{h}^{\pi'}. 
%%\end{equation}
%%
%

\paragraph{Subroutine for solving the subproblem \eqref{defi:target-empirical}.}
Let us start by tackling the convex subproblem \eqref{defi:target-empirical}, whose objective function is given by 
\begin{equation}
	f(\mu)\coloneqq\sum_{h=1}^{H}\sum_{(s,a)\in\cS\times\cA}\log\bigg[\frac{1}{KH}+ \mathop{\mathbb{E}}_{\pi\sim\mu}\big[\widehat{d}_{h}^{\pi}(s,a)\big]\bigg].	
	%f(\mu) \coloneqq \sum_{h=1}^H \sum_{(s, a)\in \cS\times \cA} \log\bigg[\frac{1}{KH} + \sum_{\pi\in\Pi} \mu(\pi)\widehat{d}_{h}^{\pi}(s, a)\bigg] . 
\end{equation}
As can be easily calculated, the first variation\footnote{The first variation of $f:\Delta(\Pi)\to\mathbb{R}$ at $\mu$ is defined as any measurable function $\delta f(\mu):\Pi \to \mathbb{R}$ that satisfies $$\lim_{\varepsilon\to 0} \frac{f(\mu+\varepsilon \mathcal{X})-f(\mu)}{\varepsilon}=\int \delta f(\mu) \mathrm{d}\mathcal{X}$$ for any signed measure $\mathcal{X}$ over $\Pi$ satisfying $\int \mathrm{d}\mathcal{X}=0$. The first variation is defined up to an additive constant. See \cite{gelfand2000calculus} for more details.} of $f$ at a measure $\mu\in \Delta(\Pi)$  is given by
\begin{equation}
	\delta f(\mu) : ~~\pi \,\mapsto\, 
	\underset{\eqqcolon \, \delta f(\mu)(\pi) }{\underbrace{ \sum_{h=1}^{H}\sum_{(s,a)\in\cS\times\cA}\frac{\widehat{d}_{h}^{\pi}(s,a)}{\frac{1}{KH}+\mathbb{E}_{\pi'\sim\mu}\big[\widehat{d}_{h}^{\pi'}(s,a)\big]} }}
	\label{eq:derivative-f-objective-explore}
\end{equation}
for any deterministic policy $\pi\in \Pi$. 
Intuitively, when $\mu,\nu\in\Delta(\Pi)$ are close, the first variation allows one to approximate $f(\nu)$ with the first-order expansion $f(\mu)+\int \delta f(\mu) \mathrm{d}(\nu-\mu)$. 
%For those readers unfamiliar with first variation, it can simply be undersood as a sort of functional differential of $f$, which will not be need for describing our final algorithm or presenting the proof.
We then propose to solve \eqref{defi:target-empirical} via the Frank-Wolfe algorithm,  
where each iteration consists of the following two steps:
\begin{itemize}
	\item (\textit{direction finding}) Find a direction $y^{(t)}\in\Delta(\Pi)$  that is a solution to the following problem
		\begin{equation}
			\mathop{\arg\max}_{\mu \in \Delta(\Pi)} ~ \int \delta  f\big(\mu^{(t)}_{\mathsf{b}}\big) \mathrm{d} \mu.
		\end{equation}
%		\begin{equation}
%			\arg \max_{\mu \in \Delta(\Pi)} \big\langle \nabla f\big(\mu^{(t)}_{\mathsf{b}}\big), \mu \big\rangle .
%		\end{equation}
		%
		Given that $\delta f(\mu)(\pi)$ is always non-negative 
		(cf.~\eqref{eq:derivative-f-objective-explore}), one can simply take 
		%
%		\begin{equation}
%			y^{(t)} = \ind\big(\pi^{(t)}\big) 
%			\qquad \text{with } \pi^{(t)} \in \arg\max_{\pi \in \Pi} ~\partial f(\mu)(\pi) \,\bigg|\,_{\mu = \mu^{(t)}_{\mathsf{b}}}. 
%		\end{equation}
		%
		\begin{equation}
			y^{(t)} = \ind\big(\pi^{(t)}\big) 
			\qquad \text{with } \pi^{(t)} \in \arg\max_{\pi \in \Pi} ~ \delta f(\mu_{\mathsf{b}}^{(t)})(\pi).
		\end{equation}
		Here, $\ind(\pi)$ is a Dirac measure centred on $\pi\in\Pi$.
		It then boils down to maximizing $\delta f(\mu) (\pi)$ over all deterministic policies $\pi$. 
		While this might seem like a challenging task at first glance, 
		one can solve it by applying dynamic programming to an MDP associated with $\widehat{P}$. 
		Note, however, that $\widehat{P}$ is not yet a valid probability transition kernel due to the truncation operation, 
		and hence needs to be slightly modified. More concretely, 
		\begin{itemize}
			\item Introduce a finite-horizon MDP 
			$\mathcal{M}_{\mathsf{b}}=(\cS \cup \{s_{\com}\},\cA,H,\widehat{P}^{\com},r_{\mathsf{b}})$, where $s_{\com}$ is an augmented state and the reward function is chosen to be
				\begin{align} \label{eq:reward-function-mub}
r_{\mathsf{b},h}(s,a)=%
\begin{cases}
	\frac{1}{\frac{1}{KH}+\mathbb{E}_{\pi\sim\mu^{(t)}_{\mathsf{b}}}\big[\widehat{d}_{h}^{\pi}(s,a)\big]}\in[0,KH],\quad & \text{if }(s,a,h)\in\cS\times\cA\times[H];\\
0, & \text{if }(s,a,h)\in\{s_{\com}\}\times\cA\times[H].
\end{cases}
		\end{align}
		In addition, the augmented probability transition kernel $\widehat{P}^{\com}$ is constructed based on $\widehat{P}$ as follows:
		\begin{subequations}
			\label{eq:augmented-prob-kernel}
		\begin{align}
			\widehat{P}^{\com}_{h}(s^{\prime}\mymid s,a) &=    \begin{cases}
			\widehat{P}_{h}(s^{\prime}\mymid s,a),
			& 
			\text{if }  s^{\prime}\in \cS \\
			1 - \sum_{s^{\prime}\in \cS} \widehat{P}_{h}(s^{\prime}\mymid s,a), & \text{if }  s^{\prime} = s_{\com}
			\end{cases}
			%
			%\qquad 
			&&\text{for all }(s,a,h)\in \cS\times \cA\times [H]; \\
			\widehat{P}^{\com}_{h}(s^{\prime}\mymid s_{\com},a) &= \ind(s^{\prime} = s_{\com})
			&&\text{for all }(a,h)\in  \cA\times [H] .
		\end{align}
		\end{subequations}
		Clearly, the augmented state is an absorbing state associated with zero immediate rewards. 
		%The main reason to construct $\widehat{P}^{\com}$ stems from the observation that $\widehat{P}$ is not yet a probability transition kernel (due to the truncation operation) and needs to be properly modified. 

%(ii) for $(s,a,h)\in \{s_{\com}\}\times \cA\times [H]$, $\widehat{P}^{\com}_{h}(s^{\prime}\mymid s,a) = \ind(s^{\prime} = s_{\com})$.

		\item 
		As can be easily seen,  maximizing $\frac{\partial f(\mu)}{\partial\mu(\pi)}$ (cf.~\eqref{eq:derivative-f-objective-explore}) can be relaxed to 
		finding the optimal policy of $\mathcal{M}_{\mathsf{b}}$, which can be accomplished efficiently using classical dynamic programming methods \citep{bertsekas2017dynamic}.

		\end{itemize}

	\item  \textit{(update)} Set the iterate to be
		\begin{equation}
			\mu^{(t+1)}_{\mathsf{b}} 
			= (1-\alpha_t)\mu^{(t)}_{\mathsf{b}} + \alpha_t \ind\big(\pi^{(t)}\big),
			\label{eq:update-rule-FW-b}
		\end{equation}
		where $0<\alpha_t<1$ denotes the learning rate. 
		The learning rate $\alpha_t$ is chosen to be 
		\begin{equation}
\alpha_{t}=\frac{\frac{1}{SAH}g\big(\pi^{(t)},\widehat{d},\mu_{\mathsf{b}}^{(t)}\big)-1}{g\big(\pi^{(t)},\widehat{d},\mu_{\mathsf{b}}^{(t)}\big)-1}, 		
			\label{eq:learning-rates-alpha-intuition}
		\end{equation}
		where we define
\begin{equation}
	g\big(\pi,\widehat{d},\mu\big)\coloneqq
	\sum_{h=1}^{H}\sum_{(s,a)\in\cS\times\cA}\frac{\frac{1}{KH}+\widehat{d}_{h}^{\pi}(s,a)}{\frac{1}{KH}+\mathbb{E}_{\pi'\sim\mu}\big[\widehat{d}_{h}^{\pi'}(s,a)\big]}.		
	%\sum_{h=1}^H \sum_{(s,a)\in \cS\times\cA}\frac{\frac{1}{KH}+\widehat{d}_{h}^{\pi}(s,a)}{\frac{1}{KH}+\widehat{d}_{h}^{\mu}(s,a)}.
\end{equation}
		%

%		\begin{equation}
%			\|\pi\|_{\widehat{d}} \coloneqq \sum_{(h, s, a) \in [H]\times\cS\times\cA} \frac{\frac{1}{KH} + \widehat{d}_h^{\pi}(s, a)}{\frac{1}{KH} + \widehat{d}_{h}(s, a)}
%		\end{equation}
%		%
		
\end{itemize}

\paragraph{Stopping rule and iteration complexity.} 
In order to achieve computational efficiency, 
we would like to terminate the algorithm after a reasonable number of iterations. 
Towards this end, we propose the following stopping rule: 
if the following condition is met: 
	\begin{align}
		g\big(\pi^{(t)},\widehat{d},\mu_{\mathsf{b}}^{(t)}\big) \le 2HSA, \label{eq:termination}
	\end{align}
then the subroutine stops and returns $\mu^{(t)}_{\mathsf{b}}$. 
As it turns out, this stopping rule yields a reasonable iteration complexity, as asserted by the following lemma.  The proof can be found in Section~\ref{sec:proof-lem:iteration-complexity-subroutine}.
\begin{lemma}
	\label{lem:iteration-complexity-subroutine}
	Armed with the learning rate \eqref{eq:learning-rates-alpha-intuition} and the stopping rule \eqref{eq:termination}, 
	this subroutine terminates within $O\big(HSA\log (KH)\big)$ iterations. 
\end{lemma}

\paragraph{Subroutine for solving the subproblem~\eqref{defi:target-empirical-h}.}
Approximately solving the subproblem \eqref{defi:target-empirical-h} 
can be accomplished by means of roughly the same procedure for solving \eqref{defi:target-empirical}. 
The main thing that needs to be slightly modified is the construction of the auxiliary MDP. 
More specifically, for solving \eqref{defi:target-empirical-h} w.r.t.~step $h$ in the $t$-th iteration, 
we construct
$\mathcal{M}_{\mathsf{b}}^h=(\cS \cup \{s^{\com}\},\cA,H,\widehat{P}^{\com, h},r_{\mathsf{b}}^h)$, where $s_{\com}$ is an augmented state as before, and the reward function is chosen to be
				\begin{align} \label{eq:reward-function-mub-h}
					r_{\mathsf{b},j}^{h}(s,a)=%
\begin{cases}
	\frac{1}{\frac{1}{KH}+\mathbb{E}_{\pi\sim\mu^{(t)}}\big[\widehat{d}_{h}^{\pi}(s,a)\big]}\in[0,KH],\quad & \text{if }(s,a,j)\in\cS\times\cA\times\{h\};\\
0, & %\text{if }(s,a,j)\in\{s_{\com}\}\times\cA\times[H] \cup \cS\times\cA\times([H] \setminus \{h\}).
	\text{if } s = s_{\com} \text{ or } j\neq h.  
\end{cases}
		\end{align}
		In addition, the augmented probability transition kernel $\widehat{P}^{\com, h}$ is constructed based on $\widehat{P}$ as follows:
		\begin{subequations}
			\label{eq:augmented-prob-kernel-h}
		\begin{align}
			\widehat{P}^{\com, h}_{j}(s^{\prime}\mymid s,a) &=    \begin{cases}
			\widehat{P}_{j}(s^{\prime}\mymid s,a),
			& 
			\text{if }  s^{\prime}\in \cS \\
			1 - \sum_{s^{\prime}\in \cS} \widehat{P}_{j}(s^{\prime}\mymid s,a), & \text{if }  s^{\prime} = s_{\com}
			\end{cases}
			%
			%\qquad 
			&&\text{for all }(s,a,j)\in \cS\times \cA\times [h]; \\
			\widehat{P}^{\com, h}_{j}(s^{\prime}\mymid s,a) &= \ind(s^{\prime} = s_{\com})
			&&\text{if } s = s_{\com} \text{ or } j > h.
		\end{align}
		\end{subequations}
The other part of the procedure is nearly identical to the one for solving \eqref{defi:target-empirical}; 
see Algorithm~\ref{alg:sub1} for details.

\subsection{Computational efficiency}

Before proceeding, 
let us briefly remark on the computational complexity of the proposed algorithm. 
Regarding Stage 1.2, 
we can see (cf.~Lemma~\ref{lem:iteration-complexity-subroutine}) that $\widetilde{O}(HSA)$ iterations are needed to solve \eqref{defi:target-empirical}, where each iteration incurs a computational complexity of  $\widetilde{O}(HN)$. 
Stage 1.1 also has similar computational complexity.
In addition, the computational complexity of Stage 2 amounts to $\widetilde{O}(HK)$.
As a result, the total computational complexity is $\widetilde{O}(HK + H^2SAN)$, 
which reduces to $\widetilde{O}(HK)$ for large enough $K$ (i.e., $K \gtrsim HSAN$) and is hence asymptotically nearly linear.

\subsection{Intuition}
\label{sec:intuition}

Before proceeding to our main theory, 
we take a moment to explain the rationale behind our algorithm design.  
For simplicity of presentation, let us look at the special case where there exists a single fixed reward function of interest. 
Imagine that we are given an exploration policy 
\begin{equation}
	\pi_{\mathsf{b}}= \mathbb{E}_{\pi \sim \mu_{\mathsf{b}}} [\pi] 
\end{equation}
for some $\mu_{\mathsf{b}}\in \Delta(\Pi)$, which takes the form of some mixture of deterministic policies. 
We would like to sample $K$ episodes using this policy $\pi_{\mathsf{b}}$, 
and perform policy learning via an offline RL algorithm. 
In light of the state-of-the-art offline RL theory \citep{yin2021towards,shi2022pessimistic,li2022settling}, 
the total regret of a sample-efficient offline RL algorithm can be upper bounded (up to some logarithm factor) by 
\begin{align}
	V^{\star}(\rho) - V^{\widehat{\pi}}(\rho) &\lesssim 
	 \sum_h\sum_{s, a} d_h^{\pi^{\star}}(s, a)\min\left\{\sqrt{\frac{\mathsf{Var}_{P_{h, s, a}}\big(V_{h+1}^{\pi^{\star}}\big)}{K  \mathbb{E}_{\pi^{\prime} \sim \mu_{\mathsf{b}}} \big[ d_h^{\pi^{\prime}}(s, a) \big]}}, \, H\right\} , 
	 %\notag \\
	%&\leq
%\max_{\pi\in \Pi} \sum_h\sum_{s, a} d_h^{\pi}(s, a)\min\left\{\sqrt{\frac{\mathsf{Var}_{P_{h, s, a}}\big(V_{h+1}^{\pi}\big)}{K\sum_{\pi^{\prime}\in \Pi} \mu_{\mathsf{b},h}(\pi^{\prime})d_h^{\pi^{\prime}}(s, a)}}, \, H\right\}, 
	\label{defi:offline}
\end{align}
where $\pi^{\star}$ is the optimal policy, and $\widehat{\pi}$ represents the policy output by the offline RL algorithm.  
To further convert \eqref{defi:offline} into a more convenient upper bound, we make the observation that
\begin{align*}
 & \sum_{h}\sum_{s,a}d_{h}^{\pi^{\star}}(s,a)\min\left\{ \sqrt{\frac{\mathsf{Var}_{P_{h,s,a}}\big(V_{h+1}^{\pi^{\star}}\big)}{K\mathbb{E}_{\pi^{\prime}\sim\mu_{\mathsf{b}}}\big[d_{h}^{\pi^{\prime}}(s,a)\big]}},H\right\} \\
 & \quad\leq\sum_{h}\sum_{s,a}d_{h}^{\pi^{\star}}(s,a)\sqrt{\frac{\mathsf{Var}_{P_{h,s,a}}\big(V_{h+1}^{\pi^{\star}}\big)+H}{K\mathbb{E}_{\pi^{\prime}\sim\mu_{\mathsf{b}}}\big[d_{h}^{\pi^{\prime}}(s,a)\big]+1/H}}\\
 & \quad\leq\left[\sum_{h}\sum_{s,a}\frac{d_{h}^{\pi^{\star}}(s,a)}{1/H+K\mathbb{E}_{\pi^{\prime}\sim\mu_{\mathsf{b}}}\big[d_{h}^{\pi^{\prime}}(s,a)\big]}\cdot\sum_{h}\sum_{s,a}d_{h}^{\pi^{\star}}(s,a)\Big(\mathsf{Var}_{P_{h,s,a}}\big(V_{h+1}^{\pi^{\star}}\big)+H\Big)\right]^{\frac{1}{2}}\\
 & \quad\leq c_{8}H\left[\sum_{h}\sum_{s,a}\frac{d_{h}^{\pi^{\star}}(s,a)}{1/H+K\mathbb{E}_{\pi^{\prime}\sim\mu_{\mathsf{b}}}\big[d_{h}^{\pi^{\prime}}(s,a)\big]}\right]^{\frac{1}{2}}
\end{align*}
for some universal constant $c_8>0$, 
where the second line hold since $\min\big\{ \frac{x}{y}+\frac{u}{w}\big\} \leq \frac{x+u}{y+w}$ for any $x,y,u,w>0$, 
the penultimate line arises from Cauchy-Schwarz, 
and the last line is valid since (see \citet{li2022settling} or our analysis in Section~\ref{sec:analysis-reward-agnostic})
\[
	\sum_{h}\sum_{s,a}d_{h}^{\pi^{\star}}(s,a)\mathsf{Var}_{P_{h,s,a}}\big(V_{h+1}^{\pi^{\star}}\big)\leq O(H^{2})\quad\text{and}\quad\sum_{h}\sum_{s,a}d_{h}^{\pi^{\star}}(s,a)H \leq O(H^{2}) .
\]
Substitution into \eqref{defi:offline} then yields
\begin{align}
	V^{\star}(\rho) - V^{\widehat{\pi}}(\rho) &\lesssim 
	H\left[\sum_{h}\sum_{s,a}\frac{d_{h}^{\pi^{\star}}(s,a)}{1/H+K\mathbb{E}_{\pi'\sim\mu_{\mathsf{b}}}\big[\widehat{d}_{h}^{\pi'}(s,a)\big]}\right]^{\frac{1}{2}}\notag\\
 & \lesssim H\max_{\pi\in\Pi}\left[\sum_{h}\sum_{s,a}\frac{d_{h}^{\pi}(s,a)}{1/H+K\mathbb{E}_{\pi'\sim\mu_{\mathsf{b}}}\big[\widehat{d}_{h}^{\pi'}(s,a)\big]}\right]^{\frac{1}{2}}.	 
	% \sum_h\sum_{s, a} d_h^{\pi^{\star}}(s, a)\min\left\{\sqrt{\frac{\mathsf{Var}_{P_{h, s, a}}\big(V_{h+1}^{\pi^{\star}}\big)}{K\sum_{\pi^{\prime}\in \Pi} \mu_{\mathsf{b},h}(\pi^{\prime})d_h^{\pi^{\prime}}(s, a)}}, \, H\right\} , 
	 %\notag \\
	%&\leq
%\max_{\pi\in \Pi} \sum_h\sum_{s, a} d_h^{\pi}(s, a)\min\left\{\sqrt{\frac{\mathsf{Var}_{P_{h, s, a}}\big(V_{h+1}^{\pi}\big)}{K\sum_{\pi^{\prime}\in \Pi} \mu_{\mathsf{b},h}(\pi^{\prime})d_h^{\pi^{\prime}}(s, a)}}, \, H\right\}, 
	\label{defi:offline-further}
\end{align}
Everything then comes down to optimizing the following quantity  
\begin{align}
	%\mathop{\text{minimize}}_{\mu\in\Delta(\Pi)}~
	\max_{\pi\in \Pi}~ \sum_h\sum_{s, a} \frac{d_h^{\pi}(s, a)}{\frac{1}{KH} + \mathbb{E}_{\pi'\sim\mu_{\mathsf{b}}}\big[\widehat{d}_{h}^{\pi'}(s,a)\big] }, \label{defi:target}
\end{align}
provided that we have sufficiently accurate information about 
the occupancy distribution $d_h^{\pi}$ for every $\pi$.

As it turns out, by choosing $\mu_{\mathsf{b}}\in \Delta(\Pi)$ to be a solution to the following convex program:  
\begin{align}
	\arg\max_{\mu\in\Delta(\Pi)}\sum_{h=1}^{H}\sum_{(s,a)\in\cS\times\cA}\log\bigg[\frac{1}{KH}+\mathbb{E}_{\pi\sim\mu}\big[\widehat{d}_{h}^{\pi}(s,a)\big]\bigg],	 
	\label{eq:defn-mu-star-intuition}
\end{align}
we can control the quantity \eqref{defi:target} to be at the desired level, as implied by the following elementary fact.  
\begin{lemma}[\citet{kiefer1960equivalence}] \label{lem:det}
By taking $\mu_{\mathsf{b}}$ to be a solution to \eqref{eq:defn-mu-star-intuition}, we have
\begin{align}
	%\max_{\pi \in \Pi} \sum_h\sum_{s, a} \frac{d_h^{\pi}(s, a)}{\frac{1}{KH} + \sum_{\pi^{\prime}\in \Pi} \mu_{\mathsf{b}}(\pi^{\prime})d_h^{\pi^{\prime}}(s, a)} \le HSA. 
	\max_{\pi\in\Pi}\sum_{h}\sum_{s,a}\frac{\widehat{d}_{h}^{\pi}(s,a)}{\frac{1}{KH}+\mathbb{E}_{\pi'\sim\mu_{\mathsf{b}}}\big[\widehat{d}_{h}^{\pi'}(s,a)\big]}\le HSA.	
\end{align}
\end{lemma}
This lemma is a direct consequence of the main theorem in \citet{kiefer1960equivalence}. 
Combining Lemma~\ref{lem:det} with \eqref{defi:offline-further}, 
we arrive at 
\begin{align}
	V^{\star}(\rho) - V^{\widehat{\pi}}(\rho) \leq O\bigg( \sqrt{\frac{H^3SA}{K}} \bigg), 
\end{align}
thus attaining the desired order. 
This explains the rationale behind the subproblem \eqref{defi:target-empirical}, provided that $d_h^{\pi}$ can be estimated faithfully.  
The other subproblem \eqref{defi:target-empirical-h} can be elucidated analogously, which we omit here.

\section{Main results}
\label{sec:main-results}

We are now positioned to present our main theoretical guarantees for the proposed algorithm, 
and we begin by looking at the reward-agnostic scenario with at most a polynomial number of fixed reward functions. 
\begin{theorem}[Reward-agnostic RL]
\label{thm:main}
Consider any given $0<\delta<1$ and $0<\varepsilon<1$. 
%and any $\varepsilon>0$. 
Suppose that there are $m_{\mathsf{reward}}=\mathsf{poly}(H,S,A)$ fixed reward functions of interest.
Using the same batch of collected data that obey 
\begin{align}
K \geq c_K \frac{H^3SA}{\varepsilon^2}\log \frac{KH}{\delta}\qquad\text{and}\qquad  KH\geq N \geq c_N \sqrt{H^9S^7A^7K} \log\frac{HSA}{\delta} 
\end{align}
for some sufficiently large universal constants $c_K,c_N>0$, 
we can guarantee that 
with probability at least $1-\delta$, the proposed algorithm (cf.~Algorithm~\ref{alg:main}) is able to achieve 
\begin{equation}
	V^{\star}(\rho) - V^{\widehat{\pi}}(\rho)   \leq  \varepsilon, 
\end{equation}
for each of these reward functions.
\end{theorem}

The proof of this theorem is deferred to Section~\ref{sec:analysis-reward-agnostic}.  
Remarkably, Theorem~\ref{thm:main} establishes the sample complexity of the proposed algorithm for the reward-agnostic setting. 
More precisely, taking $N = c_N \sqrt{H^9S^7A^7K} \log\frac{HSA}{\delta}$ in Theorem~\ref{thm:main} reveals that: 
to yield an $\varepsilon$-optimal policy, it suffices for our algorithm  to have a sample size as small as
\begin{equation}
	\text{(sample complexity)} 
	\qquad N_{\mathsf{tot}}=K+NH \leq 2K = \widetilde{O}\bigg( \frac{H^3SA}{\varepsilon^2} \bigg) \text{ episodes}, 
	\label{eq:sample-complexity-implication}
\end{equation}
with the proviso that $K$ is sufficiently large (or equivalently, $\varepsilon$ is sufficiently small). 
Here, we recall that $N_{\mathsf{tot}}=K+NH$ is the total number of collected sample episodes. 
Encouragingly, this sample complexity is provably minimax-optimal up to logarithmic factor;  
to justify this, even in the reward-aware case with a single reward function of interest, 
it has been shown by \citet{jin2018q,domingues2021episodic} that the sample complexity cannot go below $\frac{H^3SA}{\varepsilon^2}$ (up to logarithmic factor) 
regardless of the algorithm in use. 
To the best of our knowledge, 
the present paper offers the first algorithm that provably achieves minimax optimality in this scenario.

What is more, the exploration algorithm we develop turns out to be sample-efficient for the reward-free setting as well, 
the scenario where one would like to simultaneously account for arbitrary reward functions (including those designed adversarially). 
Our theoretical guarantees are as follows. 
\begin{theorem}[Reward-free RL]
\label{thm:main-RFE}
Consider any given $0<\delta<1$ and $0<\varepsilon<1$. 
%Suppose that there are an arbitrary number of possibly adversarial reward functions of interest. 
Using the same batch of collected data that obey 
\begin{align}
	K \geq c_K \frac{H^3S^2A}{\varepsilon^2}\log \frac{KH}{\delta}\qquad\text{and}\qquad  KH\geq N \geq c_N \sqrt{H^9S^7A^7K} \log\frac{HSA}{\delta} 
\end{align}
for some sufficiently large universal constants $c_K,c_N>0$, 
we can guarantee that 
with probability at least $1-\delta$, the proposed algorithm (cf.~Algorithm~\ref{alg:main}) is able to achieve 
\begin{equation}
	V^{\star}(\rho) - V^{\widehat{\pi}}(\rho)   \leq  \varepsilon 
\end{equation}
uniformly over all possible reward functions. 

\end{theorem}
The proof of this theorem is postponed to Section~\ref{sec:analysis-RFE}. 
In a nutshell, the proof is based on the analysis for the reward-agnostic case in conjuction with standard uniform concentration bounds. 
Theorem \ref{thm:main-RFE} asserts that the sample complexity required for the proposed algorithm to achieve $\varepsilon$-optimal policy is
\begin{equation}
	\text{(sample complexity)} 
	\qquad N_{\mathsf{tot}}=K+NH \leq 2K = \widetilde{O}\bigg( \frac{H^3S^2A}{\varepsilon^2} \bigg) \text{ episodes},
	\label{eq:sample-complexity-implication-reward-free}
\end{equation}
provided that we take $N = c_N \sqrt{H^9S^7A^7K} \log\frac{HSA}{\delta}$ and that $K$ is sufficiently large.  
In comparison to \eqref{eq:sample-complexity-implication}, this sample complexity \eqref{eq:sample-complexity-implication-reward-free} is $S$ times larger, 
due to the more stringent requirement to accommodate all reward functions uniformly. 
Interestingly, 
this theory matches  the state-of-the-art sample complexity result (i.e., \citet{menard2021fast}) derived so far for this reward-free setting. 
Regard lower bounds, it has been previously shown in \cite{jin2020reward} that 
a sample size on the order of $H^2S^2 A/\varepsilon^2$ (up to logarithmic factor) is necessary for {\em time-homogeneous} finite-horizon MDPs; 
it is widely conjectured that the minimax lower bound for the inhomogeneous case studied herein should be a factor of $H$ larger than the lower limit for the homogeneous counterpart.

\section{Prior art} 
\label{sec:related-works}

In this section, we briefly overview a subset of other related works.

\paragraph{Reward-aware exploration.} 
The studies of online exploration have been a central topic in RL. 
Here, we mention in passing a small number of representative works. 
The development of the UCRL algorithm, the UCRL2 algorithm and their variants \citep{jaksch2010near,auer2006logarithmic} 
exemplified earlier effort in implementing the optimism principle in the face of uncertainty. 
A more sample-efficient model-based online RL algorithm, called {\em upper confidence bound
value iteration (UCBVI)}, was later proposed by \citet{azar2017minimax}, 
which yields minimax-optimal regret asymptotically.  
Turning to model-free algorithms, \citet{jin2018q} justified the efficacy of Q-learning (in conjunction with UCB-type exploration) in online RL, 
which achieves a regret that is a factor of $\sqrt{H}$ away from optimal if the Bernstein-type confidence bounds are employed. 
Other versions of Q-learning-type algorithms, including the ones that come with low switching cost and the ones tailored to discounted infinite-horizon MDPs, 
have since been developed \citep{bai2019provably,dong2019q}, 
among which a variance-reduced variant achieves asymptotically optimal regret \citep{zhang2020almost}. 
The recent works \citet{menard2021ucb,li2021breaking,zhang2023settling} 
further investigated how to reduce the burn-in cost --- that is, the sample size required to attain sample optimality. The minimax-optimal sample complexity with minimal burn-ins was recently settled by  \citet{zhang2023settling}. 
%
% Extensions to accommodate linear function approximation have been an active research topic for the past few years as well \citet{jin2020provably,}.
All these algorithms, however, rely on {\em a priori} reward information, which are ill-suited to the reward-agnostic convext.  
Additionally, information-theoretic regret lower bounds for reward-aware online RL were first developed by \citet{jin2018q,jaksch2010near}, 
and revisited later on by \citet{domingues2021episodic}. 
Going beyond the tabular case, a number of papers have further pursued efficient reward-aware exploration in the presence of low-dimensional function approximation (e.g., \citet{jin2020provably,ayoub2020model,du2021bilinear,li2021sample}).

\paragraph{Reward-free and task-agnostic exploration.} 
Moving on to reward-free exploration (the case where one is asked to account for an arbitrarily large number of reward functions), 
the R-max type algorithm proposed in the earlier work  \citet{brafman2002r} incurs a sample complexity at least as large as $\frac{H^{11}S^2A}{\varepsilon^3}$ (see \citet[Appendix A]{jin2020reward}). 
Recently, \citet{jin2020reward} came up with a clever scheme --- with instantaneous rewards chosen to be indicator functions regarding state/action visitations 
--- that is guaranteed to work with $\widetilde{O}\big(\frac{H^5S^2A}{\varepsilon^2}\big)$ sample episodes. 
This sample complexity bound is further improved by \citet{kaufmann2021adaptive,menard2021fast}, 
with \citet{menard2021fast} designing an algorithm with optimal sample complexity (i.e., $\widetilde{O}\big(\frac{H^3S^2A}{\varepsilon^2}\big)$) in the reward-free setting. 
Another work \citet{zhang2021near} studied a different setting with ``totally bounded rewards'' (so that the sum of immediate rewards over all steps is bounded above by 1) in time-homogeneous finite-horizon MDPs; 
when translated to the bounded reward setting (so that any immediate reward is bounded above by 1),  
the algorithm put forward in \citet{zhang2021near} exhibited a sample size of  $\widetilde{O}\big(\frac{H^2S^2A}{\varepsilon^2}\big)$ episodes.\footnote{Note that this horizon dependency $H^2$ (proven for homogeneous MDPs) is 
in general not possible in inhomogeneous MDPs.}   
Reward-free exploration with low policy switching cost has been investigated in 
\citet{qiao2022sample}, while RFE in non-tabular case (e.g., the case with low-dimensional function approximation) has been further explored in 
\citet{wang2020reward,wagenmaker2022reward,qiu2021reward,zhang2021reward,chen2021near,zanette2020provably,misra2020kinematic,qiao2022near,chen2022statistical,zhang2022efficient,mhammedi2024efficient,chen2022unified}.  
It is also related to the problem of uniform policy evaluation, which aims to ensure reliable policy evaluation uniformly over all policies \citep{yin2021optimal}. 
In contrast to the reward-free setting that covers arbitrarily many reward functions, 
\citet{zhang2020task} assumed the existence of only $N$ reward functions of interest, and proposed a model-free algorithm that enjoys a sample complexity of $\widetilde{O}\big(\frac{H^5SA\log m_{\mathsf{reward}}}{\varepsilon^2}\big)$. 
This result, however, is suboptimal in terms of the horizon dependency. 
Finally, reward-free exploration has also been studied in the context of 
safe RL \citep{huang2022safe} and 
constrained RL \citep{miryoosefi2022simple}, which are beyond the scope of the current paper.

\paragraph{Offline RL.} 
The past few years have seen much activity in offline RL (also called batch RL) \citep{levine2020offline}. 
The principle of pessimism (or conservatism) in the face of uncertainty has been shown to be effective in solving offline RL  \citep{kumar2020conservative,jin2021pessimism,rashidinejad2021bridging,li2022settling,jin2022policy}. 
The sample complexity of offline RL in the tabular case has been tackled by a recent line of works 
 \citep{rashidinejad2021bridging,xie2021policy,yin2021towards,shi2022pessimistic,li2022settling,yan2022efficacy,yin2021near}; 
take finite-horizon MDPs with inhomogeneous transition kernels for example:  
the minimax-optimal sample complexity for attaining $\varepsilon$-accuracy is shown to be $\frac{H^3 SC^{\star}}{\varepsilon^2}$ episodes for any $\varepsilon\in (0,H]$, 
where $C^{\star}$ stands for some single-policy concentrability coefficient that captures the resulting distribution shift between the optimal policy and the behavior policy 
 \citep{li2022settling}. 
In particular, sample optimality for the full $\varepsilon$-range is achievable via the model-based approach  \citep{li2022settling}, 
and establishing this result relies on modern statistical tools like the leave-one-out decoupling argument \citep{agarwal2019optimality,li2020breaking}. 
Sample-efficient offline RL algorithms have also been studied in more complicated scenarios, 
including but not limited to the case with linear function approximation \citep{jin2021pessimism,xu2022provably} and zero-sum Markov games \citep{cui2022offline,cui2022provably,yan2022model}.

%\citet{}

\section{Analysis for reward-agnostic exploration (proof of Theorem~\ref{thm:main})}
\label{sec:analysis-reward-agnostic}

In this section, we present the proof of our main result for reward-agnostic exploration in Theorem~\ref{thm:main}. 
Towards this end, we shall first establish the following result when there is a single reward function of interest (i.e., $m_{\mathsf{reward}}=1$).

\begin{theorem}
\label{thm:main-single}
Consider any given $0<\delta<1$ and $0<\varepsilon<1$. 
Suppose that there is only $m_{\mathsf{reward}}=1$ reward function of interest and it is independent of the data samples. 
With probability at least $1-\delta$, the proposed algorithm (cf.~Algorithm~\ref{alg:main}) achieves 
\begin{equation}
	V^{\star}(\rho) - V^{\widehat{\pi}}(\rho)   \leq  \varepsilon, 
\end{equation}
provided that
\begin{align}
K \geq c_K \frac{H^3SA}{\varepsilon^2}\log \frac{KH}{\delta}\qquad\text{and}\qquad  KH\geq N \geq c_N \sqrt{H^9S^7A^7K} \log\frac{HSA}{\delta} 
\end{align}
for some sufficiently large universal constants $c_K,c_N>0$. 
\end{theorem}
As it turns out, Theorem~\ref{thm:main} (with $m_{\mathsf{reward}}=\mathsf{poly}(H,S,A)$) 
is a direct consequence of Theorem~\ref{thm:main}.  
To see this, replacing $\delta$ with $\delta / m_{\mathsf{reward}} \asymp \delta / \mathsf{poly}(H,S,A)$ in Theorem~\ref{thm:main-single} 
and taking the union bound over all $m_{\mathsf{reward}}$ reward functions of interest suffice to justify Theorem~\ref{thm:main}. 
Consequently, the remainder of this section is dedicated to proving Theorem~\ref{thm:main-single}.

Before continuing, we isolate several additional notation that might be useful in presenting our proof. 
We introduce the following vector notation for value functions: for any $1\leq h\leq H$, 
\begin{equation}
	V_h^{\pi} = \big[ V_h^{\pi}(s) \big]_{s\in \cS}  \quad \text{for any policy }\pi
	\qquad \text{and} \qquad 
	V_h^{\star} = \big[ V_h^{\star}(s) \big]_{s\in \cS}. 
\end{equation}
We shall also use $\widehat{V}_h=[\widehat{V}_h(s)]_{s\in \cS}$ to represent the value estimate for step $h$ by Algorithm~\ref{alg:offline_RL}.  
The state occupancy distribution for policy $\pi$ at step $h\in [H]$ is also represented by the following vector
\begin{equation}
	d_h^{\pi} = \big[ d_h^{\pi}(s) \big]_{s\in \cS} .
\end{equation}
We  also use the shorthand notation $P_{h,s,a} \in \mathbb{R}^{1\times S}$ (resp.~$\widehat{P}_{h,s,a} \in \mathbb{R}^{1\times S}$) to denote 
$P_h(\cdot\mymid s,a)$ (resp.~$\widehat{P}_h(\cdot\mymid s,a)$).

\subsection{Preliminary facts}

Before embarking on the proofs of our main theory, 
we first gather several useful preliminary results. 

%\paragraph{Properties about $\{\widehat{N}_h^{\mathsf{b}}\}$.} 
%

%Recall the definition~\eqref{eq:lower-samples} for $\widehat{N}_h^{\mathsf{b}}(s,a)$, which implies that $\widehat{N}_h^{\mathsf{b}}(s, a)$ serves as an independent lower bound for the number of sample transitions. \yxc{modify} 

\paragraph{A direct consequence of the stopping rule.} 

To begin with, recall the stopping rule of Algorithm \ref{alg:sub2} (see line~\ref{line:stopping-alg:sub2}) and the optimality of $\pi^{(t)}$ (see line~\ref{line:pit-alg:sub2} in Algorithm \ref{alg:sub2} and also the discussion at the end of Section~\ref{sec:subroutines}). 
These taken collectively allow one to demonstrate that for any policy $\pi$, 
\begin{align}
%\max_{\pi}
	\sum_{h=1}^{H}\sum_{(s,a)\in\mathcal{S}\times\mathcal{A}}\frac{\widehat{d}_{h}^{\pi}(s,a)}{\frac{1}{KH}+\mathbb{E}_{\pi'\sim\behavior}\big[\widehat{d}_{h}^{\pi^{\prime}}(s,a)\big]} & = \sum_{h=1}^{H}\sum_{(s,a)\in\mathcal{S}\times\mathcal{A}}\frac{\widehat{d}_{h}^{\pi^{(t)}}(s,a)}{\frac{1}{KH}+\mathbb{E}_{\pi'\sim\behavior}\big[\widehat{d}_{h}^{\pi^{\prime}}(s,a)\big]}\nonumber\\
 & = g(\pi^{(t)},\widehat{d},\widehat{\mu}_{\mathsf{b}})\leq2HSA, 
	\label{eq:near-optimality-condition}
\end{align}
where $\behavior$ is the output returned by Algorithm~\ref{alg:sub2}, and $\pi^{(t)}$ represents the optimal policy computed right before the termination of this algorithm. 
Here, the first relation arises since the left-hand side is the value function of the MDP $\mathcal{M}_{\mathsf{b}}$ w.r.t.~policy $\pi$, which is maximized by $\pi^{(t)}$; 
 the last inequality is due to the stopping rule. 
This property \eqref{eq:near-optimality-condition}, 
which results from our carefully designed exploration scheme, plays a key role in the subsequent analysis.

\paragraph{Properties about the model-based offline algorithm. }
Next, we collect a couple of preliminary facts that have been previously established in \citet{li2022settling} for the model-based algorithm in Algorithm~\ref{alg:offline_RL}. 
\begin{lemma}
	\label{lem:model-based-offline-prior}
	With probability exceeding $1-\delta$, 
	the policy $\widehat{\pi}$ returned by Algorithm~\ref{alg:offline_RL} obeys
	\begin{align}
		\big\langle d_{h}^{\pi^{\star}},V_{h}^{\star}-V_{h}^{\widehat{\pi}}\big\rangle 
		& \leq \big\langle d_{h}^{\pi^{\star}},V_{h}^{\star}- \widehat{V}_{h}\big\rangle
		\leq2\sum_{j:j\geq h}\sum_{s\in\mathcal{S}}d_{j}^{\pi^{\star}}(s)b_{j}\big(s,\pi_{j}^{\star}(s)\big)
		\leq2\sum_{j:j\geq h}\sum_{(s,a)\in\mathcal{S}\times\mathcal{A}}d_{j}^{\pi^{\star}}(s,a)b_{j}(s,a)	 
		\label{eq:error-decomposition-old-offline}
	\end{align}
	for all $1\leq h\leq H$. 
	Additionally, with probability exceeding $1-\delta$, one has
	\begin{equation}
		\mathsf{Var}_{\widehat{P}_{h,s,a}}(\widehat{V}_{h+1})\leq2\mathsf{Var}_{P_{h,s,a}}(\widehat{V}_{h+1})+\frac{5H^{2}\log\frac{KH}{\delta}}{\widehat{N}_{h}^{\mathsf{b}}(s,a)},
		\qquad\forall(s,a,h)\in\mathcal{S}\times\mathcal{A}\times[H].
		\label{eq:Variance-UB-offline-old}
	\end{equation}
	Here,  $\widehat{P}=\{\widehat{P}_h\}_{1\leq h\leq H}$ denotes the empirical transition kernel constructed in Algorithm~\ref{alg:offline_RL}.
\end{lemma}

\begin{proof}
	The claim~\eqref{eq:error-decomposition-old-offline} follows immediately from \citet[Eqn.~(126)]{li2022settling}. 
	The claim~\eqref{eq:Variance-UB-offline-old} follows from \citet[Lemma 8]{li2022settling} 
	and the fact that $\widehat{V}_{h+1}$ is independent of $\widehat{P}_{h,s,a}$ (owing to the algorithm design; see \citet[Section~5.2]{li2022settling}). 
\end{proof}

%\paragraph{Additional notation.} 
%

\subsection{Main steps} \label{sec:proof-main}

\paragraph{Step 1: accuracy of estimated occupancy distributions.} 
We first provide estimation guarantees for the estimated occupancy distributions $\widehat{d}_h^\pi$ 
(see \eqref{eq:hat-d-1-alg} and \eqref{eq:hat-d-h-alg} in Algorithm \ref{alg:main}). The proof is deferred to Appendix \ref{sec:proof-lemma-occupancy}.
\begin{lemma} \label{lemma:occupancy}
	Recall that $\xi=c_{\xi}H^3S^3A^3 \log \frac{HSA}{\delta}$ for some large enough constant $c_{\xi}>0$ (see \eqref{eq:hard-threshold}). 
	With probability at least $1-\delta$,  
	the estimated occupancy distributions specified in \eqref{eq:hat-d-1-alg} and \eqref{eq:hat-d-h-alg} of Algorithm \ref{alg:main} satisfy
	\begin{align}
		\frac{1}{2}\widehat{d}_{h}^{\pi}(s, a) - \frac{\xi}{4N} \le d_{h}^{\pi}(s, a) \le 2\widehat{d}_{h}^{\pi}(s, a) + 2e_{h}^{\pi}(s, a) + \frac{\xi}{4N} \label{eq:d-error}
	\end{align}
	simultaneously for all $(s,a,h)\in\mathcal{S}\times\mathcal{A} \times [H]$ and all deterministic policy $\pi\in\Pi$, 
	provided that 
	\begin{align}
		KH \geq N \geq C_N\sqrt{H^9S^7A^7K}\log\frac{HSA}{\delta} 
		\qquad \text{and}  \qquad 
		K\geq C_K HSA
		\label{eq:N-K-lower-bound-lemma2}
	\end{align}
	for some large enough constants $C_N,C_K>0$. 
	Here,  $\{e_{h}^{\pi}(s, a)\}$ is some non-negative sequence satisfying
	\begin{align}
		\sum_{s, a} e_{h}^{\pi}(s, a) \leq \frac{2SA}{K} + \frac{13SAH\xi}{N}  \lesssim \sqrt{\frac{SA}{HK}}
		\qquad \text{for all }h\in[H]\text{ and all deterministic Markov policy }\pi. \label{eq:e-sum-zeta}
	\end{align}
\end{lemma}
In a nutshell, this lemma makes apparent the effectiveness of our exploration stage: 
for all deterministic Markov policy $\pi$, 
we have obtained reasonably accurate estimation of the associated occupancy distributions. 
In particular, the estimation error for each state-action pair is inversely proportional to the number of episodes $N$ in each round,   
up to some additional error terms $\{e_h^{\pi}(s,a)\}$ whose aggregate contributions are well-controlled.

Before proceeding, we single out a direct consequence of Lemma~\ref{lemma:occupancy} that will prove useful. 
Denoting by $N_h^{\mathsf{b}}(s,a)$ the number of visits to $(s,a)$ at step $h$ during Stage 1.2 of the proposed algorithm (which employs the behavior policy $\widehat{\pi}_{\mathsf{b}}=\mathbb{E}_{\pi \sim \behavior}[\pi]$ to sample $K$ episodes),  
we have
\begin{equation}
	N_{h}^{\mathsf{b}}(s,a) \geq \widehat{N}_{h}^{\mathsf{b}}(s,a), \qquad \forall (s,a,h)\in \cS\times \cA\times [H]
\end{equation}
with probability exceeding $1-\delta$, where we remind the reader of the definition of $\widehat{N}_h^{\mathsf{b}}(s,a)$ in \eqref{eq:lower-samples}. 
This property is crucial when invoking Algorithm~\ref{alg:offline_RL}. 
\begin{proof}
It is first seen from  Lemma~\ref{lemma:occupancy} that
\begin{equation}
	\mathbb{E}[N_h^{\mathsf{b}}(s,a)] = K
	\mathop{\mathbb{E}}_{\pi\sim\behavior}\big[d_{h}^{\pi}(s,a)\big]\geq\frac{K}{2}\mathop{\mathbb{E}}_{\pi\sim\behavior}\big[\widehat{d}_{h}^{\pi}(s,a)\big]-\frac{K\xi}{4N}.
	\label{eq:lower-bound-Epi-dh}
\end{equation}
The Bernstein inequality  together the union bound
 then implies that, with probability at least $1-\delta$,
\begin{align*}
N_{h}^{\mathsf{b}}(s,a) & \geq\max\left\{ \mathbb{E}\big[N_{h}^{\mathsf{b}}(s,a)\big]-\sqrt{4\mathbb{E}\big[N_{h}^{\mathsf{b}}(s,a)\big]\log\frac{HSA}{\delta}}-\log\frac{HSA}{\delta},\ 0\right\} \\
 & \geq\max\left\{ \mathbb{E}\big[N_{h}^{\mathsf{b}}(s,a)\big]-\frac{1}{2}\mathbb{E}\big[N_{h}^{\mathsf{b}}(s,a)\big]-2\log\frac{HSA}{\delta}-\log\frac{HSA}{\delta},\ 0\right\} \\
 & \geq\max\left\{ \frac{K}{4}\mathop{\mathbb{E}}_{\pi\sim\behavior}\big[\widehat{d}_{h}^{\pi}(s,a)\big]-\frac{K\xi}{8N}-3\log\frac{HSA}{\delta},\ 0\right\} =\widehat{N}_{h}^{\mathsf{b}}(s,a)
\end{align*}
holds simultaneously over all $(s,a,h)\in\mathcal{S}\times\mathcal{A}\times[H]$,
where the second line invokes the AM-GM inequality, and the last line
makes use of \eqref{eq:lower-bound-Epi-dh}. 
\end{proof}

%Inserting the above relation \eqref{eq:near-optimality-condition} into the performance bound of offline RL (see \citet[Eqn.~(126)]{li2022settling}) \yxc{change it to lemma} leads to

\paragraph{Step 2: bounding $V_{j}^{\star}-V_{j}^{\widehat{\pi}}$ for step $j$.} 
We now turn to the sub-optimality of the final policy estimate $\widehat{\pi}$.   
To begin with, we define, for each $1\leq h \leq H$, the following set
\begin{align}
	\mathcal{I}_{h}:=\bigg\{(s,a): \mathop{\mathbb{E}}_{\pi\sim\behavior}\big[\widehat{d}_{h}^{\pi}(s,a)\big]\geq\frac{4\xi}{N}\bigg\}, 
	\label{eq:defn-Ih}
\end{align}
containing all state-action pairs whose associated empirical occupancy density (weighted over the sampling policy $\behavior$) exceed some prescribed threshold. 
For each $1\leq j\leq H$, one can then bound and decompose
\begin{align}
\big\langle d_{j}^{\pi^{\star}},V_{j}^{\star}-V_{j}^{\widehat{\pi}}\big\rangle 
	& 
	\leq\big\langle d_{j}^{\pi^{\star}},V_{j}^{\star}-\widehat{V}_{j}\big\rangle 
	\overset{\text{(i)}}{\leq}2\sum_{h:h\ge j}\sum_{s,a}d_{h}^{\pi^{\star}}(s,a)b_{h}(s,a) \notag\\
 & \overset{\text{(ii)}}{\leq}2\sum_{h:h\ge j}\sum_{s,a}\left(2\widehat{d}_{h}^{\pi^{\star}}(s,a)+2e_{h+1}^{\pi^{\star}}(s,a)+\frac{\xi}{N}\right)b_{h}(s,a)\nonumber\\
 & \overset{\text{(iii)}}{\lesssim}\sum_{h:h\ge j}\sum_{s,a}\widehat{d}_{h}^{\pi^{\star}}(s,a)b_{h}(s,a)+H^{2}\Big(\sqrt{\frac{SA}{HK}}+\frac{SA\xi}{N}\Big)\nonumber\\
 & \leq\underbrace{\sum_{h:h\ge j}\sum_{(s,a)\in\mathcal{I}_{h}}\widehat{d}_{h}^{\pi^{\star}}(s,a)b_{h}(s,a)}_{\eqqcolon\alpha_{1}}+\underbrace{\sum_{h:h\ge j}\sum_{(s,a)\notin\mathcal{I}_{h}}\widehat{d}_{h}^{\pi^{\star}}(s,a)b_{h}(s,a)}_{\eqqcolon\alpha_{2}}+H^{2}\Big(\sqrt{\frac{SA}{HK}}+\frac{SA\xi}{N}\Big).
	\label{eq:proof-1}
\end{align}
Here, (i) is a consequence of property \eqref{eq:error-decomposition-old-offline} in Lemma~\ref{lem:model-based-offline-prior}, (ii) results from \eqref{eq:d-error} in Lemma~\ref{lemma:occupancy} that quantifies the discrepancy between $d_{h}^{\pi^{\star}}$ and $\widehat{d}_{h}^{\pi^{\star}}$, whereas (iii) follows from \eqref{eq:e-sum-zeta} in Lemma~\ref{lemma:occupancy} as well as the basic fact that $b_h(s,a)\leq H$ (see \eqref{def:bonus-Bernstein-finite}). 
We shall then cope with the two terms $\alpha_1$ and $\alpha_2$ separately. 
\begin{itemize}
	\item With regards to the first term $\alpha_1$ in \eqref{eq:proof-1}, it follows from the choice \eqref{def:bonus-Bernstein-finite} of $b_h$ that
	\begin{align}
\alpha_{1} & \lesssim\sum_{h:h\ge j}\sum_{(s,a)\in\mathcal{I}_{h}}\widehat{d}_{h}^{\pi^{\star}}(s,a)\min\Bigg\{\sqrt{\frac{\log\frac{KH}{\delta}}{\widehat{N}_{h}^{\mathsf{b}}(s,a)}\mathsf{Var}_{\widehat{P}_{h,s,a}}\big(\widehat{V}_{h+1}\big)}+H\frac{\log\frac{KH}{\delta}}{\widehat{N}_{h}^{\mathsf{b}}(s,a)},\,H\Bigg\}\notag\\
 & \overset{\text{(i)}}{\lesssim}\sum_{h:h\ge j}\sum_{(s,a)\in\mathcal{I}_{h}}\widehat{d}_{h}^{\pi^{\star}}(s,a)\min\Bigg\{\sqrt{\frac{\log\frac{KH}{\delta}}{\widehat{N}_{h}^{\mathsf{b}}(s,a)}\mathsf{Var}_{P_{h,s,a}}\big(\widehat{V}_{h+1}\big)}+H\frac{\log\frac{KH}{\delta}}{\widehat{N}_{h}^{\mathsf{b}}(s,a)},\,H\Bigg\}\notag\\
 & \leq\sum_{h:h\ge j}\sum_{(s,a)\in\mathcal{I}_{h}}\widehat{d}_{h}^{\pi^{\star}}(s,a)\left\{ \min\Bigg\{\sqrt{\frac{\mathsf{Var}_{P_{h,s,a}}\big(\widehat{V}_{h+1}\big)\log\frac{KH}{\delta}}{\widehat{N}_{h}^{\mathsf{b}}(s,a)}},\,H\Bigg\}+H\frac{\log\frac{KH}{\delta}}{\widehat{N}_{h}^{\mathsf{b}}(s,a)}\right\} \notag\\
 & \overset{\text{(ii)}}{\lesssim}\sum_{h:h\ge j}\sum_{(s,a)\in\mathcal{I}_{h}}\widehat{d}_{h}^{\pi^{\star}}(s,a)\left\{ \sqrt{\frac{\mathsf{Var}_{P_{h,s,a}}\big(\widehat{V}_{h+1}\big)\log\frac{KH}{\delta}+H}{\widehat{N}_{h}^{\mathsf{b}}(s,a)+1/H}}+H\frac{\log\frac{KH}{\delta}}{\widehat{N}_{h}^{\mathsf{b}}(s,a)}\right\} \notag\\
 & \overset{\text{(iii)}}{\lesssim}\sum_{h:h\ge j}\sum_{(s,a)\in\mathcal{I}_{h}}\widehat{d}_{h}^{\pi^{\star}}(s,a)\left\{ \sqrt{\frac{\mathsf{Var}_{P_{h,s,a}}\big(\widehat{V}_{h+1}\big)\log\frac{KH}{\delta}+H}{K\mathbb{E}_{\pi\sim\behavior}\big[\widehat{d}_{h}^{\pi}(s,a)\big]+1/H}}+\frac{H\log\frac{KH}{\delta}}{1/H+K\mathbb{E}_{\pi\sim\behavior}\big[\widehat{d}_{h}^{\pi}(s,a)\big]}\right\} \notag\\
 & \overset{\text{(iv)}}{\lesssim}\left\{ \sum_{h:h\ge j}\sum_{(s,a)\in\mathcal{I}_{h}}\widehat{d}_{h}^{\pi^{\star}}(s,a)\sqrt{\frac{\mathsf{Var}_{P_{h,s,a}}\big(\widehat{V}_{h+1}\big)\log\frac{KH}{\delta}+H}{1/H+K\mathbb{E}_{\pi\sim\behavior}\big[\widehat{d}_{h}^{\pi}(s,a)\big]}}\right\} +\frac{H^{2}SA\log\frac{KH}{\delta}}{K}.		
		\label{eq:alpha1-UB-proof}
	\end{align}
		Here, (i) makes use of property~\eqref{eq:Variance-UB-offline-old} in Lemma~\ref{lem:model-based-offline-prior}, (ii) is valid due to the elementary inequality $\min\{ \frac{x}{y} , \frac{u}{w} \}\leq \frac{x+u}{y+w}$ for any $x,y,u,w>0$,  
		(iii) follows since 
	\begin{align*}
		\min_{(s,a)\in\mathcal{I}_{h}}\widehat{N}_{h}^{\mathsf{b}}(s,a)\gtrsim K \mathop{\mathbb{E}}_{\pi\sim\behavior}\big[\widehat{d}_{h}^{\pi}(s,a)\big] + \frac{1}{H}
	\end{align*}
	holds according to \eqref{eq:lower-samples} and \eqref{eq:defn-Ih}, 
		while (iv) results from property \eqref{eq:near-optimality-condition}.
	%Moreover, according to~\eqref{eq:near-optimality-condition}, it can be seen that
	%	
	%\begin{align*}
%\sum_{h:h\ge j}\sum_{(s,a)\in\mathcal{I}_{h}}\frac{\widehat{d}_{h}^{\pi^{\star}}(s,a)}{1+K\mathbb{E}_{\pi\sim\behavior}\big[\widehat{d}_{h}^{\pi}(s,a)\big]}\lesssim\frac{HSA}{K}.	
%	\end{align*}
	
	\item When it comes to the other term $\alpha_2$ in \eqref{eq:proof-1}, we observe that
	\begin{align*}
		\sum_{h:h\ge j}\sum_{(s,a)\notin\mathcal{I}_{h}}\widehat{d}_{h}^{\pi^{\star}}(s,a) & \leq\sum_{h:h\ge j}\sum_{(s,a)\notin\mathcal{I}_{h}}\frac{\big(\frac{1}{KH}+\frac{4\xi}{N}\big)\widehat{d}_{h}^{\pi^{\star}}(s,a)}{\frac{1}{KH}+\mathbb{E}_{\pi\sim\behavior}\big[\widehat{d}_{h}^{\pi}(s,a)\big]}\\
 & \leq\left(\frac{1}{KH}+\frac{4\xi}{N}\right)2HSA\\
 & \asymp\frac{HSA\xi}{N},
	\end{align*}
	where the first line follows from the definition \eqref{eq:defn-Ih} of $\mathcal{I}_{h}$, 
	the second line relies on \eqref{eq:near-optimality-condition}, 
		and the last line is valid since $\frac{1}{KH}\lesssim \frac{\xi}{N}$ 
		holds as long as $N\leq KH$ and $\xi \geq 1$. This taken together with the fact that $b_h(s,a)\leq H$  (see \eqref{def:bonus-Bernstein-finite}) immediately gives
	\begin{align}
		\alpha_2\lesssim \frac{H^2 SA\xi}{N}.
		\label{eq:alpha2-UB-proof}
	\end{align}
\end{itemize}
Substituting the preceding bounds \eqref{eq:alpha1-UB-proof} and \eqref{eq:alpha2-UB-proof} on $\alpha_1$ and $\alpha_2$ into \eqref{eq:proof-1}, we arrive at
\begin{align}
 & \big\langle d_{j}^{\pi^{\star}},V_{j}^{\star}-V_{j}^{\widehat{\pi}}\big\rangle
	\leq\big\langle d_{j}^{\pi^{\star}},V_{j}^{\star}-\widehat{V}_{j}\big\rangle
	\lesssim\alpha_{1}+\alpha_{2}+H^{2}\Big(\sqrt{\frac{SA}{HK}}+\frac{SA\xi}{N}\Big)\nonumber\\
 & \quad\lesssim\underbrace{\sum_{h:h\ge j}\sum_{(s,a)\in\mathcal{I}_{h}}\widehat{d}_{h}^{\pi^{\star}}(s,a)\sqrt{\frac{\mathsf{Var}_{P_{h,s,a}}\big(\widehat{V}_{h+1}\big)\log\frac{KH}{\delta}+H}{1/H+K\mathbb{E}_{\pi\sim\behavior}\big[\widehat{d}_{h}^{\pi}(s,a)\big]}}}_{\eqqcolon\beta}+\frac{H^{2}SA\log\frac{KH}{\delta}}{K}+H^{2}\Big(\sqrt{\frac{SA}{HK}}+\frac{SA\xi}{N}\Big).
	 \label{eq:proof-2}
\end{align}
Everything then comes down to controlling the term $\beta$, which forms the main content of the remaining proof.

\paragraph{Step 3: controlling the term $\beta$.} 

To bound the term $\beta$ in \eqref{eq:proof-2}, we first apply the Cauchy-Schwarz inequality to reach
\begin{align}
\beta & \leq\underbrace{\left(\sum_{h:h\geq j}\sum_{(s,a)\in\mathcal{I}_{h}}\widehat{d}_{h}^{\pi^{\star}}(s,a)\left[\mathsf{Var}_{P_{h,s,a}}\big(\widehat{V}_{h+1}\big)\log\frac{KH}{\delta}+H\right]\right)^{1/2}}_{\eqqcolon\beta_{1}}\underbrace{\left(\sum_{h:h\ge j}\sum_{(s,a)\in\mathcal{I}_{h}}\frac{\widehat{d}_{h}^{\pi^{\star}}(s,a)}{1/H+K\mathbb{E}_{\pi\sim\behavior}\big[\widehat{d}_{h}^{\pi}(s,a)\big]}\right)^{1/2}}_{\eqqcolon\beta_{2}},	 
	\label{eq:proof-4}
\end{align}
thus motivating us to bound $\beta_{1}$ and $\beta_{2}$ separately. 
\begin{itemize}
	\item Regarding the first term $\beta_{1}$, we make the observation that
	\begin{align*}
		(\beta_{1})^{2} & \leq\sum_{h}\sum_{(s,a)}\widehat{d}_{h}^{\pi^{\star}}(s,a)\left\{ \mathsf{Var}_{P_{h,s,a}}\big(\widehat{V}_{h+1}\big)\log\frac{KH}{\delta}\right\} +\sum_{h}\sum_{(s,a)}\widehat{d}_{h}^{\pi^{\star}}(s,a)H\\
 & \leq\sum_{h}\sum_{(s,a)}\left(2d_{h}^{\pi^{\star}}(s,a)+\frac{2\xi}{N}\right)\mathsf{Var}_{P_{h,s,a}}\big(\widehat{V}_{h+1}\big)\log\frac{KH}{\delta}+H^{2}\\
 & \le2\sum_{h}\sum_{(s,a)}d_{h}^{\pi^{\star}}(s,a)\Big\{\mathsf{Var}_{P_{h,s,a}}\big(V_{h+1}^{\star}\big)+\mathsf{Var}_{P_{h,s,a}}\big(V_{h+1}^{\star}-\widehat{V}_{h+1}\big)\Big\}\log\frac{KH}{\delta}+\frac{2\xi}{N}SAH^{3}\log\frac{KH}{\delta}+H^{2}\\
 & \lesssim\underbrace{\sum_{h}\sum_{(s,a)}d_{h}^{\pi^{\star}}(s,a)\mathsf{Var}_{P_{h,s,a}}\big(V_{h+1}^{\star}\big)}_{\eqqcolon\beta_{1,1}}\log\frac{KH}{\delta}+\underbrace{\sum_{h}\sum_{(s,a)}d_{h}^{\pi^{\star}}(s,a)\mathsf{Var}_{P_{h,s,a}}\big(V_{h+1}^{\star}-\widehat{V}_{h+1}\big)}_{\eqqcolon\beta_{1,2}}\log\frac{KH}{\delta}\\
 & \qquad+\frac{\xi}{N}SAH^{3}\log\frac{KH}{\delta}+H^{2},
	\end{align*}
	where the second line invokes Lemma~\ref{lemma:occupancy} and the
fact $\sum_{(s,a)}\widehat{d}_{h}^{\pi^{\star}}(s,a)\leq1$, and the
third line is valid since $\mathsf{Var}(X+Y)\leq2\mathsf{Var}(X)+2\mathsf{Var}(Y)$. 	

	\begin{itemize}
		\item

	To cope with the first term $\beta_{1,1}$ in the above display, we find it helpful to define, for each $1\leq h\leq H$,  
	a distribution vector $d_{h}^{\pi^{\star}}=\big[d_{h}^{\pi^{\star}}(s)\big]_{s\in\mathcal{S}}\in \mathbb{R}^S$, 
			a reward vector $r_h^{\pi^{\star}}= \big[r_{h}\big(s, \pi_h^{\star}(s) \big)\big]_{s\in\mathcal{S}}\in \mathbb{R}^S$,
	and a matrix $P_{h}^{\star}\in\mathbb{R}^{S\times S}$ obeying
	\[
		P_{h}^{\star}(s,s')=P_{h}\big(s'\mymid s,\pi_h^{\star}(s)\big),\qquad s,s'\in\mathcal{S}.
	\]
	Given that $\pi^{\star}$ is a deterministic policy, one can write
\begin{align*}
\sum_{(s,a)}d_{h}^{\pi^{\star}}(s,a)\mathsf{Var}_{P_{h,s,a}}\big(V_{h+1}^{\star}\big) & =\sum_{s}d_{h}^{\pi^{\star}}\big(s,\pi_h^{\star}(s)\big)\mathsf{Var}_{P_{h,s,\pi_h^{\star}(s)}}\big(V_{h+1}^{\star}\big)\\
 & =\big\langle d_{h}^{\pi^{\star}},P_{h}^{\star}(V_{h+1}^{\star}\circ V_{h+1}^{\star})-(P_{h}^{\star}V_{h+1}^{\star})\circ(P_{h}^{\star}V_{h+1}^{\star})\big\rangle. 
\end{align*}
	This allows one to deduce that
\begin{align}
\beta_{1,1} & =\sum_{h=1}^{H}\big\langle d_{h}^{\pi^{\star}},P_{h}^{\star}(V_{h+1}^{\star}\circ V_{h+1}^{\star})-(P_{h}^{\star}V_{h+1}^{\star})\circ(P_{h}^{\star}V_{h+1}^{\star})\big\rangle\nonumber \\
 & \overset{\text{(i)}}{=}\sum_{h=1}^{H}\big\langle d_{h+1}^{\pi^{\star}},V_{h+1}^{\star}\circ V_{h+1}^{\star}\big\rangle-\sum_{h=1}^{H}\big\langle d_{h}^{\pi^{\star}},(P_{h}^{\star}V_{h+1}^{\star})\circ(P_{h}^{\star}V_{h+1}^{\star})\big\rangle\nonumber \\
	& \overset{\text{(ii)}}{=} \sum_{h=1}^{H}\big\langle d_{h+1}^{\pi^{\star}},V_{h+1}^{\star}\circ V_{h+1}^{\star}\big\rangle-\sum_{h=1}^{H}\big\langle d_{h}^{\pi^{\star}},V_{h}^{\star}\circ V_{h}^{\star}\big\rangle+2\sum_{h=1}^{H}\big\langle d_{h}^{\pi^{\star}},r_{h}^{\pi^{\star}}\circ(P_{h}^{\star}V_{h+1}^{\star})\big\rangle 
	-\sum_{h=1}^{H}\big\langle d_{h}^{\pi^{\star}},r_{h}^{\star}\circ r_{h}^{\star}\big\rangle \nonumber \\
	& \leq \sum_{h=1}^{H}\big\langle d_{h+1}^{\pi^{\star}},V_{h+1}^{\star}\circ V_{h+1}^{\star}\big\rangle-\sum_{h=1}^{H}\big\langle d_{h}^{\pi^{\star}},V_{h}^{\star}\circ V_{h}^{\star}\big\rangle+2\sum_{h=1}^{H}\big\langle d_{h}^{\pi^{\star}},r_{h}^{\pi^{\star}}\circ(P_{h}^{\star}V_{h+1}^{\star})\big\rangle\nonumber \\
 & \overset{\text{(iii)}}{\leq}2\sum_{h=1}^{H}\big\langle d_{h}^{\pi^{\star}},r_{h}^{\pi^{\star}}\circ(P_{h}^{\star}V_{h+1}^{\star})\big\rangle 
	\overset{\text{(iv)}}{\leq} 2H^{2},\label{eq:beta-11-UB}
\end{align}
	where (i) follows from $(d_{h}^{\pi^{\star}})^\top P_h^\star = (d_{h+1}^{\pi^{\star}})^\top$, 
	(ii) makes use of the Bellman equation $V_h^\star=r_h+P_h^\star V_{h+1}^\star$,  
	(iii) results from the telescoping sum and the fact that $V_{H+1}^{\star}=0$, 
			and (iv) holds since $V_{h}^{\star}(s)\leq H$ and $r_h(s,a)\leq 1$ for all $(s,a)\in \cS\times \cA$.

	\item Regarding $\beta_{1,2}$, we first derive a useful crude bound as follows:
\begin{align}
\big\langle d_{j}^{\pi^{\star}},V_{j}^{\star}-\widehat{V}_{j}\big\rangle\lesssim 
	& \sum_{h:h\ge j}\sum_{(s,a)\in\mathcal{I}_{h}}\frac{H\widehat{d}_{h}^{\pi^{\star}}(s,a) \sqrt{\log\frac{KH}{\delta}}}{\sqrt{1/H+K\mathbb{E}_{\pi\sim\behavior}\big[\widehat{d}_{h}^{\pi}(s,a)\big]}}+\frac{H^{2}SA\log\frac{KH}{\delta}}{K}+H^{2}\Big(\sqrt{\frac{SA}{HK}}+\frac{SA\xi}{N}\Big)\notag\nonumber \\
\lesssim & H\sqrt{\log\frac{KH}{\delta}} \left\{ \sum_{h}\sum_{(s,a)}\widehat{d}_{h}^{\pi^{\star}}(s,a)\right\} ^{\frac{1}{2}}
	\left[\sum_{h}\sum_{(s,a)}\frac{\widehat{d}_{h}^{\pi^{\star}}(s,a)}{1/H+K\mathbb{E}_{\pi'\sim\behavior}\big[\widehat{d}_{h}^{\pi'}(s,a)\big]}\right]^{\frac{1}{2}}\nonumber \\
 & \qquad+\frac{H^{2}SA\log\frac{KH}{\delta}}{K}+H^{2}\Big(\sqrt{\frac{SA}{HK}}+\frac{SA\xi}{N}\Big)\notag\\
\lesssim & \sqrt{\frac{H^{4}SA \log\frac{KH}{\delta}}{K}}+\frac{H^{2}SA\xi}{N}+\frac{H^{2}SA\log\frac{KH}{\delta}}{K}\lesssim1.\label{eq:proof-3}
\end{align}
	Here, the first inequality follows from \eqref{eq:proof-2} and the fact that $\mathsf{Var}_{P_{h, s, a}}\big(\widehat{V}_{h+1}\big)\leq H^2$; 
			the second line applies the Cauchy-Schwarz inequality; 
			the third inequality arises from property \eqref{eq:near-optimality-condition} and the fact $\sum_{(s,a)}\widehat{d}_{h}^{\pi^{\star}}(s,a)\leq 1$; 
			and the last inequality is valid as long as $K\gtrsim H^4 SA \log\frac{KH}{\delta}$ and $N\gtrsim H^2 SA\xi$. 
			With this bound in mind, one can then control $\beta_{1,2}$ as follows
	\begin{align}
\beta_{1,2} & \leq\sum_{h}\sum_{(s,a)}d_{h}^{\pi^{\star}}(s,a)\mathbb{E}_{P_{h,s,a}}\big[\big(V_{h+1}^{\star}-\widehat{V}_{h+1}\big)\circ\big(V_{h+1}^{\star}-\widehat{V}_{h+1}\big)\big]\nonumber \\
 & \leq H\sum_{h}\sum_{(s,a)}d_{h}^{\pi^{\star}}(s,a)\mathbb{E}_{P_{h,s,a}}\big[V_{h+1}^{\star}-\widehat{V}_{h+1}\big]\nonumber \\
 & =H\sum_{h}\big\langle d_{h+1}^{\pi^{\star}},V_{h+1}^{\star}-\widehat{V}_{h+1}\big\rangle\leq H^{2},\label{eq:beta-22-UB}
\end{align}
where the second line holds since $V^\star_{h+1}, \widehat{V}_{h+1}\in [0,H]$, 
			the third line relies on the fact that $(d_h^{\pi^\star})^\top P_h^\star = (d_{h+1}^{\pi^{\star}})^\top$, 
		 and the last relation follows from \eqref{eq:proof-3}. 
	\end{itemize}

Therefore, the above bounds \eqref{eq:beta-11-UB} and \eqref{eq:beta-22-UB} allow one to readily conclude that
\begin{align}
	(\beta_1)^2\lesssim (\beta_{1,1}+\beta_{1,2})\log\frac{KH}{\delta}+\frac{\xi}{N}SAH^3+H^2 \lesssim H^2\log\frac{KH}{\delta},
	\label{eq:beta1-UB-123}
\end{align}
with the proviso that $N\gtrsim \xi SAH$.

	\item 	When it comes to $\beta_2$, we can invoke \eqref{eq:near-optimality-condition} to yield
\begin{align}
(\beta_{2})^{2} & \leq\sum_{h:h\ge j}\sum_{(s,a)\in\mathcal{I}_{h}}\frac{\widehat{d}_{h}^{\pi^{\star}}(s,a)}{1/H+K\mathbb{E}_{\pi\sim\behavior}\big[\widehat{d}_{h}^{\pi}(s,a)\big]}\lesssim\frac{HSA}{K}.\label{eq:beta2-UB-123}
\end{align}
\end{itemize}

\paragraph{Step 4: putting everything together.}
Taking \eqref{eq:proof-2} and \eqref{eq:proof-4} together with the above bounds on $\beta_1$ and $\beta_2$ (see \eqref{eq:beta1-UB-123} and \eqref{eq:beta2-UB-123}) leads to
\begin{align*}
	\big\langle d_{j}^{\pi^{\star}},V_{j}^{\star}-V_{j}^{\widehat{\pi}}\big\rangle & \lesssim\beta_{1}\beta_{2}+\frac{H^{2}SA\log\frac{KH}{\delta}}{K}+H^{2}\Big(\sqrt{\frac{SA}{HK}}+\frac{SA\xi}{N}\Big)\\
 & \lesssim\sqrt{\frac{H^{3}SA}{K}\log\frac{KH}{\delta}}+\frac{H^{2}SA\log\frac{KH}{\delta}}{K}+\frac{H^{2}SA\xi}{N}\\
 & \asymp\sqrt{\frac{H^{3}SA}{K}\log\frac{KH}{\delta}},
\end{align*}
with the proviso that $K\gtrsim HSA\log\frac{KH}{\delta}$, and $N \gtrsim \sqrt{H^7S^7A^7K\log\frac{HSA}{\delta}}$. 
In particular, taking $j=1$ in the above display and using the fact that $d_{1}^{\pi^{\star}}(s)=\rho(s)$, we arrive at
\begin{align*}
V_{1}^{\star}(\rho)-V_{1}^{\widehat{\pi}}(\rho)= & \big\langle d_{1}^{\pi^{\star}},V_{1}^{\star}-V_{1}^{\widehat{\pi}}\big\rangle\lesssim\sqrt{\frac{H^{3}SA}{K}\log\frac{KH}{\delta}}\leq\varepsilon,
\end{align*}
provided that $K\geq\frac{c_{K}H^{3}SA\log\frac{KH}{\delta}}{\varepsilon^{2}}$
for some large enough universal constant $c_{K}>0$. 
This concludes the proof of Theorem~\ref{thm:main-single}.

\subsection{Proof of Lemma~\ref{lemma:occupancy}} \label{sec:proof-lemma-occupancy}

We intend to establish this lemma via an inductive argument.

\paragraph{Step 1: the base case with $h=1$.}

Let us begin by looking at the case with $h=1$. 
It is seen from our construction \eqref{eq:init-occupancy} and the fact $d_1^\pi(s,a) = \rho(s) \pi_1(a\mymid s)$ that
\begin{align}
	\widehat{d}_1^\pi(s,a)-d_1^\pi(s,a)= \pi_1(a\mymid s)\bigg(\frac{1}{N}\sum_{n=1}^N \ind (s_1^{n,0}=s)-\rho(s)\bigg).
	\label{eq:d1-pi-hat-diff}
\end{align}
Recall that $\{s_1^{n,0}\}_{1\leq n\leq N}$ are independently drawn from the initial distribution $\rho$, 
and hence $\{\ind(s_1^{n,0}=s)\}_{1\leq n \leq N}$ are independent Bernoulli random variables  with mean $\rho(s)$. 
Apply the Bernstein inequality  \citep[Theorem 2.8.4]{vershynin2018high} in conjunction with the union bound to show that:  there exists some universal constant $\widetilde{C}>0$ such that with probability exceeding $1-\delta$, 
\begin{align}
	\left|\frac{1}{N}\sum_{n=1}^N \ind (s_1^{n,0}=s)-\rho(s)\right|
	& \leq \sqrt{\frac{\widetilde{C}}{N}\rho(s)\log\frac{SAH}{\delta}} + \frac{\widetilde{C}}{N}\log\frac{SAH}{\delta} \notag\\
	& \leq \frac{1}{2} \rho(s) + \frac{\widetilde{C}}{2N}\log\frac{SAH}{\delta}  + \frac{\widetilde{C}}{N}\log\frac{SAH}{\delta} \notag\\
	& \leq \frac{1}{2} \rho(s) +  \frac{2\widetilde{C}}{N}\log\frac{SAH}{\delta}
	\label{eq:concentration-h-1-Lemma2}
\end{align}
holds simultaneously for all $s\in\mathcal{S}$, where the second line comes from the AM-GM inequality. 
Substituting this into \eqref{eq:d1-pi-hat-diff} and recalling that $\xi=c_{\xi}H^3S^3A^3 \log \frac{HSA}{\delta}$ give
\begin{align*}
	\left|\widehat{d}_1^\pi(s,a)-d_1^\pi(s,a)\right| 
	\leq   \frac{1}{2} \rho(s) \pi_1(a\mymid s) +  \frac{2\widetilde{C} \pi_1(a\mymid s)}{N}\log\frac{SAH}{\delta}
	\leq \frac{1}{2} d_1^\pi(s,a) + \frac{\xi}{4N}
\end{align*}
as long as $c_\xi>0$ is sufficiently large, 
which clearly holds true for all policy $\pi$. 
%
%Here, we remind the reader of the definition of $\xi$ in  \eqref{eq:hard-threshold}. 
%
This finishes the proof of the claim \eqref{eq:d-error} when $h=1$.

\paragraph{Step 2: the inductive step.} 
Next, we carry out the inductive step.  
Assuming that the claim \eqref{eq:d-error} holds for the $j$-th step ($\forall 1\leq j\leq h$), we would like to establish its validity for the $(h+1)$-th step as well. 
In what follows, we introduce the following shorthand notation for the transition probabilities under policy $\pi$: 
\begin{subequations}
\label{eq:notation-P-jh-pi}
\begin{align}
	P_{j \to h}^{\pi}(s, a \mymid s', a') &= \mathbb{P}\big( s_h = s, a_h = a \mymid s_j = s', a_j = a'; \pi \big), 
	\label{eq:P-j-arrow-h-pi}\\
	P_{h}^{\pi}(s, a \mymid s', a') & 
	%= \mathbb{P}\big( s_{h+1} = s, a_{h+1} = a \mymid s_h = s', a_h = a'; \pi \big) 
	= P_h(s\mymid s', a') \pi_{h+1}(a\mymid s), \\
	\widehat{P}_{h}^{\pi}(s, a \mymid s', a') &
	%= 
	%\mathbb{P}\big( s_{h+1} = s, a_{h+1} = a \mymid s_h = s', a_h = a'; \pi \big) 
	= \widehat{P}_h(s\mymid s', a') \pi_{h+1}(a\mymid s),
\end{align}
\end{subequations}
where in \eqref{eq:P-j-arrow-h-pi} the probability is calculated assuming policy $\pi$ is executed.

%the probability of reaching the state-action pair $(s,a)$ in the $h$-th step conditional on 
%$(s_{j}, a_{j})$ in the $j$-th step when policy $\pi$ is adopted. 

\paragraph{Step 2.1: decomposing $\widehat{d}_{h+1}^\pi-d_{h+1}^\pi$ into two terms.}
By virtue of the construction of $\widehat{d}_{h+1}^{\pi}$ in~\eqref{eq:d_empirical} as well as the basic identities $d_{h+1}^{\pi}(s)=\big\langle P_{h}(s\mymid\cdot,\cdot),\,d_{h}^{\pi}(\cdot,\cdot)\big\rangle$ and $d_{h+1}^{\pi}(s,a)=d_{h+1}^{\pi}(s)\pi_{h+1}(a\mymid s)$, we have 
\begin{align}
 & \widehat{d}_{h+1}^{\pi}(s,a)-d_{h+1}^{\pi}(s,a)=\pi_{h+1}(a\mymid s)\left\{ \big\langle\widehat{P}_{h}(s\mymid\cdot,\cdot),\,\widehat{d}_{h}^{\pi}(\cdot,\cdot)\big\rangle-\big\langle P_{h}(s\mymid\cdot,\cdot),\,d_{h}^{\pi}(\cdot,\cdot)\big\rangle\right\} \nonumber \\
 & \qquad=\pi_{h+1}(a\mymid s)\left\{ \big\langle\widehat{P}_{h}(s\mymid\cdot,\cdot)-P_{h}(s\mymid\cdot,\cdot),\,\widehat{d}_{h}^{\pi}(\cdot,\cdot)\big\rangle+\big\langle P_{h}(s\mymid\cdot,\cdot),\,\widehat{d}_{h}^{\pi}(\cdot,\cdot)-d_{h}^{\pi}(\cdot,\cdot)\big\rangle\right\} \nonumber \\
 & \qquad=\big\langle\widehat{P}_{h}^{\pi}(s,a\mymid\cdot,\cdot)-P_{h}^{\pi}(s,a\mymid\cdot,\cdot),\,\widehat{d}_{h}^{\pi}(\cdot,\cdot)\big\rangle+\big\langle P_{h\rightarrow h+1}^{\pi}(s,a\mymid\cdot,\cdot),\,\widehat{d}_{h}^{\pi}(\cdot,\cdot)-d_{h}^{\pi}(\cdot,\cdot)\big\rangle,\label{eq:decompose-dhplus1-dh-diff}
\end{align}
where the last identity makes use of the notation \eqref{eq:notation-P-jh-pi}. 
Continuing this derivation, we arrive at
\begin{align*}
 & \widehat{d}_{h+1}^{\pi}(s,a)-d_{h+1}^{\pi}(s,a)\\
 & \quad=\big\langle\widehat{P}_{h}^{\pi}(s,a\mymid\cdot,\cdot)-P_{h}^{\pi}(s,a\mymid\cdot,\cdot),\,\widehat{d}_{h}^{\pi}(\cdot,\cdot)\big\rangle+\sum_{s_{h},a_{h}}P_{h\rightarrow h+1}^{\pi}(s,a\mymid s_{h},a_{h})\Big(\widehat{d}_{h}^{\pi}(s_{h},a_{h})-d_{h}^{\pi}(s_{h},a_{h})\Big)\\
 & \quad=\big\langle\widehat{P}_{h}^{\pi}(s,a\mymid\cdot,\cdot)-P_{h}^{\pi}(s,a\mymid\cdot,\cdot),\,\widehat{d}_{h}^{\pi}(\cdot,\cdot)\big\rangle+\sum_{s_{h},a_{h}}P_{h\rightarrow h+1}^{\pi}(s,a\mymid s_{h},a_{h})\\
 & \quad\quad\cdot\Big\{\big\langle\widehat{P}_{h-1}^{\pi}(s_{h},a_{h}\mymid\cdot,\cdot)-P_{h-1}^{\pi}(s_{h},a_{h}\mymid\cdot,\cdot),\,\widehat{d}_{h-1}^{\pi}(\cdot,\cdot)\big\rangle+\big\langle P_{h-1\rightarrow h}^{\pi}(s_{h},a_{h}\mymid\cdot,\cdot),\,\widehat{d}_{h-1}^{\pi}(\cdot,\cdot)-d_{h-1}^{\pi}(\cdot,\cdot)\big\rangle\Big\}\\
 & \quad=\big\langle\widehat{P}_{h}^{\pi}(s,a\mymid\cdot,\cdot)-P_{h}^{\pi}(s,a\mymid\cdot,\cdot),\,\widehat{d}_{h}^{\pi}(\cdot,\cdot)\big\rangle\\
 & \quad\quad\quad+\sum_{s_{h},a_{h}}P_{h\rightarrow h+1}^{\pi}(s,a\mymid s_{h},a_{h})\big\langle\widehat{P}_{h-1}^{\pi}(s_{h},a_{h}\mymid\cdot,\cdot)-P_{h-1}^{\pi}(s_{h},a_{h}\mymid\cdot,\cdot),\,\widehat{d}_{h-1}^{\pi}(\cdot,\cdot)\big\rangle\\
 & \qquad\quad+\big\langle P_{h-1\rightarrow h+1}^{\pi}(s_{h},a_{h}\mymid\cdot,\cdot),\,\widehat{d}_{h-1}^{\pi}(\cdot,\cdot)-d_{h-1}^{\pi}(\cdot,\cdot)\big\rangle,
\end{align*}
where the second identity applies \eqref{eq:decompose-dhplus1-dh-diff} to the term $\widehat{d}_h^{\pi} - d_h^{\pi}$,  
and the last line is valid since 
\[
\sum_{s_{h},a_{h}}P_{h\rightarrow h+1}^{\pi}(s,a\mymid s_{h},a_{h})P_{h-1\rightarrow h}^{\pi}(s_{h},a_{h}\mymid s',a')=P_{h-1\rightarrow h+1}^{\pi}(s,a\mymid s',a').
\]
Repeating the above argument recursively, we arrive at
\begin{align}
 & \widehat{d}_{h+1}^{\pi}(s,a)-d_{h+1}^{\pi}(s,a) \notag\\
 & \quad=\Big\langle\widehat{P}_{h}^{\pi}(s,a\mymid\cdot,\cdot)-P_{h}^{\pi}(s,a\mymid\cdot,\cdot),\,\widehat{d}_{h}^{\pi}(\cdot,\cdot)\Big\rangle \notag\\
 & \quad\quad+\sum_{j=1}^{h-1}\sum_{s_{j+1},a_{j+1}}P_{j+1\to h+1}^{\pi}(s,a\mymid s_{j+1},a_{j+1})\Big\langle\widehat{P}_{j}^{\pi}(s_{j+1},a_{j+1}\mymid\cdot,\cdot)-P_{j}^{\pi}(s_{j+1},a_{j+1}\mymid\cdot,\cdot),\,\widehat{d}_{j}^{\pi}(\cdot,\cdot)\Big\rangle \notag\\
 & \quad=\sum_{j=1}^{h}\sum_{s_{j+1},a_{j+1}}P_{j+1\to h+1}^{\pi}(s,a\mymid s_{j+1},a_{j+1})\Big\langle\widehat{P}_{j}^{\pi}(s_{j+1},a_{j+1}\mymid\cdot,\cdot)-P_{j}^{\pi}(s_{j+1},a_{j+1}\mymid\cdot,\cdot),\,\widehat{d}_{j}^{\pi}(\cdot,\cdot)\Big\rangle.
	\label{eq:diff-dhat-d-long}
\end{align}
%
%where $\mathbb{P}_{j+1 \to h+1}^{\pi}(s, a \mymid s_{j+1}, a_{j+1})$ is the transition probability from the state-action pair $(s_{j+1}, a_{j+1})$ on the $j+1$-th step to $(s, a)$ on the $h+1$-th step,  under the policy $\pi$. 

Next, we would like to make use of the construction of $\widehat{P}_h$ in the above decomposition \eqref{eq:diff-dhat-d-long}. 
For this purpose, we find it convenient to introduce the following notation $k_n(s,a,h)$ such that
\begin{equation}
	k_n(s,a,h): \text{ the index of the episode in which the trajectory visits $(s,a)$ in step $h$ for the $n$-th time.}
\end{equation}
In view of our construction of $\widehat{P}_h$ in \eqref{eq:empirical-P-step-h}, we can divide the following inner product into two components based on whether $N_j(s_j,a_j)>\xi$ or not: 
\begin{align*}
 & \Big\langle\widehat{P}_{j}^{\pi}(s_{j+1},a_{j+1}\mymid\cdot,\cdot)-P_{j}^{\pi}(s_{j+1},a_{j+1}\mymid\cdot,\cdot),\,\widehat{d}_{j}^{\pi}(\cdot,\cdot)\Big\rangle\\
 & \qquad=\sum_{(s_{j},a_{j}):N_{j}(s_{j},a_{j})>\xi}\frac{\widehat{d}_{j}^{\pi}(s_{j},a_{j})}{N_{j}(s_{j},a_{j})}\sum_{n=1}^{N_{j}(s_{j},a_{j})}\Big[\ind\big(s_{j+1}^{k_{n}(s_{j},a_{j},j),j}=s_{j+1}\big)-P_{j}(s_{j+1}\mymid s_{j},a_{j})\Big]\pi_{j+1}(a_{j+1}\mymid s_{j+1})\\
 & \qquad\qquad-\sum_{(s_{j},a_{j}):N_{j}(s_{j},a_{j})\le\xi}P_{j}^{\pi}(s_{j+1},a_{j+1}\mymid s_{j},a_{j})\widehat{d}_{j}^{\pi}(s_{j},a_{j}). 
\end{align*}
Substitution into \eqref{eq:diff-dhat-d-long} allows one to decompose
\begin{align}
	\widehat{d}_{h+1}^{\pi}(s, a) - d_{h+1}^{\pi}(s, a) = \beta_{h+1}^\pi (s,a) - e_{h+1}^\pi (s,a) , \label{eq:proof-5}
\end{align}
where the two terms $e_{h+1}^\pi (s,a)$ and $\beta_{h+1}^\pi (s,a)$ are given respectively by
\begin{align}
e_{h+1}^{\pi}(s, a) &\coloneqq \sum_{j=1}^{h} \sum_{s_{j+1}, a_{j+1}} P_{j+1 \to h+1}^{\pi}(s, a \mymid s_{j+1}, a_{j+1})\sum_{(s_j, a_j) : N_j(s_j, a_j) \le \xi}P_{j}^{\pi}(s_{j+1}, a_{j+1} \mymid s_j, a_j)\widehat{d}_{j}^{\pi}(s_j, a_j) \notag\\
&= \sum_{j=1}^h\sum_{(s_j, a_j) : N_j(s_j, a_j) \le \xi} P_{j \to h+1}^{\pi}(s, a \mymid s_{j}, a_{j})\widehat{d}_{j}^{\pi}(s_{j}, a_{j}) \ge 0 \label{eq:e-defi}
\end{align}
and
\begin{align}
	\beta_{h+1}^{\pi}(s, a) 
	& \coloneqq \sum_{j=1}^h \sum_{s_{j+1}, a_{j+1}} P_{j+1 \to h+1}^{\pi}(s, a \mymid s_{j+1}, a_{j+1})\sum_{(s_j, a_j) : N_j(s_j, a_j) > \xi}\frac{\widehat{d}_{j}^{\pi}(s_j, a_j)}{N_j(s_j, a_j)} \notag \\ 
	& \qquad \cdot\sum_{n = 1}^{N_j(s_j, a_j)} \Big[\ind\big(s_{j+1}^{k_n(s_j, a_j,j),j} = s_{j+1} \big) - P_j(s_{j+1} \mymid s_j, a_j)\Big]\pi_{j+1}(a_{j+1} \mymid s_{j+1}).
	\label{eq:beta-defi}
\end{align}
In words, $\{e_{h+1}^{\pi}(s, a)\}$ captures the total contributions of those state-action pairs $(s_j,a_j)$ with $N_j(s_j,a_j)\leq\xi$, 
while $\beta_{h+1}^\pi (s,a)$ reflects the contributions of the remaining state-action pairs. 
In order to establish Lemma~\ref{lemma:occupancy}, it boils down to controlling $e_{h+1}^{\pi}(s, a)$ and $\beta_{h+1}^{\pi}(s, a)$ in the decomposition \eqref{eq:proof-5}, 
which we shall accomplish separately in the following.

\begin{comment}
%
\yxc{TBD}
To establish the desired result, it suffices to show that
%
\begin{align}
\sum_{s, a} e_{h+1}^{\pi}(s, a) \le \zeta \qquad\text{and}\qquad \big|\beta_{h+1}^{\pi}(s, a)\big| \le \frac{\xi}{4N} + \frac{1}{2}d_{h+1}^{\pi}(s, a) \label{eq:condition}
\end{align}
%
hold for all $\pi \in \Pi$ and $(s, a) \in \cS \times \cA$ simultaneously, where $\xi$ is defined in~\eqref{eq:hard-threshold}, and $\zeta \lesssim \frac{H^4S^3A^3}{N}\log\frac{HSA}{\delta} \lesssim \sqrt{\frac{SA}{HK}}$, provided that $N \gtrsim \sqrt{H^9S^5A^5K}\log\frac{HSA}{\delta}$.
%
\end{comment}

\paragraph{Step 2.2: controlling the term $\sum_{s,a} e_{h+1}^{\pi}(s, a)$.} 

In order to bound the sum $\sum_{s,a} e_{h+1}^{\pi}(s, a)$ (cf.~\eqref{eq:e-defi}),  
we first make the following claim: with probability exceeding $1-\delta$,   
\begin{align}
	\Big\{(s,a):N_j(s,a)\leq \xi\Big\} 
	\overset{\mathrm{(a)}}{\subseteq} \bigg\{(s, a) :  \mathop{\mathbb{E}}_{\pi\sim\widehat{\mu}^j}\big[d_{j}^{\pi}(s,a)\big]  \leq \frac{3\xi}{N} \bigg\} 
	\overset{\mathrm{(b)}}{\subseteq} \bigg\{(s, a) : \mathop{\mathbb{E}}_{\pi\sim\widehat{\mu}^j}\big[\widehat{d}_{j}^{\pi}(s,a)\big]  \leq \frac{6.5\xi}{N} \bigg\} \eqqcolon \mathcal{J}_j 
	\label{eq:proof-6}
\end{align}
holds  for all $1\leq j\leq h$, 
which we shall justify below. 
\begin{itemize}
	\item {\em Inclusion relation (a).}  
		To see why the first relation (a) is valid, 
		recall that $N_j(s,a)$ can be viewed as the sum of $N$ independent Bernoulli random variables and obeys
		\[
			\mathbb{E} \big[ N_j(s,a) \big] = N \mathop{\mathbb{E}}_{\pi\sim\widehat{\mu}^j}\big[d_{j}^{\pi}(s,a)\big], 
		\]
		given that in this round we take samples using the exploration policy $\pi^{\mathsf{explore},j}=\mathbb{E}_{\pi \sim \widehat{\mu}^j}[\pi]$ (see line~\ref{line:output-Alg-sub1} of Algorithm~\ref{alg:sub1}).  
		Repeating the Bernstein-type concentration argument as in \eqref{eq:concentration-h-1-Lemma2} and applying the union bound, 
		we see that with probability at least $1-\delta$, 
\begin{align}
	N_j(s,a) \geq \frac{N}{2}\mathop{\mathbb{E}}_{\pi\sim\widehat{\mu}^j}\big[d_{j}^{\pi}(s,a)\big]- 2\widetilde{C}\log\frac{HSA}{\delta}
\end{align}
		holds simultaneously for all $1\leq j\leq h$ and all $(s,a)\in \cS\times \cA$, where $\widetilde{C}>0$ is some universal constant.   
		Consequently, for any $(s,a)$ obeying $N_j(s,a)\leq \xi$ one has
\[
	\mathop{\mathbb{E}}_{\pi\sim\widehat{\mu}^j}\big[d_{j}^{\pi}(s,a)\big]  \leq 
		\frac{2N_j(s,a)}{N} + \frac{4\widetilde{C}}{N}\log\frac{HSA}{\delta}
		\leq \frac{3\xi}{N},
\]
		where the last inequality also relies on the choice \eqref{eq:hard-threshold} of $\xi$. This establishes the advertised relation (a).

\item {\em Inclusion relation (b).} 
Regarding the second claim (b) in \eqref{eq:proof-6}, 
one can easily verify it using the induction hypothesis \eqref{eq:d-error} for any $j\leq h$; 
		namely, for any $(s,a)$ obeying $\mathop{\mathbb{E}}_{\pi\sim\widehat{\mu}^j}\big[d_{j}^{\pi}(s,a)\big]  \leq \frac{3\xi}{N}$, 
		one has
		\begin{align}
			&\frac{1}{2}\mathop{\mathbb{E}}_{\pi\sim\widehat{\mu}^j}\big[\widehat{d}_{j}^{\pi}(s,a)\big]  - \frac{\xi}{4N} \leq 
			\mathop{\mathbb{E}}_{\pi\sim\widehat{\mu}^j}\big[d_{j}^{\pi}(s,a)\big]  \leq \frac{3\xi}{N}  \notag\\
			&\hspace{2em}\Longrightarrow \qquad
			\mathop{\mathbb{E}}_{\pi\sim\widehat{\mu}^j}\big[\widehat{d}_{j}^{\pi}(s,a)\big]  \leq 
			 \frac{6.5\xi}{N}
		\end{align}

\end{itemize}

\noindent 
Thus far, we have validated the claim \eqref{eq:proof-6}. With this result in place, invoke the definition \eqref{eq:e-defi} to derive 
\begin{align}
		\sum_{s, a} e_{h+1}^{\pi}(s, a) &= \sum_{s, a}\sum_{j=1}^h\sum_{(s_{j}, a_{j}) : N_j(s_j, a_j) 
		\le \xi} P_{j \to h+1}^{\pi}(s, a \mymid s_{j}, a_{j})\widehat{d}_{j}^{\pi}(s_{j}, a_{j})  \notag \\
		&= \sum_{j=1}^h\sum_{(s_{j}, a_{j}) : N_j(s_j, a_j) 
		\le \xi} \bigg( \sum_{s, a} P_{j \to h+1}^{\pi}(s, a \mymid s_{j}, a_{j}) \bigg) \widehat{d}_{j}^{\pi}(s_{j}, a_{j}) \notag \\
		&\overset{\text{(i)}}{\leq} \sum_{j=1}^h\sum_{(s_{j}, a_{j}) \in \mathcal{J}_j}\widehat{d}_{j}^{\pi}(s_{j}, a_{j}) \notag \\
		 & \overset{\text{(ii)}}{\leq}\sum_{j=1}^{h}\sum_{(s_{j},a_{j})\in\mathcal{J}_{j}}\widehat{d}_{j}^{\pi}(s,a)\cdot\frac{\frac{1}{KH}+\frac{6.5\xi}{N}}{\frac{1}{KH}+\mathbb{E}_{\pi'\sim\widehat{\mu}^{j}}[\widehat{d}_{h}^{\pi'}(s,a)]} \notag \\
 & \leq\left(\frac{1}{KH}+\frac{6.5\xi}{N}\right)\sum_{j=1}^{h}\sum_{(s,a)\in\cS\times\cA}\frac{\frac{1}{KH}+\widehat{d}_{j}^{\pi}(s,a)}{\frac{1}{KH}+\mathbb{E}_{\pi'\sim\widehat{\mu}^{j}}[\widehat{d}_{h}^{\pi'}(s,a)]} .
	\label{eq:dh-sum-UB-lemma-2-proof-intermediate}
	\end{align}
Here, (i) arises from \eqref{eq:proof-6} and the fact $\sum_{s, a} P_{j \to h+1}^{\pi}(s, a \mymid s_{j}, a_{j})= 1$, 
whereas (ii) is a consequence of the definition of $\mathcal{J}_j$ in \eqref{eq:proof-6}. 
Moreover, suppose that the subroutine specified in Algorithm~\ref{alg:sub1} for step $j$ terminates with iterates $\widehat{\mu}^j=\mu^{(t)}$ and $\pi^{(t)}$, 
where we recall that $\pi^{(t)}$ is a deterministic Markov policy chosen to obey 
\begin{equation}
\sum_{(s,a)\in\cS\times\cA}\frac{\widehat{d}_{h}^{\pi^{(t)}}(s,a)}{\frac{1}{KH}+\mathbb{E}_{\pi'\sim\widehat{\mu}^{j}}\big[\widehat{d}_{h}^{\pi'}(s,a)\big]}=\max_{\pi\in\Pi}\sum_{(s,a)\in\cS\times\cA}\frac{\widehat{d}_{h}^{\pi}(s,a)}{\frac{1}{KH}+\mathbb{E}_{\pi'\sim\widehat{\mu}^{j}}\big[\widehat{d}_{h}^{\pi'}(s,a)\big]}.\label{eq:property-pit-step-j}
\end{equation}
The stopping criterion specified in line~\ref{line:stopping-subroutine-h} of Algorithm~\ref{alg:sub1} then reveals that
\begin{equation}
2SA\geq g\big(\pi^{(t)},\widehat{d},\widehat{\mu}^{j}\big)=\sum_{(s,a)\in\cS\times\cA}\frac{\frac{1}{KH}+\widehat{d}_{h}^{\pi^{(t)}}(s,a)}{\frac{1}{KH}+\mathbb{E}_{\pi'\sim\widehat{\mu}^{j}}\big[\widehat{d}_{h}^{\pi'}(s,a)\big]}=\max_{\pi\in\Pi}\sum_{(s,a)\in\cS\times\cA}\frac{\frac{1}{KH}+\widehat{d}_{h}^{\pi}(s,a)}{\frac{1}{KH}+\mathbb{E}_{\pi'\sim\widehat{\mu}^{j}}\big[\widehat{d}_{h}^{\pi'}(s,a)\big]}.\label{eq:property-pit-step-j-max}
\end{equation}
Substitution into \eqref{eq:dh-sum-UB-lemma-2-proof-intermediate} indicates that
\begin{align}
		\sum_{s, a} e_{h+1}^{\pi}(s, a) 
	& \leq  \left(\frac{1}{KH}+\frac{6.5\xi}{N}\right) 2SAH = \frac{2SA}{K}+\frac{13c_{\xi}S^4A^4H^4 \log\frac{HSA}{\delta}}{N}  \notag \\
	& \lesssim \sqrt{\frac{SA}{HK}},
	\label{eq:dh-sum-UB-lemma-2-proof}
\end{align}
where the validity of the last line is guaranteed as long as $K\gtrsim HSA$ and $N \gtrsim \sqrt{H^9S^7A^7K}\log\frac{HSA}{\delta}$.

\paragraph{Step 2.3: controlling the term $\beta_{h+1}^{\pi}(s, a)$.} 
Next, we turn attention to controlling $\beta_{h+1}^\pi (s,a)$. 
Recall from \eqref{eq:beta-defi} that
 \begin{align*}
 	& \beta_{h+1}^{\pi}(s, a) = \sum_{j=1}^h \sum_{(s_j, a_j) : N_j(s_j, a_j) > \xi}\sum_{n = 1}^{N_j(s_j, a_j)} \\
 	&\quad\underbrace{\frac{\widehat{d}_{j}^{\pi}(s_j, a_j)}{N_j(s_j, a_j)} \sum_{s_{j+1}, a_{j+1}} P_{j+1 \to h+1}^{\pi}(s, a \mymid s_{j+1}, a_{j+1})\Big[\ind\big(s_{j+1}^{k_n(s_j, a_j), j} = s_{j+1}\big) - P_j(s_{j+1} \mymid s_j, a_j)\Big]\pi_{j+1}(a_{j+1} \mymid s_{j+1})}_{\eqqcolon X_j^n(s_j,a_j)}, 
 \end{align*}
which we shall bound by resorting to the Bernstein inequality. 
Consider each $1\leq j \leq h$. 
Evidently, when conditional on $\{N_j(s_j,a_j): (s_j,a_j)\in\mathcal{S}\times\mathcal{A}\}$, 
the random variables $\{X_j^n(s_j,a_j): (s_j,a_j)\in\mathcal{S}\times\mathcal{A},1\leq n \leq N_j(s_j,a_j)\}$ are statistically independent with mean zero. 
Let us look at the size of each term as well as the variance statistics separately, both of which are needed in order to apply the Bernstein inequality.

\bigskip
\noindent {\em Step 2.3.1: bounding the size of each $X_j^n(s_j,a_j)$.}
Towards this end, we make the following two observations: 
\begin{itemize}
	\item Firstly, let us write
\begin{align}
 & \sum_{s_{j+1},a_{j+1}}P_{j+1\to h+1}^{\pi}(s,a\mymid s_{j+1},a_{j+1})\ind\big(s_{j+1}^{k_{n}(s_{j},a_{j}),j+1}=s_{j+1}\big)\pi_{j+1}(a_{j+1}\mymid s_{j+1}) \notag\\
 & \qquad=\sum_{s_{j+1}}\ind\big(s_{j+1}^{k_{n}(s_{j},a_{j}),j+1}=s_{j+1}\big)\underset{\eqqcolon Y(s_{j+1},s,a)}{\underbrace{\sum_{a_{j+1}}P_{j+1\to h+1}^{\pi}(s,a\mymid s_{j+1},a_{j+1})\pi_{j+1}(a_{j+1}\mymid s_{j+1})}} , 
	\label{eq:term-1234}
\end{align}
where $Y(s_{j+1},s,a)$ 
is the transition probability from state $s_{j+1}$ in the $(j+1)$-th step to the state-action pair $(s,a)$ in the $(h+1)$-th step under policy $\pi$. 
Hence, the sum in \eqref{eq:term-1234} is bounded above by 1. 

	\item Secondly, it is seen that
\begin{align*}
	\sum_{s_{j+1}, a_{j+1}}P_{j+1 \to h+1}^{\pi}(s, a \mymid s_{j+1}, a_{j+1}) P_j(s_{j+1} \mymid s_j, a_j) \pi_{j+1}(a_{j+1} \mymid s_{j+1}) 
	= P_{j \to h+1}^{\pi}(s, a \mymid s_{j}, a_{j})
\end{align*}
is the transition probability from the state-action pair $(s_j,a_j)$ in the $j$-th step to $(s,a)$ in the $(h+1)$-th step under policy $\pi$, 
		which is again bounded above by 1. 

	\item The preceding two observations taken collectively with the definition of $X_{j}^n(s_j, a_j)$ imply that: for any state-action pair $(s_j,a_j)$ such that $N_j(s_j,a_j)>\xi$, we have
\begin{align}
	\big|X_{j}^n(s_j, a_j)\big| &\leq 2\frac{\widehat{d}_{j}^{\pi}(s_j, a_j)}{N_j(s_j, a_j)}. \label{eq:proof-10}
\end{align}
\end{itemize}

While \eqref{eq:proof-10} already offers an upper bound on the size of $X_{j}^n(s_j, a_j)$, 
we can further bound it by exploiting other properties about $N_j(s_j, a_j)$, given that we have restricted attention to those state-action pairs obeying $N_j(s_j,a_j)>\xi$. 
In view of Bernstein's inequality \citep[Theorem 2.8.4]{vershynin2018high} and the sampling policy $\pi^{\mathsf{explore},j}$ in use (see line~\ref{line:output-Alg-sub1} of Algorithm~\ref{alg:sub1}), there exists some universal constant $\widetilde{C}>0$ such that with probability exceeding $1-\delta$, 
\begin{align}
	\bigg|N_j(s_j, a_j) - N\mathop{\mathbb{E}}_{\pi\sim\widehat{\mu}^j}\big[d_{j}^{\pi}(s_j,a_j)\big] \bigg| 
	\leq \widetilde{C}\sqrt{N \mathop{\mathbb{E}}_{\pi\sim\widehat{\mu}^j}\big[d_{j}^{\pi}(s_j,a_j)\big]\log\frac{HSA}{\delta}} + \widetilde{C} \log\frac{HSA}{\delta} \label{eq:proof-8}
\end{align}
holds for any $1\leq j\leq h$ and $(s_j,a_j)\in\mathcal{S}\times\mathcal{A}$. Therefore, for any $(s_j,a_j)$ such that $N_j(s_j,a_j)>\xi$, we have
\begin{align*}
	\xi < N_j(s_j,a_j) \leq N \mathop{\mathbb{E}}_{\pi\sim\widehat{\mu}^j}\big[d_{j}^{\pi}(s_j,a_j)\big] +\widetilde{C} \sqrt{N \mathop{\mathbb{E}}_{\pi\sim\widehat{\mu}^j}\big[d_{j}^{\pi}(s_j,a_j)\big]\log\frac{HSA}{\delta}} +\widetilde{C} \log\frac{HSA}{\delta} ,
\end{align*}
which together with the choice $\xi = c_\xi H^3S^3A^3 \log\frac{HSA}{\delta}$ indicates that 
\begin{align}
	N \mathop{\mathbb{E}}_{\pi\sim\widehat{\mu}^j}\big[d_{j}^{\pi}(s_j,a_j)\big] \geq \frac{3}{4}\xi 
	= \frac{3c_\xi}{4} H^3S^3A^3 \log\frac{HSA}{\delta}
	\label{eq:N-d-E-3-4-xi}
\end{align} 
as long as $c_\xi>0$ is sufficiently large. This combined with \eqref{eq:proof-8} gives
\begin{align}
N_{j}(s_{j},a_{j}) & \geq N\mathop{\mathbb{E}}_{\pi\sim\widehat{\mu}^{j}}\big[d_{j}^{\pi}(s_{j},a_{j})\big]-\widetilde{C}\sqrt{N\mathop{\mathbb{E}}_{\pi\sim\widehat{\mu}^{j}}\big[d_{j}^{\pi}(s_{j},a_{j})\big]\log\frac{HSA}{\delta}}-\widetilde{C}\log\frac{HSA}{\delta}\nonumber\\
 & \geq\frac{1}{2}N\mathop{\mathbb{E}}_{\pi\sim\widehat{\mu}^{j}}\big[d_{j}^{\pi}(s_{j},a_{j})\big]\geq\frac{1}{4}N\mathop{\mathbb{E}}_{\pi\sim\widehat{\mu}^{j}}\big[d_{j}^{\pi}(s_{j},a_{j})\big]+\frac{3}{16}\xi\nonumber\\
 & \geq\frac{1}{8}N\mathop{\mathbb{E}}_{\pi\sim\widehat{\mu}^{j}}\big[\widehat{d}_{j}^{\pi}(s_{j},a_{j})\big]-\frac{1}{16}\xi+\frac{3}{16}\xi\nonumber\\
	& \geq\frac{1}{8} \Big( N\mathop{\mathbb{E}}_{\pi\sim\widehat{\mu}^{j}}\big[\widehat{d}_{j}^{\pi}(s_{j},a_{j})\big]+1 \Big),
	\label{eq:proof-9}
\end{align}
where the second line makes use of (\ref{eq:N-d-E-3-4-xi}), and the
third line invokes the induction hypothesis (\ref{eq:d-error}) for
the $j$-th step. 
Therefore, we can deduce that
\begin{align}
\sum_{(s_j, a_j) : N_j(s_j, a_j) > \xi}\frac{\widehat{d}_{j}^{\pi}(s_j, a_j)}{N_j(s_j, a_j) } 
	\leq \sum_{(s_j, a_j) : N_j(s_j, a_j) > \xi} \frac{8}{N} \cdot \frac{\widehat{d}_{j}^{\pi}(s_j, a_j)}{1/N +  \mathop{\mathbb{E}}_{\pi'\sim\widehat{\mu}^j}\big[\widehat{d}_{j}^{\pi'}(s_j,a_j)\big]} 
	\leq \frac{16SA}{N}, \label{eq:proof-7}
\end{align}
where the first relation follows from \eqref{eq:proof-9}, and the second relation follows from the assumption $N\leq KH$ and 
inequality \eqref{eq:property-pit-step-j-max}. Combine \eqref{eq:proof-10} and \eqref{eq:proof-7} to yield
\begin{align}
	L_j\coloneqq \max_{(s_j, a_j) : N_j(s_j, a_j) > \xi} \max_{1\leq n\leq N_j(s_j,a_j)} |X_j^n(s_j,a_j)|\leq \frac{16SA}{N}.
	\label{eq:Lj-UB-lemma-2}
\end{align}

\bigskip
\noindent {\em Step 2.3.2: controlling the variance statistics.}  
Consider now any given deterministic Markov policy $\pi\in \Pi$. 
Conditional on $\{N_j(s_j,a_j): (s_j,a_j)\in\mathcal{S}\times\mathcal{A}\}$ and $\{\widehat{d}_{j}^{\pi}(s_{j},a_{j}): (s_j,a_j)\in \cS\times \cA\}$, we can calculate that
\begin{align*}
\mathsf{Var}\big(X_{j}^{n}(s_{j},a_{j})\big) & \leq\bigg(\frac{\widehat{d}_{j}^{\pi}(s_{j},a_{j})}{N_{j}(s_{j},a_{j})}\bigg)^{2}\sum_{s_{j+1},a_{j+1}}P_{j}(s_{j+1}\mymid s_{j},a_{j})\Big[P_{j+1\to h+1}^{\pi}(s,a\mymid s_{j+1},a_{j+1})\pi_{j+1}(a_{j+1}\mymid s_{j+1})\Big]^{2}\\
 & \leq\bigg(\frac{\widehat{d}_{j}^{\pi}(s_{j},a_{j})}{N_{j}(s_{j},a_{j})}\bigg)^{2}\sum_{s_{j+1},a_{j+1}}P_{j+1\to h+1}^{\pi}(s,a\mymid s_{j+1},a_{j+1})P_{j}^{\pi}(s_{j+1},a_{j+1}\mymid s_{j},a_{j})\\
 & =\bigg(\frac{\widehat{d}_{j}^{\pi}(s_{j},a_{j})}{N_{j}(s_{j},a_{j})}\bigg)^{2}P_{j\to h+1}^{\pi}(s,a\mymid s_{j},a_{j}),
\end{align*}
where the second line is valid since 
\[
P_{j}(s_{j+1}\mymid s_{j},a_{j})\pi_{j+1}(a_{j+1}\mymid s_{j+1})P_{j+1\to h+1}^{\pi}(s,a\mymid s_{j+1},a_{j+1})=P_{j+1\to h+1}^{\pi}(s,a\mymid s_{j+1},a_{j+1})P_{j}^{\pi}(s_{j+1},a_{j+1}\mymid s_{j},a_{j}).
\]
As a consequence,  conditional on $\{N_j(s_j,a_j): (s_j,a_j)\in\mathcal{S}\times\mathcal{A}\}$  and $\{\widehat{d}_{j}^{\pi}(s_{j},a_{j}): (s_j,a_j)\in \cS\times \cA\}$ one has
\begin{align}
V_{j} & \coloneqq\sum_{(s_{j},a_{j}):N_{j}(s_{j},a_{j})>\xi}\sum_{n=1}^{N_{j}(s_{j},a_{j})}\mathsf{Var}\big(X_{j}^{n}(s_{j},a_{j})\big) \notag\\
 & \leq\sum_{(s_{j},a_{j}):N_{j}(s_{j},a_{j})>\xi}\frac{\big[\widehat{d}_{j}^{\pi}(s_{j},a_{j})\big]^{2}}{N_{j}(s_{j},a_{j})}P_{j\to h+1}^{\pi}(s,a\mymid s_{j},a_{j}) \notag\\
 & \le\Bigg(\max_{(s_{j},a_{j}):N_{j}(s_{j},a_{j})>\xi}\frac{\widehat{d}_{j}^{\pi}(s_{j},a_{j})}{N_{j}(s_{j},a_{j})}\Bigg)\bigg(\sum_{s_{j},a_{j}}\widehat{d}_{j}^{\pi}(s_{j},a_{j})P_{j\to h+1}^{\pi}(s,a\mymid s_{j},a_{j})\bigg) \notag\\
 & \leq\frac{16SA}{N}\Big(2d_{h+1}^{\pi}(s,a)+\frac{SA\xi}{2N}\Big).
	\label{eq:Vj-UB-lemma-2}
\end{align}
Here, the last inequality follows from \eqref{eq:proof-7} and
\begin{align*}
\sum_{s_{j},a_{j}}\widehat{d}_{j}^{\pi}(s_{j},a_{j})P_{j\to h+1}^{\pi}(s,a\mymid s_{j},a_{j}) & \leq2\sum_{s_{j},a_{j}}\left(d_{j}^{\pi}(s_{j},a_{j})+\frac{\xi}{4N}\right)P_{j\to h+1}^{\pi}(s,a\mymid s_{j},a_{j})\\
 & \leq2d_{h+1}^{\pi}(s,a)+\frac{SA\xi}{2N},
\end{align*}
where the first line invokes the induction hypothesis \eqref{eq:d-error} for the $j$-th step.

\bigskip
\noindent {\em Step 2.3.3: invoking the Bernstein inequality and union bound.}
Armed with the above results, we now develop some concentration bounds for $\beta_{h+1}^{\pi}(s, a)$. 
Towards this, let us begin by looking at any given deterministic Markov policy $\pi\in \Pi$. 
By virtue of the Bernstein inequality, there exist some universal constants $C_1,C_2>0$  such that: 
for any $1\leq j\leq h$ and any given $(s,a)\in \cS\times \cA$, 
%conditional on $\{N_j(s_j,a_j): (s_j,a_j)\in\mathcal{S}\times\mathcal{A}\}$  and $\{\widehat{d}_{j}^{\pi}(s_{j},a_{j}): (s_j,a_j)\in \cS\times \cA\}$, 
%
\begin{align}
	& \Bigg|\sum_{(s_{j},a_{j}):N_{j}(s_{j},a_{j})>\xi}\sum_{n=1}^{N_{j}(s_{j},a_{j})}X_{j}^{n}(s_{j},a_{j})\Bigg|  \leq C_{1}\left[\sqrt{V_{j}\log\frac{|\Pi|SAH}{\delta}}+L_{j}\log\frac{|\Pi|SAH}{\delta}\right] \notag\\
 & \qquad \leq C_{2}\sqrt{\frac{HS^{2}A}{N}\Big(d_{h+1}^{\pi}(s,a)+\frac{SA\xi}{N}\Big)\log\frac{SAH}{\delta}}+C_{2}\frac{HS^{2}A}{N}\log\frac{SAH}{\delta} \notag\\
 & \qquad \leq C_{2}\sqrt{\frac{HS^{2}A}{N}d_{h+1}^{\pi}(s,a)\log\frac{SAH}{\delta}}+C_{2}\sqrt{\frac{HS^{3}A^{2}\xi}{N^{2}}\log\frac{SAH}{\delta}}+C_{2}\frac{HS^{2}A}{N}\log\frac{SAH}{\delta} 
	\label{eq:one-term-beta-control-UB}
\end{align}	
holds true with probability at least $1 - \delta / (|\Pi|HSA)$, 
where the second line relies on \eqref{eq:Lj-UB-lemma-2}, \eqref{eq:Vj-UB-lemma-2} as well as the fact that $|\Pi| \le A^{HS}$, 
and the third line invokes the elementary inequality $\sqrt{x+y}\leq \sqrt{x}+\sqrt{y}$. 
Taking the union bound over all $(s,a,j)\in \cS\times \cA\times [h]$ and substituting \eqref{eq:one-term-beta-control-UB} into the expression $\beta_{h+1}^{\pi}(s,a)$ yield: 
with probability exceeding $1-\delta$, 
\begin{align*}
\big|\beta_{h+1}^{\pi}(s, a)\big| 
%& \leq C_1 \sum_{j=1}^h \left[\sqrt{V_j\log\frac{|\Pi|SAH}{\delta}} + L_j\log\frac{|\Pi|SAH}{\delta} \right] \\
%&\leq C_2 H\sqrt{\frac{HS^2A}{N}\Big(d_{h+1}^{\pi}(s, a) + \frac{SA\xi}{N}\Big)\log \frac{SAH}{\delta}} + C_2 \frac{H^2S^2A}{N}\log \frac{SAH}{\delta} \\
&\leq C_2 H\sqrt{\frac{HS^2A}{N} d_{h+1}^{\pi}(s, a) \log \frac{SAH}{\delta}} 
	+ C_2 H\sqrt{\frac{HS^3A^2 \xi }{N^2}\log \frac{SAH}{\delta}} + C_2 \frac{H^2S^2A}{N}\log \frac{SAH}{\delta} \\
	&\leq \frac{1}{2} d^\pi_{h+1}(s,a) + \frac{1}{2} C_2^2 \frac{H^3 S^2A}{N} \log\frac{SAH}{\delta} + \frac{\xi}{8N} + 2C_2^2 \frac{H^3S^3A^2}{N}\log \frac{SAH}{\delta} + C_2 \frac{H^2S^2A}{N}\log \frac{SAH}{\delta} \\
& \leq \frac{1}{2} d^\pi_{h+1}(s,a) + \frac{\xi}{4N}
\end{align*}
holds simultaneously for all $(s, a)\in\mathcal{S}\times\mathcal{A}$ and all deterministic Markov policies $\pi\in \Pi$.  
Here,  
the penultimate relation follows from  the AM-GM inequality, whereas the last inequality holds as long as $c_\xi>0$ is sufficiently large. 
Taking the union bound over all deterministic policies $\pi\in \Pi$, we arrive at
\begin{align}
	\big|\beta_{h+1}^{\pi}(s, a)\big| &  \leq \frac{1}{2} d^\pi_{h+1}(s,a) + \frac{\xi}{4N}
	\qquad \text{for all }\pi \in\Pi \text{ and }(s,a)\in \cS\times \cA. 
	\label{eq:beta-UB-lemma2}
\end{align}
%

%We  have thus established the advertised bound \eqref{eq:condition} for $\beta_{h+1}^{\pi}(s, a)$.

\paragraph{Step 2.3: putting everything together.}

Substituting \eqref{eq:dh-sum-UB-lemma-2-proof} and \eqref{eq:beta-UB-lemma2} into \eqref{eq:proof-5}, 
we immediately establish the claim \eqref{eq:d-error} for step $h+1$.  
This together with standard induction arguments concludes the proof of Lemma~\ref{lemma:occupancy}.

%\begin{lemma} \label{lem:det}
%For any projection $\phi : \mathcal{X} \to \mathbb{R}^d$ and probability set $\mathcal{V}$ on $\mathcal{X}$, 
%one has
%\begin{align}
%\max_{\nu \in \mathcal{V}} \Big\langle \mathbb{E}_{\nu}\big[\phi(x)\phi(x)^{\top}\big], \Big(\epsilon I + \mathbb{E}_{\widehat{\mu}}\big[\phi(x)\phi(x)^{\top}\big]\Big)^{-1}\Big\rangle \le d,
%\end{align}
%where $\epsilon > 0$ is any quantity, and
%\begin{align}
%\widehat{\mu} = \arg\max_{\mu \in \Delta(\nu)} \mathrm{det}\Big(\epsilon I + \mathbb{E}_{\mu}\big[\phi(x)\phi(x)^{\top}\big]\Big).
%\end{align}
%\end{lemma}
%
%\begin{proof}
%This proof follows from...
%Let $M(\mu) := \epsilon I + \mathbb{E}_{\mu}\big[\phi(x)\phi(x)^{\top}\big]$.
%With a little calculation, one has
%\begin{align}
%0 \ge \frac{\partial}{\partial \alpha} \log \mathrm{det}\big[M((1-\alpha)\widehat{\mu} + \alpha \pi)\big] \Big|_{\alpha = 0_+} = \Big\langle M(\pi), \big[M(\widehat{\mu})\big]^{-1} \Big\rangle - d,
%\end{align}
%and the conclusion is ready.
%\end{proof}
%

\section{Analysis for reward-free exploration}
\label{sec:analysis-RFE}

We now turn attention to the proof of Theorem~\ref{thm:main-RFE}, 
which is concerned with the reward-free setting that seeks to cover all possible reward functions simultaneously. 
The proof of Theorem~\ref{thm:main-RFE} largely resembles that of Theorem~\ref{thm:main}, 
except that we are in need of a different uniform concentration result as follows. 
\begin{lemma} \label{lemma:reward-free-concentration}
Let $\widehat{P}=\{\widehat{P}_h\}_{1\leq h\leq H}$ denote the empirical transition kernel constructed in Algorithm~\ref{alg:offline_RL}. 
With probability at least $1-\delta$, one has
\begin{subequations}
\begin{align}
	\big|\big(\widehat{P}_{h, s, a} - P_{h, s, a}\big)V\big| &\leq \sqrt{\frac{48S}{\widehat{N}_{h}^{\mathsf{b}}\left(s,a\right)}\mathsf{Var}_{\widehat{P}_{h,s,a}}\big(V\big)\log\frac{KH}{\delta}}+\frac{64HS}{\widehat{N}_{h}^{\mathsf{b}}\left(s,a\right)}\log\frac{KH}{\delta} \\
	\mathsf{Var}_{\widehat{P}_{h,s,a}}
	&\leq 8\mathsf{Var}_{P_{h,s,a}}\big(V\big)+\frac{10H^{2}S}{\widehat{N}_{h}^{\mathsf{b}}\left(s,a\right)}\log\frac{KH}{\delta}
\end{align}
\end{subequations}
	hold simultaneously for all $V \in \mathbb{R}^{S}$ obeying $\|V\|_{\infty} \le H$ and all $(s,a,h)\in \cS\times \cA\times [H]$.
\end{lemma}
In words, this lemma extends the result in Lemma~\ref{lem:model-based-offline-prior} to provide uniform control over all possible vectors $V$. 
The proof of this lemma is provided in Section \ref{sec:proof-lemma-reward-free-concentration}.

Armed with Lemma~\ref{lemma:reward-free-concentration}, 
we can repeat the same analysis as in \citet[Section 7]{li2022settling} to demonstrate that: with probability exceeding $1-\delta$, 
\begin{align}
	\big\langle d_{h}^{\pi^{\star}},V_{h}^{\star}-V_{h}^{\widehat{\pi}}\big\rangle 
	\leq2\sum_{j:j\geq h}\sum_{(s,a)\in\mathcal{S}\times\mathcal{A}}d_{j}^{\pi^{\star}}(s,a)b_{j}(s,a)
\end{align}
holds simultaneously for all possible reward functions, where the penalty function $b_j(s,a)$ is defined in \eqref{def:bonus-Bernstein-finite-free}. 
As a result, 
one can then repeat the same analysis as in the proof of Theorem \ref{thm:main} in Section \ref{sec:proof-main} to establish the desired sample complexity for the reward-free case. 
The only additional thing that needs to be taken care of now is that we should replace the penalty function $b_j(s,a)$ defined in \eqref{def:bonus-Bernstein-finite} for the reward-agnostic counterpart with  \eqref{def:bonus-Bernstein-finite-free} for the reward-free case. This in turn leads to
\begin{align*}
	\big\langle d_{j}^{\pi^{\star}},V_{j}^{\star}-V_{j}^{\widehat{\pi}}\big\rangle & \lesssim\sqrt{\frac{H^{3}S^2A}{K}\log\frac{KH}{\delta}}+\frac{H^{2}S^2A\log\frac{KH}{\delta}}{K}+H^{2}\Big(\sqrt{\frac{SA}{HK}}+\frac{SA\xi}{N}\Big) \asymp\sqrt{\frac{H^{3}S^2A}{K}\log\frac{KH}{\delta}}
\end{align*}
as long as $KH\geq N \gtrsim \sqrt{H^9 S^7 A^7 K} \log\frac{HSA}{\delta}$, thus indicating that 
\begin{align*}
	V_{1}^{\star}(\rho)-V_{1}^{\widehat{\pi}}(\rho)= & \big\langle d_{1}^{\pi^{\star}},V_{1}^{\star}-V_{1}^{\widehat{\pi}}\big\rangle\lesssim\sqrt{\frac{H^{3}S^2A}{K}\log\frac{KH}{\delta}}\leq\varepsilon,
\end{align*}
as long as $K\geq\frac{c_{K}H^{3}S^2A\log\frac{KH}{\delta}}{\varepsilon^{2}}$
for some sufficiently large constant $c_{K}>0$. 
We omite other details for the sake of brevity.

\subsection{Proof of Lemma \ref{lemma:reward-free-concentration}} \label{sec:proof-lemma-reward-free-concentration}

	Let us construct an $\epsilon$-net $\mathcal{N}$ for the set $[0,H]^{S}$ under
	metric $\Vert\cdot\Vert_{\infty}$; as is well known,  one can choose $\mathcal{N}$ such that $|\mathcal{N}\leq \vert (H/\epsilon)^{S}$ \citep{vershynin2018high}.
	In view of \citet[Lemma 8]{li2022settling}, 
	we know that  any vector $\widetilde{V}\in\mathcal{N}$ satisfies
	\[
		\left|\big(\widehat{P}_{h,s,a}-P_{h,s,a}\big)\widetilde{V}\right|\leq\sqrt{\frac{48}{\widehat{N}_{h}^{\mathsf{b}}\left(s,a\right)}\mathsf{Var}_{\widehat{P}_{h,s,a}}\big(\widetilde{V}\big)\log \frac{KH\vert\mathcal{N}\vert}{\delta}}+\frac{48H}{\widehat{N}_{h}^{\mathsf{b}}\left(s,a\right)}\log \frac{KH\vert\mathcal{N}\vert}{\delta}
	\]
	and
	\[
	\mathsf{Var}_{\widehat{P}_{h,s,a}}\big(\widetilde{V}\big)\leq2\mathsf{Var}_{P_{h,s,a}}\big(\widetilde{V}\big)+\frac{5H^{2}}{3\widehat{N}_{h}^{\mathsf{b}}\left(s,a\right)}\log \frac{KH\vert\mathcal{N}\vert}{\delta}
	\]
	with probability exceeding $1-\delta/\vert\mathcal{N}\vert$. Taking the union bound over all $\widetilde{V}\in\mathcal{N}$,
	we know that with probability exceeding $1-\delta$, 
	\[
	\left|\big(\widehat{P}_{h,s,a}-P_{h,s,a}\big)\widetilde{V}\right|\leq\sqrt{\frac{48S}{\widehat{N}_{h}^{\mathsf{b}}\left(s,a\right)}\mathsf{Var}_{\widehat{P}_{h,s,a}}\big(\widetilde{V}\big)\log \frac{KH^2}{\delta\epsilon}}+\frac{48HS}{\widehat{N}_{h}^{\mathsf{b}}\left(s,a\right)}\log \frac{KH^2}{\delta\epsilon}
	\]
	and
	\begin{align} \label{eq:var-UB-uniform-eps-net}
	\mathsf{Var}_{\widehat{P}_{h,s,a}}\big(\widetilde{V}\big)\leq 2\mathsf{Var}_{P_{h,s,a}}\big(\widetilde{V}\big)+\frac{5H^{2}S}{3\widehat{N}_{h}^{\mathsf{b}}\left(s,a\right)}\log \frac{KH^2}{\delta\epsilon}
	\end{align}
	hold simultaneously for all $\widetilde{V}\in\mathcal{N}$.

	We now look at an arbitrary  $V\in\mathbb{R}^{S}$
	obeying $\Vert V\Vert_{\infty}\leq H$. 
	From the definition of the $\epsilon$-net, 
	we know the existence of some $\widetilde{V}\in\mathcal{N}$ 
	such that $\Vert V-\widetilde{V}\Vert_{\infty}\leq\epsilon$. 
	By choosing $\epsilon = 1/K$, we can deduce that
\begin{align*}
 & \left|\big(\widehat{P}_{h,s,a}-P_{h,s,a}\big)V\right|\leq\left|\big(\widehat{P}_{h,s,a}-P_{h,s,a}\big)\widetilde{V}\right|+2\big\Vert\widetilde{V}-V\big\Vert_{\infty}\\
 & \quad\leq\sqrt{\frac{48S}{\widehat{N}_{h}^{\mathsf{b}}\left(s,a\right)}\mathsf{Var}_{\widehat{P}_{h,s,a}}\big(\widetilde{V}\big)\log\frac{KH^{2}}{\delta\epsilon}}+\frac{48HS}{\widehat{N}_{h}^{\mathsf{b}}\left(s,a\right)}\log\frac{KH^{2}}{\delta\epsilon}+2\epsilon\\
 & \quad\leq\sqrt{\frac{96S}{\widehat{N}_{h}^{\mathsf{b}}\left(s,a\right)}\mathsf{Var}_{\widehat{P}_{h,s,a}}\big(V\big)\log\frac{KH^{2}}{\delta\epsilon}}+\sqrt{\frac{96S}{\widehat{N}_{h}^{\mathsf{b}}\left(s,a\right)}\mathsf{Var}_{\widehat{P}_{h,s,a}}\big(V-\widetilde{V}\big)\log\frac{KH^{2}}{\delta\epsilon}}+\frac{48HS}{\widehat{N}_{h}^{\mathsf{b}}\left(s,a\right)}\log\frac{KH^{2}}{\delta\epsilon}+2\epsilon\\
 & \quad\leq\sqrt{\frac{96S}{\widehat{N}_{h}^{\mathsf{b}}\left(s,a\right)}\mathsf{Var}_{\widehat{P}_{h,s,a}}\big(V\big)\log\frac{KH^{2}}{\delta\epsilon}}+\sqrt{\frac{96S\epsilon^{2}}{\widehat{N}_{h}^{\mathsf{b}}\left(s,a\right)}\log\frac{KH^{2}}{\delta\epsilon}}+\frac{48HS}{\widehat{N}_{h}^{\mathsf{b}}\left(s,a\right)}\log\frac{KH^{2}}{\delta\epsilon}+2\epsilon\\
 & \quad\leq\sqrt{\frac{48S}{\widehat{N}_{h}^{\mathsf{b}}\left(s,a\right)}\mathsf{Var}_{\widehat{P}_{h,s,a}}\big(\widetilde{V}\big)\log\frac{KH^{2}}{\delta\epsilon}}+\frac{60HS}{\widehat{N}_{h}^{\mathsf{b}}\left(s,a\right)}\log\frac{KH^{2}}{\delta\epsilon}+2\epsilon^{2}+2\epsilon\\
 & \quad\leq\sqrt{\frac{48S}{\widehat{N}_{h}^{\mathsf{b}}\left(s,a\right)}\mathsf{Var}_{\widehat{P}_{h,s,a}}\big(\widetilde{V}\big)\log\frac{KH^{2}}{\delta\epsilon}}+\frac{64HS}{\widehat{N}_{h}^{\mathsf{b}}\left(s,a\right)}\log\frac{KH^{2}}{\delta\epsilon},
\end{align*}
where the third line is valid since $\mathsf{Var}(X+Y)\leq2\mathsf{Var}(X)+2\mathsf{Var}(Y)$,
the penultimate line follows from the AM-GM inequality, and the last
inequality holds since $\epsilon=1/K\leq1/\widehat{N}_{h}^{\mathsf{b}}\left(s,a\right)$. 
In addition, one also obtains
\begin{align*}
\mathsf{Var}_{\widehat{P}_{h,s,a}}\big(V\big) & \leq2\mathsf{Var}_{\widehat{P}_{h,s,a}}\big(\widetilde{V}\big)+2\mathsf{Var}_{\widehat{P}_{h,s,a}}\big(V-\widetilde{V}\big)\leq2\mathsf{Var}_{\widehat{P}_{h,s,a}}\big(\widetilde{V}\big)+2\big\Vert\widetilde{V}-V\big\Vert_{\infty}^{2}\\
 & \leq4\mathsf{Var}_{P_{h,s,a}}\big(\widetilde{V}\big)+\frac{10H^{2}S}{3\widehat{N}_{h}^{\mathsf{b}}\left(s,a\right)}\log\frac{KH^{2}}{\delta\varepsilon}+2\big\Vert\widetilde{V}-V\big\Vert_{\infty}^{2}\\
 & \leq8\mathsf{Var}_{P_{h,s,a}}\big(V\big)+8\mathsf{Var}_{P_{h,s,a}}\big(V-\widetilde{V}\big)+\frac{10H^{2}S}{3\widehat{N}_{h}^{\mathsf{b}}\left(s,a\right)}\log\frac{KH^{2}}{\delta\varepsilon}+2\big\Vert\widetilde{V}-V\big\Vert_{\infty}^{2}\\
 & \leq8\mathsf{Var}_{P_{h,s,a}}\big(V\big)+\frac{10H^{2}S}{3\widehat{N}_{h}^{\mathsf{b}}\left(s,a\right)}\log\frac{KH^{2}}{\delta\varepsilon}+10\big\|\widetilde{V}-V\big\|_{\infty}^{2}\\
 & \leq8\mathsf{Var}_{P_{h,s,a}}\big(V\big)+\frac{10H^{2}S}{3\widehat{N}_{h}^{\mathsf{b}}\left(s,a\right)}\log\frac{KH^{2}}{\delta\varepsilon}+10\epsilon^{2}\\
 & \leq8\mathsf{Var}_{P_{h,s,a}}\big(V\big)+\frac{10H^{2}S}{\widehat{N}_{h}^{\mathsf{b}}\left(s,a\right)}\log\frac{KH^{2}}{\delta\varepsilon},
\end{align*}
where the first and the third line arise since $\mathsf{Var}(X+Y)\leq2\mathsf{Var}(X)+2\mathsf{Var}(Y)$,
and the second line comes from \eqref{eq:var-UB-uniform-eps-net}. This concludes the proof of this lemma.

\section{Discussion}
\label{sec:discussion}

In this paper, we have introduced a reward-agnostic pure exploration algorithm that works statistically optimally for two scenarios:
(a) the case where there are at most polynomially many reward functions of interest; 
(b) the case where one needs to accommodate an arbitrarily large number of unknown reward functions. 
Drawing on insights from recent advances in offline reinforcement learning, 
our algorithm design differs drastically from prior reward-agnostic/reward-free algorithms, 
and illuminates new connections between online and offline reinforcement learning.

We conclude this paper by pointing out a few directions worthy of future investigation. 
To begin with, the current algorithm is only guaranteed to work when the target accuracy level $\varepsilon$ is small enough; put another way, this means that our algorithm might incur a high burn-in cost, 
so that its optimality does not come into effect  until the  total sample size exceeds the burn-in cost. Can we hope to improve the algorithm design so as to cover the full $\varepsilon$ range?  
Additionally, how to extend the ideas of the current algorithm to accommodate the case where low-dimensional representation of the MDP is available, 
in the hope of further enhancing data efficiency?  
Furthermore, the notion of minimax optimality might be too conservative in some practical applications. It would be of interest to design more adaptive exploration paradigms that achieve some sort of instance optimality, if not infeasible.

\section*{Acknowledgements}

Y.~Chen is supported in part by the Alfred P.~Sloan Research Fellowship, the Google Research Scholar Award, the AFOSR grant FA9550-22-1-0198, 
the ONR grant N00014-22-1-2354,  and the NSF grants CCF-2221009 and CCF-1907661. Y.~Yan is supported in part by the Charlotte Elizabeth Procter Honorific Fellowship from Princeton University and the Norbert Wiener Postdoctoral Fellowship from MIT.  J.~Fan's research is partially supported by the NSF grants DMS-2210833 and ONR grant N00014-22-1-2340.
%, IIS-2218713 and IIS-2218773  and DMS-2014279. 

\appendix

\section{A pessimistic model-based algorithm for offline RL}
\label{sec:offline-algorithm}

In this section, we present the precise procedure for the model-based offline RL algorithm studied in \citet{li2022settling}. 
Before proceeding, we first convert the $K$ sample episodes collected in Stage 1.2~into a dataset of the following form: 
\begin{align}
	\mathcal{D} = \big\{ \big(s_h^{n,\mathsf{b}}, a_h^{n,\mathsf{b}}, s_{h+1}^{n,\mathsf{b}} \big) \big\}_{1\leq n\leq K, 1 \le h < H}, 
	\label{eq:dataset-transition-D}
\end{align}
comprising all sample transitions in these $K$ episodes. 
In particular, for each $(s,a,h)\in \cS\times \cA\times [H]$ we define 
\begin{align} \label{eq:defn-Nh-sa-finite}
	N_h(s,a) &\coloneqq \sum_{n=1}^K \mathds{1} \big\{ (s_h^{n,\mathsf{b}}, a_h^{n,\mathsf{b}}) = (s,a) \big\} ,
\end{align}
which stands for the total number of sample transitions from the state-action pair $(s,a)$ at the $h$-th step.

\paragraph{Empirical MDP.}
Given the data set $\mathcal{D}$ in \eqref{eq:dataset-transition-D}, 
we first subsample   $\mathcal{D}$ to obtain another dataset $\mathcal{D}^{\mathsf{trim}}$ 
such that for each $(s,a,h)$, $\mathcal{D}^{\mathsf{trim}}$ contains exactly $\min\big\{\widehat{N}^{\mathsf{b}}_h(s, a), N_h(s,a)\big\}$ sample transitions from $(s,a)$ at step $h$. 
Here, we remind the reader of the definition of $\widehat{N}^{\mathsf{b}}_h(s, a)$ in \eqref{eq:lower-samples}. 
We can then compute the empirical transition kernel $\widehat{P}=\{\widehat{P}_h\}_{1\leq h\leq H}$ as follows:
\begin{align}
	\widehat{P}_{h}(s'\mymid s,a) = 
	\frac{\ind\big(\min\big\{\widehat{N}^{\mathsf{b}}_h(s, a), N_h(s,a)\big\} > 0\big)}{\min\big\{\widehat{N}^{\mathsf{b}}_h(s, a), N_h(s,a)\big\}} \sum\limits_{(s_i, a_i, h_i, s_i') \in \mathcal{D}^{\mathsf{trim}}} \mathds{1} \big\{ (s_i, a_i, h_i, s_i') = (s,a,h,s') \big\} 
	\label{eq:empirical-P-finite}
\end{align}
for each $(s,a,h,s')\in \cS\times \cA\times [H] \times \cS$. Here, we abuse the notation $\widehat{P}$ as long as it is clear from the context.

\paragraph{Pessimism in the face of uncertainty.} 
An effective strategy to solve offline RL is to resort to the pessimism principle in the face of uncertainty. 
Towards this end, one needs to specify how to quantify the uncertainty of value estimation, which specify now. 
In our algorithm, we choose the penalty term $b_h(s, a)$ for each $(s,a,h)\in \cS\times \cA\times [H]$ based on Bernstein-style concentration bounds; 
more specifically, 
\begin{itemize}
	\item In the reward-agnostic setting, the Bernstein-style penalty is chosen to be
\begin{align}
	b_h(s, a) = \min \Bigg\{ \sqrt{\frac{\cb\log\frac{HSA}{\delta}}{\min\big\{\widehat{N}^{\mathsf{b}}_h(s, a), N_h(s,a)\big\}}\mathsf{Var}_{\widehat{P}_{h, s, a}}\big(\widehat{V}_{h+1}\big)} + \cb H\frac{\log\frac{HSA}{\delta}}{\min\big\{\widehat{N}^{\mathsf{b}}_h(s, a), N_h(s,a)\big\}} ,\, H \Bigg\} 
	\label{def:bonus-Bernstein-finite}
\end{align}
for some universal constant $\cb> 0$. Here, $\mathsf{Var}_{\widehat{P}_{h, s, a}}\big(\widehat{V}_{h+1}\big)$ corresponds to the variance of $\widehat{V}_{h+1}$ w.r.t.~the distribution $\widehat{P}_{h,s,a}$. 

	\item In the reward-free setting, the penalty term is taken as
\begin{align}
	b_h(s, a) = \min \Bigg\{ \sqrt{\frac{\cb S\log\frac{HSA}{\delta}}{\min\big\{\widehat{N}^{\mathsf{b}}_h(s, a), N_h(s,a)\big\}}\mathsf{Var}_{\widehat{P}_{h, s, a}}\big(\widehat{V}_{h+1}\big)} + \cb SH\frac{\log\frac{HSA}{\delta}}{\min\big\{\widehat{N}^{\mathsf{b}}_h(s, a), N_h(s,a)\big\}} ,\, H \Bigg\}.
	\label{def:bonus-Bernstein-finite-free}
\end{align}
\end{itemize}
With such penalty terms in place, 
we are ready to present the whole model-based offline RL algorithm in Algorithm~\ref{alg:vi-lcb-finite}.

\begin{algorithm}[!h]
	\DontPrintSemicolon
	\SetNoFillComment
	\vspace{-0.3ex}
	\textbf{input:} a dataset $\mathcal{D} = \big\{ (s_h^{n,\mathsf{b}}, a_h^{n,\mathsf{b}}, s_{h+1}^{n,\mathsf{b}}) \big\}_{1\leq n\leq K, 1 \le h < H}$; reward function $r$. \\
	% $N_h(s,a)$:  number of sample transitions in $\mathcal{D}_0$ for each $(s,a,h)\in \cS\times \cA\times [H]$; 
	\textbf{initialization:} $\widehat{V}_{H+1}=0$. \\
	
	\textbf{subsampling:} 
	compute the lower bound of the number of sample transitions $\widehat{N}^{\mathsf{b}}$ according to~\eqref{eq:lower-samples}; \\
	subsample $\mathcal{D}$ to obtain $\mathcal{D}^{\mathsf{trim}}$, such that for each $(s,a,h)\in \cS\times\cA\times [H]$, $\mathcal{D}^{\mathsf{trim}}$ contains $\min\big\{\widehat{N}^{\mathsf{b}}_h(s, a), N_h(s,a)\big\}$ sample transitions randomly drawn from $\mathcal{D}$.

   \For{$h=H,\cdots,1$}
	{
		compute the empirical transition kernel $\widehat{P}_h$ according to \eqref{eq:empirical-P-finite}. \\
		\For{$s\in \cS, a\in \cA$}{
			compute the penalty term $b_h(s,a)$ according to \eqref{def:bonus-Bernstein-finite} for the reward-agnostic case or \eqref{def:bonus-Bernstein-finite-free} for the reward-free case. \\
			set $\widehat{Q}_h(s, a) = \max\big\{r_h(s, a) + \widehat{P}_{h, s, a} \widehat{V}_{h+1} - b_h(s, a), 0\big\}$. \\ 
		}
		%\blue{O.} \\ 
		\For{$s\in \cS$}{
			set $\widehat{V}_h(s) = \max_a \widehat{Q}_h(s, a)$ and $\widehat{\pi}_h(s) = \arg\max_{a} \widehat{Q}_h(s,a)$.
		}
	}

	\textbf{output:} $\widehat{\pi}=\{\widehat{\pi}_h\}_{1\leq h\leq H}$. 
	\caption{Pessimistic model-based offline RL.\label{alg:offline_RL}}
 \label{alg:vi-lcb-finite}
\end{algorithm}

\section{Proof of Lemma~\ref{lem:iteration-complexity-subroutine}}
\label{sec:proof-lem:iteration-complexity-subroutine}

If $g\big(\pi^{(t)},\widehat{d},\mu_{\mathsf{b}}^{(t)}\big) \geq 2HSA$, 
then the learning rate necessarily obeys 
\begin{align}
	\frac{1}{2HSA-1} \le \alpha_t < \frac{1}{HSA}.
	\label{eq:range-alphat-FW}
\end{align}
This combined with the update rule \eqref{eq:update-rule-FW-b} allows one to lower bound the progress made in each iteration:  
\begin{align}
 & \sum_{(h,s,a)\in[H]\times\cS\times\cA}\left\{ \log\bigg[\frac{1}{KH}+\mathbb{E}_{\pi\sim\mu_{\mathsf{b}}^{(t+1)}}\big[\widehat{d}_{h}^{\pi}(s,a)\big]\bigg]-\log\bigg[\frac{1}{KH}+\mathbb{E}_{\pi\sim\mu_{\mathsf{b}}^{(t)}}\big[\widehat{d}_{h}^{\pi}(s,a)\big]\bigg]\right\} \notag\\
 & \qquad=\sum_{(h,s,a)\in[H]\times\cS\times\cA}\log\Bigg[\frac{\frac{1}{KH}+\mathbb{E}_{\pi\sim(1-\alpha_{t})\mu_{\mathsf{b}}^{(t)}(\pi)+\alpha_{t}\ind(\pi^{(t)})}\big[\widehat{d}_{h}^{\pi}(s,a)\big]}{\frac{1}{KH}+\mathbb{E}_{\pi\sim\mu_{\mathsf{b}}^{(t+1)}}\big[\widehat{d}_{h}^{\pi}(s,a)\big]}\Bigg]\notag\\
 & \qquad=\sum_{(h,s,a)\in[H]\times\cS\times\cA}\log\Bigg[1+\alpha_{t}\bigg(\frac{\frac{1}{KH}+\widehat{d}_{h}^{\pi^{(t)}}(s,a)}{\frac{1}{KH}+\mathbb{E}_{\pi\sim\mu_{\mathsf{b}}^{(t+1)}}\big[\widehat{d}_{h}^{\pi}(s,a)\big]}-1\bigg)\Bigg]\notag\\
 & \qquad\overset{\mathrm{(i)}}{\geq}\,\log\left[1+\alpha_{t}\Big(g\big(\pi^{(t)},\widehat{d},\mu_{\mathsf{b}}^{(t)}\big)-1\Big)\right]+(HSA-1)\log\left(1-\alpha_{t}\right)\notag\\
 & \qquad\overset{\mathrm{(ii)}}{\geq}\,\log\left[1+\alpha_{t}\big(2HSA-1\big)\right]+(HSA-1)\log\left(1-\alpha_{t}\right)\notag\\
 & \qquad\overset{\mathrm{(iii)}}{\geq}\,\min\left\{ \log2+(HSA-1)\log\left(1-\frac{1}{2HSA-1}\right),\log\left(3-\frac{1}{HSA}\right)+(HSA-1)\log\left(1-\frac{1}{HSA}\right)\right\} \notag\\
 & \qquad\overset{\mathrm{(iv)}}{\geq}0.09,
	\label{eq:per-iteration-progress}
\end{align}
where (ii) arises since  $g\big(\pi^{(t)},\widehat{d},\mu_{\mathsf{b}}^{(t)}\big) \ge 2HSA$, 
(iii) makes use of~\eqref{eq:range-alphat-FW}, 
and (iv) is derived by numerically calculating the function $\min\big\{ \log 2 + (x-1)\log\big(1-\frac{1}{2x-1}\big), \log\big(3-\frac{1}{x}\big) + (x-1)\log\big(1-\frac{1}{x}\big) \big\}$ 
for $x>1$.
%the last line is valid due to \eqref{eq:range-alphat-FW} and the fact that $\log(1+x)\geq 0.54x$ for any $0<x\leq 2$.   
To see why (i) holds, we make note of the following fact that holds  
 for any $x_1, x_2 \ge 0$ and any $\alpha\in (0,1)$: 
\begin{align*}
\log\big[1 + \alpha(x_1 - 1)\big] + \log\big[1 + \alpha(x_2 - 1)\big] &= \log\big[1 + \alpha(x_1 + x_2 - 2) + \alpha^2(x_1 - 1)(x_2 - 1)\big] \\
&\ge \log\big[1 + \alpha(x_1 + x_2 - 2) - \alpha^2(x_1 + x_2 - 1)\big] \\
&= \log\big[1 + \alpha(x_1 + x_2 - 1)\big] + \log\big[1 - \alpha\big],
\end{align*}
which in turn implies that
\[
\sum_{i=1}^{n}\log\big[1+\alpha(x_{i}-1)\big]\geq\log\bigg[1+\alpha\bigg(\sum_{i=1}^{n}x_{i}-1\bigg)\bigg]+(n-1)\log\big[1-\alpha\big]. 
\]
In addition, it is straightforward to check that the objective function satisfies
\begin{align*}
	-HSA\log (KH) \le f(\mu) = \sum_{(h, s, a) \in [H]\times\cS\times\cA} \log\bigg[\frac{1}{KH} 
	+\mathbb{E}_{\pi\sim\mu}\big[\widehat{d}_{h}^{\pi}(s,a)\big]\bigg]\le HSA\log2 
\end{align*}
for any $\mu \in \Delta(\Pi)$. 
Taking this collectively with \eqref{eq:per-iteration-progress}, 
we immediately see that 
the subroutine terminates within  $O\big(HSA\log (KH)\big)$ iterations.

\bibliographystyle{apalike}
\bibliography{bibfileRL}

\end{document}